\documentclass{article}

\usepackage{preamble_hal}

\usepackage{authblk}

\newcommand{\R}{\mathbb R}
\newcommand{\cN}{\mathcal N}
\newcommand{\cR}{\mathcal R}
\newcommand{\cE}{\mathcal E}

\newcommand{\softmax}{\operatorname{softmax}}

\newcommand{\Id}{I_d}

\usepackage{etoc}

\usepackage{lmodern}
\usepackage{microtype}
\usepackage[dvipsnames,svgnames]{xcolor}
\usepackage[most]{tcolorbox}
\usepackage{hyperref}
 
\colorlet{inkblue}{RoyalBlue!80!black}
\definecolor{softgray}{RGB}{95,95,95}   
 
\hypersetup{colorlinks=true, 
            linkcolor=inkblue,     
            citecolor=Peach,
            urlcolor=inkblue}
 
\newcommand{\papertitle}{Softmax Attention on Gaussian Mixtures}
\newcommand{\papersubtitle}{Linear When It Can, Selective When It Must}
\newcommand{\paperdate}{October 2026}
 
\newcommand{\aff}[1]{\textsuperscript{#1}}
 
\newcommand{\paperabstract}{%
Softmax attention, at the heart of Transformers, has demonstrated remarkable capabilities. Yet its underlying mechanisms remain only partially understood. Recent theoretical work studies Gaussian prompts, where the infinite-prompt limit reduces softmax attention to a linear map, but also removes the query-dependent selection that distinguishes it from linear attention. This work studies the infinite-prompt limit of softmax attention on Gaussian mixtures, which retain the tractability of Gaussian data while introducing latent structure, multimodality, and nonlinear dependencies. We show that softmax attention can represent and learn, via gradient-based methods, optimal solutions to a range of statistical tasks, including supervised classification and denoising. Our results highlight two complementary capabilities of softmax attention: it can recover linear tasks as effectively as its simpler linear counterpart, while also exploiting query-dependent context selection to solve nonlinear tasks beyond the reach of linear attention.
}
\newcommand{\paperkeywords}{softmax attention $\cdot$ Gaussian mixtures
$\cdot$ large-prompt limit $\cdot$ optimization landscape $\cdot$ denoising
$\cdot$ in-context learning}
 
\newcommand{\makepapertitle}{%
  \thispagestyle{empty}%
  \begingroup
  \noindent{\color{inkblue}\rule{\linewidth}{1.4pt}}\par
  \vspace{5pt}
  \noindent{\sffamily\footnotesize\color{softgray}%
    PREPRINT\hfill\paperdate}\par
  \vspace{22pt}
  \noindent{\fontsize{23}{28}\selectfont\bfseries\color{inkblue}%
    \papertitle\par}
  \vspace{7pt}
  \noindent{\fontsize{15.5}{19}\selectfont\itshape\color{softgray}%
    \papersubtitle\par}
  \vspace{20pt}
  \noindent{\large
    Simon Gabet\aff{1}\qquad
    Etienne Boursier\aff{1}\qquad
    Claire Boyer\aff{1,2}\par}
  \vspace{8pt}
  \noindent{\small\color{softgray}%
    \aff{1}Universit\'e Paris-Saclay, CNRS, Inria, Laboratoire de
    math\'ematiques d'Orsay, 91405 Orsay, France\\[1pt]
    \aff{2}Institut Universitaire de France\par}
  \vspace{6pt}
  \vspace{16pt}
  \begin{tcolorbox}[enhanced, sharp corners, boxrule=0pt,
      colback=inkblue!4, colframe=inkblue!4,
      borderline west={2.2pt}{0pt}{inkblue},
      left=12pt, right=12pt, top=9pt, bottom=9pt]
    {\sffamily\bfseries\footnotesize\color{inkblue}ABSTRACT}\par
    \vspace{5pt}
    {\small\paperabstract\par
    }
    \vspace{7pt}
    {\footnotesize\sffamily\bfseries\color{inkblue}Keywords}\enspace
    {\footnotesize\paperkeywords}
  \end{tcolorbox}
  \vspace{10pt}
  \endgroup
}

\begin{document}

\makepapertitle

\section{Introduction}

Softmax attention lies at the core of modern Transformer models. Yet, despite its remarkable empirical success, its theoretical behavior remains poorly understood, even in highly simplified settings. 
Theoretical accounts of attention often replace softmax by a linear kernel or study distributions for which the population softmax operator itself becomes linear. These reductions lead to precise learning-dynamics results, but they remove the very mechanism that makes attention selective: which part of the prompt is relevant may depend on the query.

We study attention in the large-prompt regime. In this work, a prompt of length $L$ is a sequence of~$L$ tokens drawn independently from a distribution~$\mu$ on~$\R^d$. Softmax
attention over these tokens depends on them only through their empirical
measure, and converges, as $L\to\infty$, to an operator acting directly on
$\mu$ (Section~\ref{sec:mixtures}). In this large-prompt limit, an attention layer
becomes a map that takes a probability measure, the prompt, and returns a
function of the query. This viewpoint removes finite-sample fluctuations
and exposes the geometry of the operator itself; concentration results
relate it to its finite-prompt counterpart for sub-Gaussian tokens
\citep{vuckovic2020,boursier2026,bohbot2026}.

The choice of $\mu$ then determines what can be observed. When $\mu$ is
Gaussian \citep{castin2025,boursier2026}, the limiting operator is an affine function of the query: the
analysis is tractable, but selection has disappeared. Gaussian mixtures
are the natural next step. They retain closed-form computations while
introducing latent structure, multimodality and, crucially, a reason for
a query to attend to one part of the context rather than another. They
thus offer a setting in which one can ask precisely what the softmax
nonlinearity brings.

\paragraph{Contributions.} We answer this question through one structural identity and two statistical tasks that play complementary roles.
\begin{itemize}
\item \textbf{Softmax attention as a gated mixture of experts.}
On any Gaussian mixture, a softmax head acts as a query-dependent
combination of affine experts; with shared covariances, it splits exactly
into a global linear map and a softmax gate that selects, for each query,
the relevant components (Proposition~\ref{prop:gmm-attention}). Linear
attention, by contrast, weighs the components by fixed proportions, blind
to the query. This identity is the lens of the paper: it tells us when the
gate is superfluous, and when it does the essential work.
\item \textbf{When the gate is redundant: learning LDA.}
In linear discriminant analysis (LDA), the optimal classifier is linear, and a
single head can represent it exactly. The question is then one of
optimization: does a redundant, nonconvex parameterization remain
trainable? In effective (reparameterized) coordinates, the logistic risk has no spurious
critical points (Theorem~\ref{thm:lda-landscape}); in the original attention
parameters, every nonoptimal critical point is a strict saddle as soon as
the class means are not symmetric
(Theorem~\ref{thm:lda-landscape-initial}). The landscape is benign but not
coercive: near-optimal directions escape to infinity
(Proposition~\ref{prop:lda-escape-main}). We therefore turn to the
dynamics, and prove exponential convergence of gradient flow near the
regular manifold of optimal parameters (Theorem~\ref{thm:lda-dynamics})
and, in the symmetric balanced case, global convergence on a
Krylov-structured invariant subspace (Theorem~\ref{thm:lda-krylov}).
\item \textbf{When the gate is essential: denoising.}
Optimal denoising of a Gaussian mixture requires two operations at once:
shrinking the observation, and deciding which cluster it comes from. The
optimal denoiser is therefore itself a gated affine map. Two softmax heads
reproduce it exactly; a single head generally cannot, and no
linear-attention predictor, whatever its number of heads, comes close to
the optimal error (Theorem~\ref{thm:denoise-expressivity}). Iterating the
same map changes its meaning: below a sharp separation threshold, it
collapses every input to the midpoint between the component centers; above it, it clusters
(Proposition~\ref{prop:phase-transition}). Finally, with residual
connections, a single set of parameters denoises optimally in context, for
any cluster locations and proportions (Proposition~\ref{prop:icl}): the
weights store what is shared across tasks, while attention infers from the
prompt what is specific to each one.
\end{itemize}

\paragraph{Related work.}
The measure formulation of attention was developed by \citet{vuckovic2020} and used to study deep attention dynamics and clustering \citep{sander2022,geshkovski2023,castin2025}. Recent work analyzes training in the large-prompt Gaussian regime \citep{boursier2026,maulen2026,goel2026training}; our mixture identity is the first step beyond the linear operator resulting from Gaussian prompt. Transformer classification has been studied through linear attention and algorithm emulation \citep{frei2025,shen2025,zhang2026}. Our result instead learns the Bayes boundary directly with a single softmax layer and exposes the geometry of its population risk in the attention parameters. For denoising, \citet{smart2025} connect one-layer attention to associative-memory retrieval when the context is clean; \citet{rosu2025} and \citet{li2026} study attention-based denoising for diffusion. We treat noisy contexts and nondegenerate mixtures with standard dot-product softmax, give exact Bayes constructions, and prove a strict separation from linear attention. Concurrent work studies a more general two-stage empirical-Bayes viewpoint using RBF attention \citep{smart2026}. A complementary statistical separation between softmax and linear attention appears in \citet{duranthon2026statistical}.

\section{Population Attention on Mixtures}
\label{sec:mixtures}

\paragraph{Attention as an operator on measures.}
We represent a prompt by a probability measure $\mu$ on $\R^d$, absorb the
query and key matrices into a single interaction matrix
$U\in\R^{d\times d}$, and denote by $V\in\R^{d\times d}$ the value matrix.
For a query $z\in\R^d$, the attention operator reads
\begin{equation}
T_{U,V}[\mu](z)
=\frac{\int e^{(z')^\top Uz}\,Vz'\,\mathrm d\mu(z')}
      {\int e^{(z')^\top Uz}\,\mathrm d\mu(z')}.
\label{eq:attention}
\end{equation}
For a finite prompt $z_1,\ldots,z_L$ with empirical measure
$\widehat\mu_L=L^{-1}\sum_{i=1}^L\delta_{z_i}$, the quantity
$T_{U,V}[\widehat\mu_L](z)$ is exactly softmax dot-product attention over
the tokens. When the tokens are drawn i.i.d.\ from a sub-Gaussian measure~$\mu$,
the law of large numbers, applied to the numerator and the denominator,
gives $T_{U,V}[\widehat\mu_L](z)\to T_{U,V}[\mu](z)$ almost surely for every
query~$z$; quantitative versions are given in
\citet{boursier2026,bohbot2026}. All our results concern the limiting
operator $T_{U,V}[\mu]$.

\paragraph{From a Gaussian to a mixture.}
If $\mu$ is Gaussian, centered at $m$ and of covariance matrix $\Gamma$, the attention operator is
\[
T_{U,V}[\mu](z)=V(m+\Gamma Uz),
\]
an affine function of the query, in which the softmax leaves no trace. On a
mixture, the same computation applies component by component, and the
components then compete through the normalization. Throughout, Gaussian
covariances are positive definite unless a degenerate case is stated
explicitly.
\begin{proposition}[Softmax-gated affine experts]
\label{prop:gmm-attention}
\label{prop:attention_op_for_gauss_mix}
Consider a mixture $\mu=\sum_{k=1}^M\pi_k\cN(m_k,\Gamma_k)$ of $M$ Gaussian components. The infinite-prompt attention layer parameterized by $U$ and $V$ gives for any $z\in \mathbb{R}^d$, 
\begin{align}
T_{U,V}[\mu](z)
&=V\sum_{k=1}^M\alpha_k(z)(m_k+\Gamma_kUz), \label{eq:gmmatt}
\end{align}
with 
$$
\alpha_k(z)=\softmax_k\!\left(\log\pi_k+m_k^\top Uz+\tfrac12(Uz)^\top\Gamma_kUz\right). 
$$
If $\Gamma_k=\Gamma$ for all $k$, the common quadratic offsets in the softmax function cancels and
\begin{align}
T_{U,V}[\mu](z)&=V\Gamma Uz+V\sum_k\alpha_k(z)m_k,
\label{eq:shared}
\end{align}
with
$
\alpha_k(z)=\softmax_k(\log\pi_k+m_k^\top Uz)$.
\end{proposition}%
The proof of Proposition~\ref{prop:gmm-attention} 
is given in
Appendix~\ref{app:gmm-attention-proof}.
Equation~\eqref{eq:shared} is the organizing principle of the paper: the first term performs global linear prediction; the second selects components and returns a query-dependent barycenter. 

\paragraph{Attention with generalized scores.} More generally, attention on any mixture $\sum_k\pi_k\mu_k$ is a query-dependent convex combination of the component-wise attention maps, with gates proportional to the alignment to the query (see the supplement). Indeed, for $T_a[\mu](z):=\int a(z,z')z'\mathrm d\mu(z')/\int a(z,z')\mathrm d\mu(z')$, we have
$$
T_a[\mu](z)
=
\sum_{k=1}^M r_k(z)\,T_a[\mu_k](z)
$$
with $r_k(z)
\coloneqq 
\pi_k
{
 \int_{\mathbb{R}^d} a(z,z')\,\mathrm{d}\mu_k(z')
}\,/\,{
 \int_{\mathbb{R}^d} a(z,z')\,\mathrm{d}\mu(z')
}$ and $a>0$ an arbitrary alignment function (e.g., exponential dot product for softmax scores), see
Appendix~\ref{app:general-mixture-attention} for details.
Normalized attention, such as softmax attention, can therefore perform \emph{context selection}: in regions where $r_k(z)$ is close to one, the operator essentially selects the context $\mu_k$; allowing to adaptively select which cluster is relevant to a given query. This behavior contrasts with classical linear attention, defined by
$
T_{\mathrm{lin}}[\mu](z)
\coloneqq 
\int_{\mathbb{R}^d}
\langle Uz,z'\rangle Vz'\,\mathrm{d}\mu(z').
$
Indeed,
$
T_{\mathrm{lin}}[\mu](z)
=
\sum_{k} \pi_k T_{\mathrm{lin}}[\mu_k](z)$, so linear attention combines the different contexts with fixed mixture proportions $\pi_k$, independently of the query, and cannot separate the contribution of the different clusters.

\section{Attention-based Linear Discriminant Analysis}
\label{sec:lda}

We now turn to classification. Under linear
discriminant analysis (LDA), the Bayes score is affine, and hence does not
require the nonlinear selection mechanism of softmax attention. The question
is instead whether a single softmax head can represent this score and whether
gradient-based training can recover it despite this redundant and nonconvex parameterization.

\subsection{From attention to an affine--sigmoid predictor}

Let $Y\in\{-1,1\}$ with $\mathbb{P}(Y=1)=\pi_+$, and assume the classical LDA model, with input features $X\in\mathbb{R}^d$ such that
\begin{equation}
X\mid Y=c\sim\cN(m_c,\Gamma),\qquad c\in\{-1,1\},
\label{eq:lda-model}
\end{equation}
where the two classes have different means and share a positive-definite
covariance matrix. The Bayes classifier is $\operatorname{sign}(f^\star)$, where the Bayes
score $f^\star(x)=w^{\star\top}x+b^\star$ is affine, with
\begin{align}
w^\star&=\Gamma^{-1}(m_1-m_{-1}), \label{eq:lda-bayes} \\
b^\star&=-\tfrac12\bigl(m_1^\top\Gamma^{-1}m_1
-m_{-1}^\top\Gamma^{-1}m_{-1}\bigr)
+\log\frac{\pi_+}{1-\pi_+}.
\nonumber
\end{align}
We feed augmented context tokens $(X,Y)\in\R^{d+1}$ to one attention head and query it at $(x,0)$; only the last output coordinate is therefore used to store the prediction $f_{\theta}(x)$. Writing the relevant blocks of $U,V$ interacting with $X$ or $Y$ (see Appendix~\ref{app:lda-effective-model}) as $\theta = (U_{11},u_{12},v_{21},v_{22})$,
Proposition~\ref{prop:gmm-attention} yields for $\mu$ the Gaussian mixture described in~\eqref{eq:lda-model}:
\begin{align}
T_{U,V}[\mu]((x,0))_{d+1} &= f_\theta(x)\notag\\
&\coloneqq w^\top x+b+r\sigma(s_0+s^\top x), 
\label{eq:effective}
\end{align}
where $\sigma(t)=(1+e^{-t})^{-1}$ and
\begin{equation}
\begin{split}
w&=U_{11}\Gamma v_{21},\qquad b=v_{21}^\top m_{-1}-v_{22},\\
r&=v_{21}^\top(m_1-m_{-1})+2v_{22},\\
s&=2u_{12}+U_{11}(m_1-m_{-1}),\\
 s_0&=\log\frac{\pi_+}{1-\pi_+}.
\end{split}
\label{eq:Phi}
\end{equation}
\looseness=-1
The exact derivation is given in Appendix~\ref{app:lda-effective-model}.
Thus a softmax head reduces, in effective coordinates, to an affine predictor
plus a single sigmoidal correction. In particular, one head can realize the
Bayes score exactly; an explicit construction is given in
Appendix~\ref{app:lda-realizability}. At the effective level, this happens through two distinct
mechanisms: either $r=0$ suppresses the nonlinear term, or $s=0$ makes it
constant and absorbable into the bias~$b$. Their preimages in the original
attention parameters need not have the same geometry---and, depending on the
class means and priors, both mechanisms need not be simultaneously available.
This distinction is central to the optimization analysis below.

\subsection{A benign but noncoercive optimization landscape}

We train the score with the population logistic risk\footnote{From now on, we only consider the relevant parameters $\theta$, since the remaining blocks of $U,V$ are neither updated during training nor affect the predicted value.
}
\begin{align}
\label{eq:lda-risk}
    \cR(\theta)
&\coloneqq \mathbb E\left[\log(1+e^{-Yf_\theta(X)})\right] \\
&=\widetilde\cR(\Phi(\theta)), 
\notag
\end{align}
where $\Phi(\theta)=(w,b,r,s)$ is defined in~\eqref{eq:Phi}, and 
$$\widetilde\cR(w,b,r,s) \coloneqq  \mathbb E\!\left[\log(1+e^{-Yg_{w,b,r,s}(X)})\right],
$$
with $g_{w,b,r,s}(x) \coloneqq w^\top x+b+r\sigma(s_0+s^\top x)$.

Thus $\cR$ is the risk in the original attention parameters, whereas
$\widetilde\cR$ is expressed in effective coordinates. The logistic loss
is important here: its unique population minimizer over measurable scores
is the Bayes score $f^\star$. Since the attention model realizes this
score, the unrestricted optimum is attained in both parameterizations;
we denote its value by $\cR^\star$.

\paragraph{The effective landscape.}
The reduced model is not jointly convex in $(w,b,r,s)$. However, fixing
$s$ leaves a convex problem in $(w,b,r)$ that still contains the Bayes
predictor, through the choice $(w,b,r)=(w^\star,b^\star,0)$. This simple
observation explains why stationarity is enough to guarantee optimality.

\begin{theorem}[Effective landscape]
\label{thm:lda-landscape}
All critical points of $\widetilde\cR$ are global minimizers, and
\begin{align*}
\cE^\star\coloneqq \arg\min\widetilde\cR
&=\operatorname{Crit}(\widetilde\cR)=\cE_1\cup\cE_2,\\
\cE_1&=\{(w^\star,b^\star,0,s):s\in\R^d\},\\
\cE_2&=\{(w^\star,b^\star-r\sigma(s_0),r,0):r\in\R\}.
\end{align*}
\end{theorem}
These two affine spaces recover the representation mechanisms identified above:
$\cE_1$ removes the sigmoid, while $\cE_2$ makes it constant and absorbs its
contribution into the bias. They represent the same predictor but leave
different effective parameters unconstrained.

\paragraph{Returning to attention parameters.}
Does this favorable landscape survive the return to the attention parameters? The chain rule
\[
\nabla\cR(\theta)
=D\Phi(\theta)^*\nabla\widetilde\cR(\Phi(\theta))
\]
shows the possible obstruction: a singular differential can annihilate a
nonzero effective gradient. The following result identifies exactly where
this obstruction is avoided.

\begin{lemma}[Regularity of the reparameterization]
\label{lem:lda-regularity-main}
For any $\theta \in \mathbb{R}^{d\times d}\times \mathbb{R}^{d}\times \mathbb{R}^{d}\times \mathbb{R}$, the adjoint $D\Phi(\theta)^*$ is injective if and only if
$m_1\neq-m_{-1}$ and $v_{21}\neq\mathbf{0}$.
\end{lemma}

Consequently, under asymmetric means, every critical point with
$v_{21}\neq\mathbf{0}$ is globally optimal. The exceptional set $\{v_{21}=\mathbf{0}\}$
requires a separate curvature argument: its nonoptimal critical points
admit a negative-curvature direction. Together, these facts yield the
following characterization.

\begin{theorem}[Landscape in attention parameters]
\label{thm:lda-landscape-initial}
The minimizer set of $\cR$ is $\Theta^\star=\Theta_1\cup\Theta_2$, where
\[
\Theta_1=\Phi^{-1}(\cE_1),\qquad
\Theta_2=\Phi^{-1}(\cE_2).
\]
If $m_1\neq-m_{-1}$, every critical point of $\cR$ is either a global
minimizer or a strict saddle. In particular, all local minima are global.
\end{theorem}

\paragraph{Why this does not establish global convergence.}
The preceding theorem controls finite critical points, but not parameter
escape. In fact, the effective risk admits diverging minimizing sequences,
as shown by the following construction.

\begin{proposition}[Minimizing sequences at infinity]
\label{prop:lda-escape-main}
For every $c\in\R^d$, the sequence
\[
\zeta_R=(w^\star-\sigma'(s_0)c,
 b^\star-\sigma(s_0)R,R,c/R)
\]
satisfies $g_{\zeta_R}\to f^\star$ in $L^1(\mu_X)$ and
$\widetilde\cR(\zeta_R)\to\cR^\star$ as $R\to\infty$, where
$g_{w,b,r,s}(x)$ is defined in \eqref{eq:lda-risk} and $\mu_X$ is the
query distribution.
\end{proposition}
As the sigmoid amplitude grows and its direction shrinks, its constant
and linear contributions are canceled by the affine part. The predictor
approaches Bayes optimality while the parameters diverge. This does not
prove divergence of gradient flow, but shows why the landscape alone
cannot supply a boundedness argument. Proofs of the landscape and
noncompactness results are given in Appendix~\ref{app:lda-landscape-proof}.

\subsection{Gradient-flow convergence}

We now study the gradient flow $\dot\theta_t=-\nabla\cR(\theta_t)$. The two Bayes families
$\Theta_1$ and $\Theta_2$ have different local geometries, so their common
optimality does not imply the same convergence behavior.

\paragraph{Local convergence near the regular component.}
Assume first that $m_1\neq-m_{-1}$, and define
$\Theta_1^{\mathrm{reg}}=\{\theta\in\Theta_1:s(\theta)\neq\mathbf{0}\}$.
At every $\theta^\star\in\Theta_1^{\mathrm{reg}}$, the Hessian satisfies
\[
\ker\nabla^2\cR(\theta^\star)
=T_{\theta^\star}\Theta_1^{\mathrm{reg}}.
\]
Thus the only flat directions are tangent to the minimizer manifold;
the risk has strictly positive curvature in the normal directions.
This Morse--Bott structure provides the local stability needed for the
following convergence result.

\begin{theorem}[Local convergence to the regular Bayes manifold]
\label{thm:lda-dynamics}
Assume $m_1\neq-m_{-1}$ and let $\theta^\star\in\Theta_1$ satisfy
$s(\theta^\star)\neq\mathbf{0}$. If initialized sufficiently close to $\theta^\star$,
gradient flow exists for all $t\geq0$, remains in a neighborhood of
$\Theta_1$, and converges exponentially to a point
$\theta_\infty\in\Theta_1$. In particular,
$f_{\theta_\infty}=f^\star$ almost surely.
\end{theorem}
The proof (see Appendix~\ref{app:lda-local-proof}) controls both the decay
normal to the manifold and the displacement along it, giving convergence
to an optimal parameter.

\paragraph{Why the constant-gate component is different.}
For every $\theta^\star$ in $\Theta_2$ with $r(\theta^\star)\neq0$ perturbing the gate direction can be compensated to first
order by changing the affine coefficient. Under the same asymmetry
assumption, the resulting Hessian kernel has $d$ extra dimensions:
\[
\dim\ker\nabla^2\cR(\theta^\star)
-\dim T_{\theta^\star}\Theta_2=d.
\]
These are flat directions that are transverse to the local minimizer manifold. Thus the preceding Morse--Bott argument does not extend to $\Theta_2$. The calculation is given in
Appendix~\ref{app:lda-theta2-geometry}. This degeneracy alone determines
neither convergence nor its rate. The slower decay in
Figure~\ref{fig:lda-local} is an empirical observation consistent with the previous discussion (see
Appendix~\ref{appendix:LDA_experimental_details} for a description of the numerical protocol).

\paragraph{Global convergence on a symmetric invariant subspace.}
The centrally symmetric case $m_1=-m_{-1}$ falls outside the regularity
argument for $\Phi$. When the classes are also balanced, symmetry instead
allows us to construct an invariant subspace on which the predictor
remains linear throughout training. Assume $m_1=-m_{-1}\eqqcolon m$ and
$\pi_+=1/2$, and define the Krylov space
\[
K=\operatorname{span}\{m,\Gamma m,\ldots,\Gamma^{d-1}m\}.
\]
With $P_K$ the orthogonal projector onto $K$, let
\[
\mathcal I_K
=\{u_{12}=\mathbf{0},\ v_{22}=0,\ v_{21}\in K^\perp,\ U_{11}P_K=\mathbf{0}\}.
\]
Orthogonality of $v_{21}$ to $m$ alone is not generally preserved by the flow,
which repeatedly applies $\Gamma$. The Krylov space ensures that
$K^\perp$ is $\Gamma$-invariant, so these constraints close under the
dynamics.

\begin{theorem}[Global convergence on the invariant subspace]
\label{thm:lda-krylov}
\label{thm:exponential_convergence_krylov_submanifold}
The subspace $\mathcal I_K$ is invariant under gradient flow. If
$U_{11}(0)=\mathbf{0}$, $u_{12}(0)=\mathbf{0}$, $v_{22}(0)=0$ and $v_{21}(0)\in K^\perp\setminus\{\mathbf{0}\}$, then $\theta_t$ converges
exponentially to a limit $\theta_\infty\in\mathcal I_K$ satisfying
$f_{\theta_\infty}=f^\star$.
\end{theorem}
On $\mathcal I_K$, the bias and sigmoid term vanish, leaving
$f_{\theta_t}(x)=\langle U_{11}(t)\Gamma v_{21}(t),x\rangle$.
The proof combines coercivity of the linear logistic risk with the
conserved balance
\[
\|U_{11}(t)\|_F^2-\|v_{21}(t)\|^2
=-\|v_{21}(0)\|^2.
\]
The former bounds the effective predictor; the latter prevents its
induced gradient dynamics from degenerating. Further estimates control
the factors and yield their exponential convergence; see
Appendix~\ref{app:lda-krylov-proof}.

From the initialization point of view, setting $U_{11}(0)=\mathbf{0}$, $u_{12}(0)=\mathbf{0}$, and $v_{22}(0)=0$ does not raise any difficulty. The only non-trivial condition is to choose $v_{21}(0)\in K^\perp\setminus\{\mathbf{0}\}$, requiring the knowledge of \(\Gamma\) and \( m\). However, from a practical perspective, this assumption is mild in high-dimensional settings when \(\dim K\ll d\), since an isotropic random initialization typically has only a fraction of order \(\dim(K)/d\) of its mass in $K$. In particular, when \(\Gamma=\gamma^2 \mathrm{I}_d\), one has \(\dim K=1\) and $\mathrm d(v_{21}(0)/\|v_{21}(0)\|,K^\perp)=O_\mathbb{P}(\frac{1}{\sqrt{d}})$.
Our experiments in dimension $d=20$ show rapid convergence
to near-Bayes risk without knowledge of $m$ nor $\Gamma$
(Appendix~\ref{appendix:LDA_experimental_details},
Figure~\ref{fig:risk-sym-v21-random-d=20}). 

\begin{figure}[t]
\centering
\includegraphics[width=0.45\linewidth]{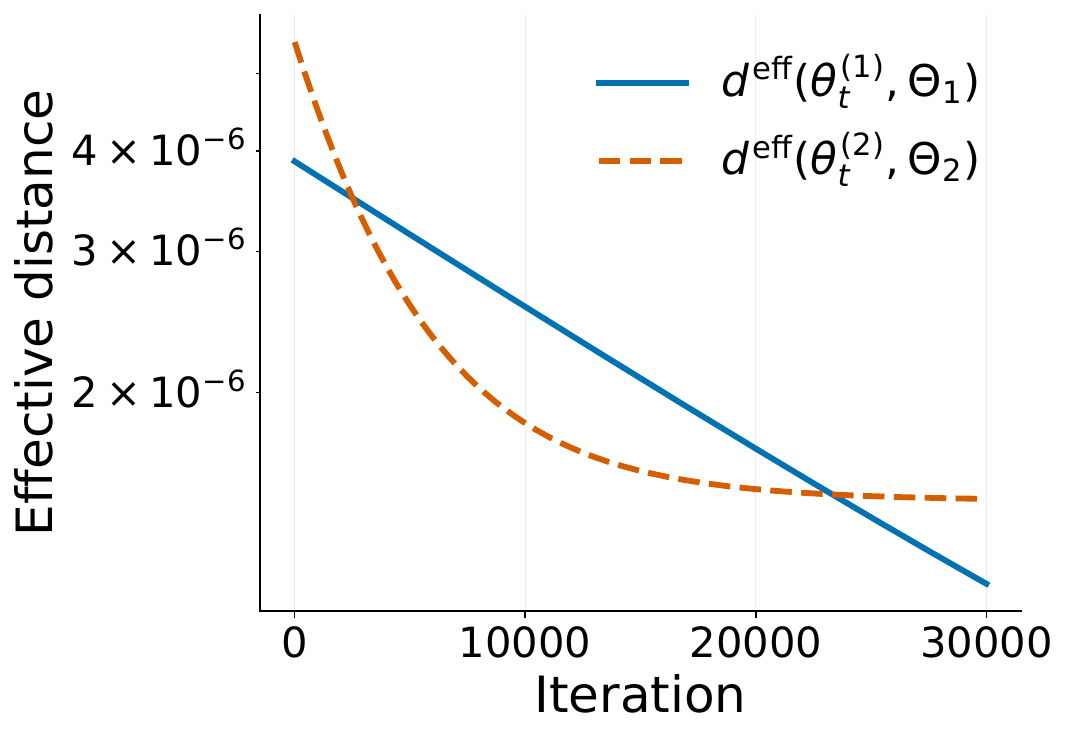}
\caption{Population gradient descent near the two Bayes manifolds (see table~\ref{tab:initializations} for initialization details). The effective distance (euclidean distance between the effective parameters and $\cE_1$ or $\cE_2$) to the regular component $\Theta_1$ is linear on a log scale, while convergence toward the flatter component $\Theta_2$ is markedly slower, matching the Morse--Bott characterization.}
\label{fig:lda-local}
\end{figure}

\section{Denoising: Selection Becomes Essential}
\label{sec:denoising}

Classification used only the affine part of the mixture-attention
decomposition. Denoising reveals the role of its nonlinear part. Denoising a
sample from a Gaussian mixture requires both shrinking the observation within
each component and inferring which component generated it. The latter is a
genuinely query-dependent selection problem, and it is precisely where
softmax attention differs from linear attention.

\subsection{The Bayes denoiser as a gated affine map}

Let $X\sim\sum_{k=1}^M\pi_k\cN(m_k,\Gamma)$, consider a noisy version of it
\begin{equation}\label{eq:denoisingmodel}
Z=X+\varepsilon,
\qquad \varepsilon\sim\cN(0,\xi^2\Id),
\end{equation}
where the noise $\varepsilon$ is independent of $X$, and set $S=\Gamma+\xi^2\Id$. Under
squared loss $\cR_\mathrm{sq}(f):=\mathbb E[\|X-f(Z)\|^2]$, 
with a slight abuse, $f^\star$ now denotes
the optimal reconstruction of $X$ from $Z$ (see Appendix~\ref{app:bayes-denoiser-proof}): 
\begin{equation}
f^\star(z)=\mathbb E[X\mid Z=z]
=\Gamma S^{-1}z+\xi^2S^{-1}\sum_k\beta_k(z)m_k,
\label{eq:bayes-denoiser}
\end{equation}
where
\[
\beta_k(z)=\softmax_k\!\left(\log\pi_k+m_k^\top S^{-1}z-\tfrac12m_k^\top S^{-1}m_k\right).
\]
Equivalently, if $C$ denotes the latent component label, then
$\beta_k(z)=\mathbb{P}(C=k\mid Z=z)$. Thus the first term in
\eqref{eq:bayes-denoiser} performs within-component Gaussian shrinkage, while
the second returns the posterior barycenter of the component means. In the
high-noise regime the latter approaches the prior mean $\sum_k\pi_km_k$; in
the low-noise regime the full denoiser approaches the identity.

Denoting by $\mu_Z$ the distribution of $Z$ defined by \eqref{eq:denoisingmodel}, this decomposition is strikingly close to the shared-covariance attention
formula~\eqref{eq:shared}:
\begin{align*}
T_{U,V}[\mu_Z](z)&=VSUz+V\sum_k\alpha_k(z)m_k,\\
\text{with} \quad \alpha_k(z)&=\softmax_k(\log\pi_k+m_k^\top Uz).
\end{align*}
Both maps combine a linear transform of the observation with a softmax-weighted
barycenter. Two obstructions remain, however. First, the posterior weights $\beta_k$
contain the offsets $-\frac12m_k^\top S^{-1}m_k$, which standard dot-product
attention cannot encode. Second, within a single attention head, the same matrices
$(U,V)$ govern both the linear and gated terms.

To isolate the second obstruction, we henceforth consider a mixture of 2
components in dimension $d\geq2$, with $\Gamma=\gamma^2\Id$, distinct means
satisfying $\|m_1\|=\|m_2\|$, and arbitrary positive mixture weights. The
equal-norm assumption cancels out the offsets in the barycenter coefficients, giving
\begin{equation}
f^\star(z)=
\frac{\gamma^2}{\gamma^2+\xi^2}z
+\frac{\xi^2}{\gamma^2+\xi^2}
\sum_{k=1}^2\beta_k(z)m_k,
\label{eq:spherical-denoiser}
\end{equation}
with
\(
\beta_k(z)=\softmax_k\!\left(
\log\pi_k+\frac{m_k^\top z}{\gamma^2+\xi^2}\right).
\)

\subsection{Two heads are sufficient, and softmax is necessary}

\begin{figure*}[!t]
\centering
\includegraphics[height=0.3\linewidth]{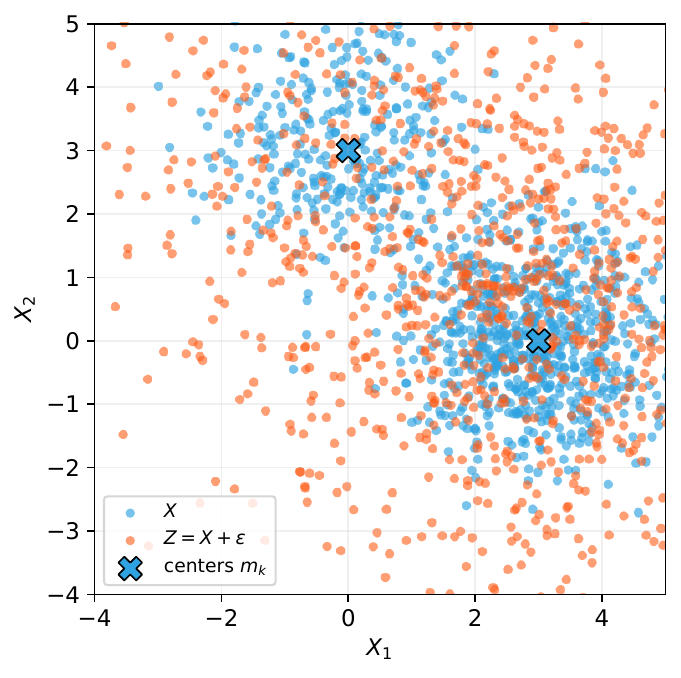}\hspace{20pt}
\includegraphics[width=0.52\linewidth, height=0.3\linewidth]{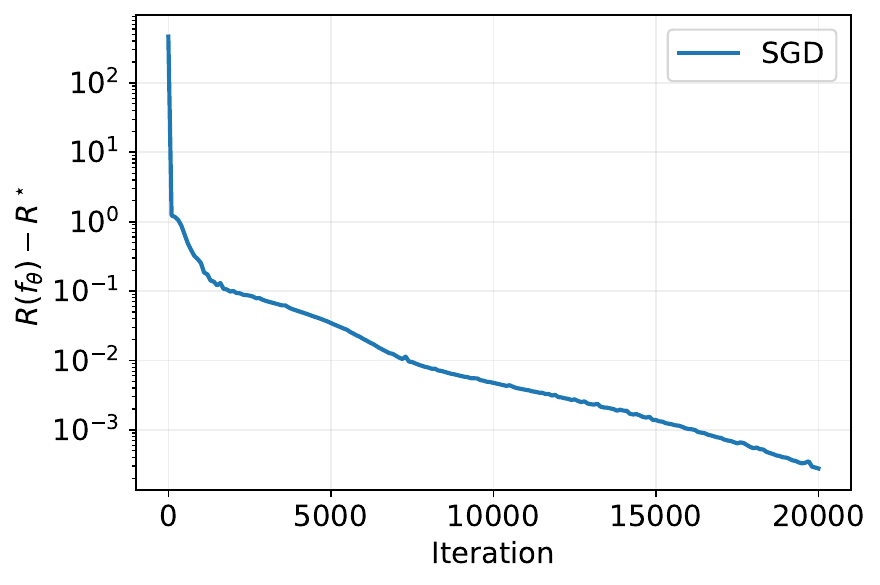}
\caption{Learning the two-head Bayes denoiser in $d=2$. Left: clean samples
$X$ and noisy observations $Z=X+\varepsilon$ for $m_1=(3,0)$,
$m_2=(0,3)$, $\gamma=1$, and $\xi=2$. Right: excess population risk along stochastic
gradient descent with step size $2\times10^{-3}$.}
\label{fig:denoising-opt}
\end{figure*}

Although Equation~\eqref{eq:spherical-denoiser} has exactly the form of an attention
map, one head is generally unable to match its two terms simultaneously. 
%
\begin{theorem}[Exact representation and a softmax separation]
\label{thm:denoise-expressivity}
Assume $\gamma>0$. Two softmax heads realize the Bayes denoiser, i.e., there exist $(U_h,V_h)_{h=1,2}$ such that for all $z$,
\[
T_{U_1,V_1}[\mu_Z](z)+T_{U_2,V_2}[\mu_Z](z)=f^\star(z).
\]
A single head suffices if and only if $\gamma^2=\xi^2$.

Moreover, writing $\mathcal F_{\rm sm}$ and $\mathcal F_{\rm lin}$ for
finite multi-head softmax-attention and multi-layer linear-attention predictors, respectively,
\begin{align*}
\inf_{f\in\mathcal F_{\rm sm}}\{\cR_\mathrm{sq}(f)-\cR_\mathrm{sq}(f^\star)\}&=0,\\
\inf_{f\in\mathcal F_{\rm lin}}\{\cR_\mathrm{sq}(f)-\cR_\mathrm{sq}(f^\star)\}&>0.
\end{align*}
\end{theorem}
%
A single head couples the two terms of $f^\star$ through the same value
matrix. To see this, take $U=S^{-1}$, which already reproduces the gate
exactly ($\alpha_k=\beta_k$). The head then reads
\[
T_{U,V}[\mu_Z](z)=Vz+V\sum_{k=1}^2\beta_k(z)m_k,
\]
whereas $f^\star$ multiplies $z$ by $\gamma^2/(\gamma^2+\xi^2)$ and the
barycenter by $\xi^2/(\gamma^2+\xi^2)$. The same matrix $V$ would have to
apply both factors, which is possible only when they coincide, that is,
when $\gamma^2=\xi^2$; then $V=\tfrac12\Id$ gives $f^\star$ exactly. The
proof shows that this tension persists for every choice of $U$
(Appendix~\ref{app:one-head-obstruction}).

An extra attention head path decouples the two factors. For non-collinear means,
each value matrix keeps one mean and cancels the other, while the two
linear parts add up to the required shrinkage; opposite means admit a
similar construction (Appendix~\ref{app:two-head-nondegenerate}). 

\begin{remark}[Degenerate mixtures]
\textit{When $\gamma=0$, the clean distribution consists of two point masses,
and $f^\star$ reduces to their posterior barycenter. Two heads still
suffice, and the strict separation from linear attention persists.
The single-head obstruction, however, depends on the geometry: it holds
for non-collinear means in dimension two, whereas opposite means admit
a single-head realization. We keep these complementary cases in
Appendix~\ref{app:two-head-degenerate}. They will also help isolate the
clustering dynamics of repeated denoising below. }
\end{remark}

\paragraph{Linear vs.\ softmax attention.} Theorem \ref{thm:denoise-expressivity} tells us that no network composed of linear attention layers can approximate the Bayes risk. Although the setting and task considered differ, this result aligns with the findings of \cite{duranthon2026statistical}, providing another example where softmax attention is provably more expressive than its linear counterpart. 
The key reason is that the Bayes-optimal denoiser is inherently nonlinear: it maps each input $z$ to an adaptive average of the mixture centroids $m_1$ and $m_2$.

\paragraph{Can the exact denoiser be learned?}
Unlike in LDA, the sigmoidal term cannot be suppressed without losing the
target itself, so the optimization problem remains genuinely nonlinear and
nonconvex. We therefore examine it numerically at the population level.
In dimension $d=2$, we optimize the excess risk $\mathbb E_{Z\sim\mu_Z}\|f(Z)-f^\star(Z)\|^2$ using online stochastic gradient descent, with gradient norm clipped at (10) to prevent numerical instabilities during the first optimization steps. Throughout training, we follow the evolution of the excess risk, estimated by Gauss--Hermite quadrature within each Gaussian component (Figure~\ref{fig:denoising-opt}).
The excess risk decreases close to zero in the reported experiment,
suggesting that the Bayes predictor can be approached by direct optimization.
This is numerical evidence for learnability.

\begin{figure*}[!t]
\centering
\includegraphics[width=0.4\linewidth]{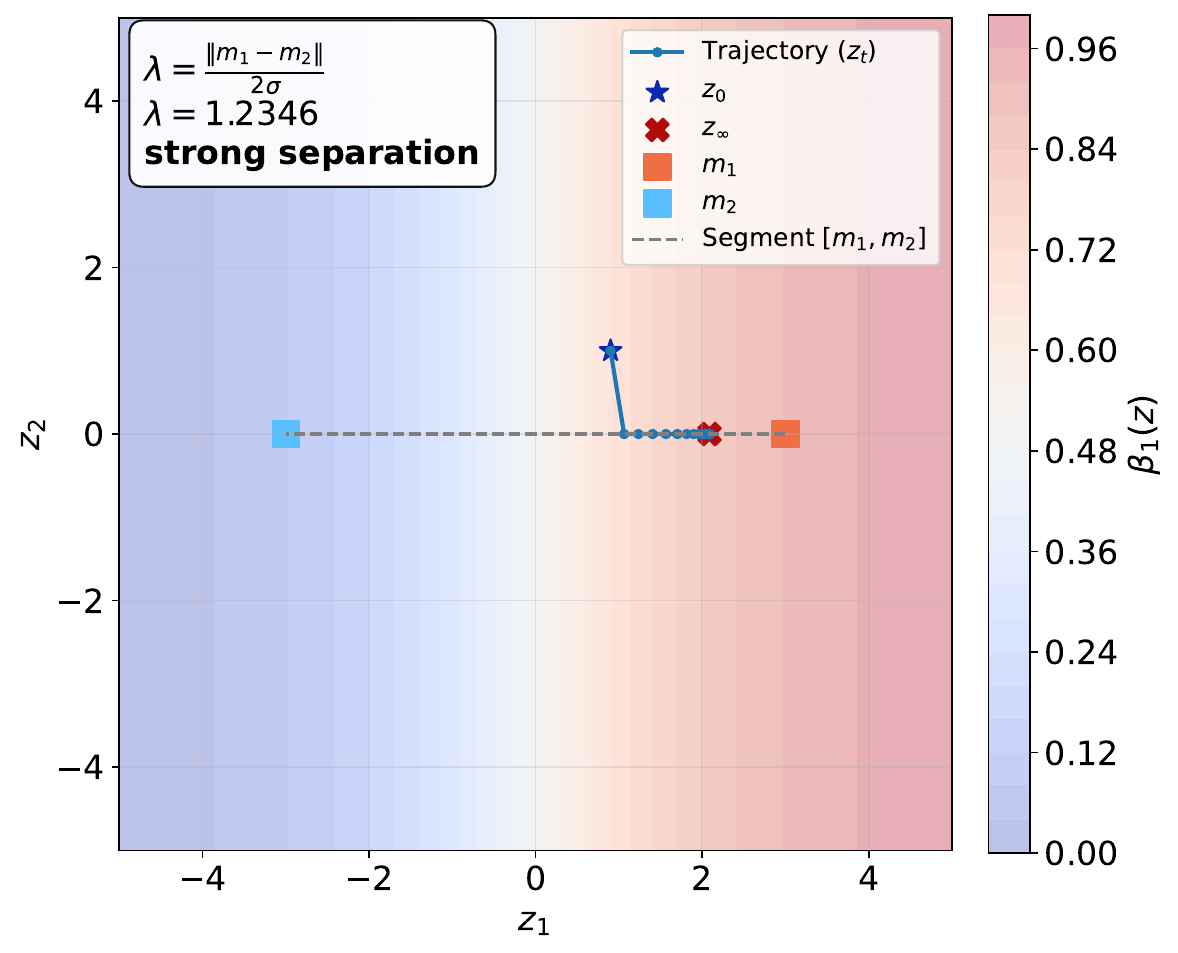}
\hspace{20pt}
\includegraphics[width=0.4\linewidth]{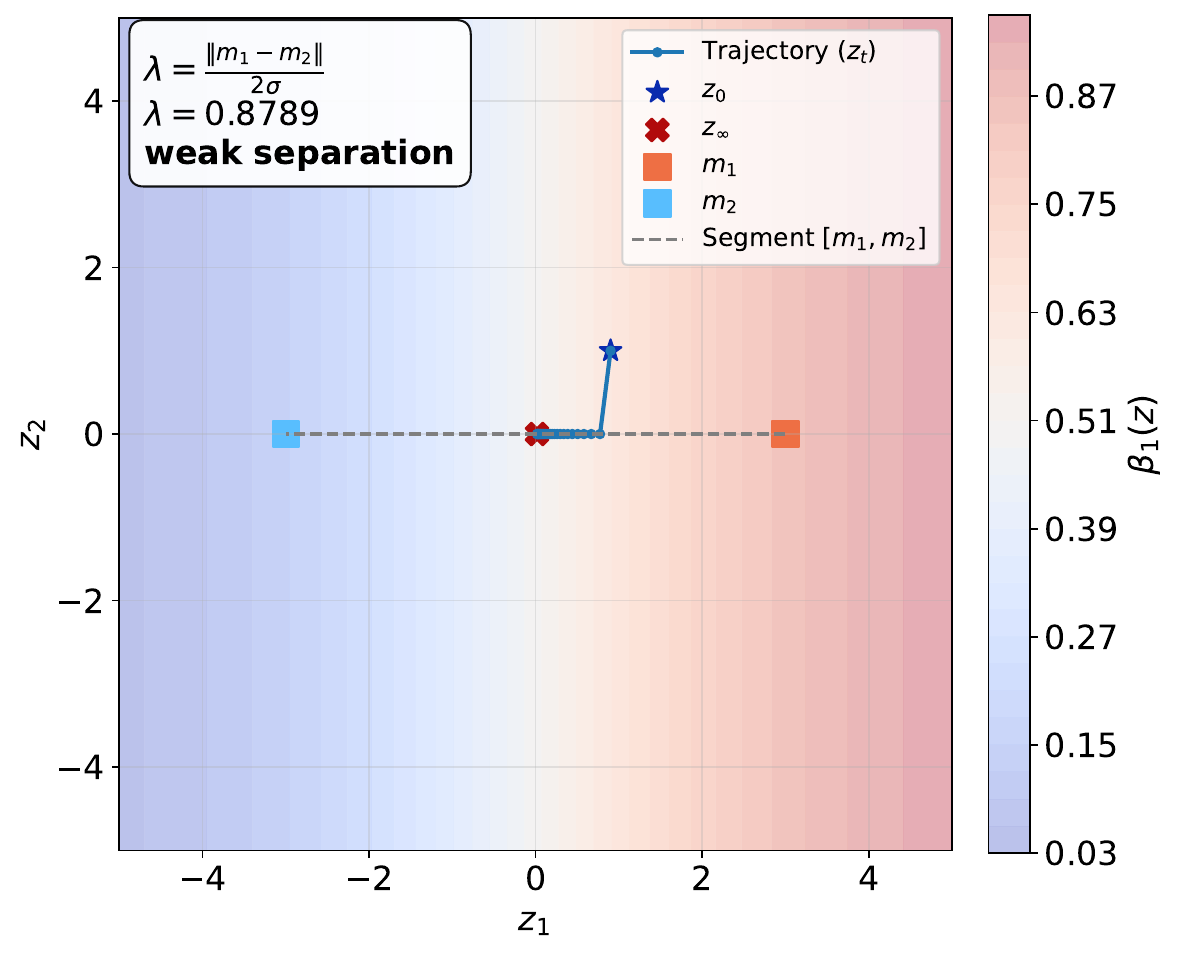}
\caption{Trajectories of the exact denoiser under strong separation (left) and
weak separation (right). The background shows the conditional probability of belonging to cluster 1;
blue dots trace the iterates and the red cross marks their limit. Here
$m_1=-m_2=(0,3)$ and $\gamma=0$.}
\label{fig:phase}
\end{figure*}

\subsection{Iterating the denoiser produces a clustering transition}

While a single denoising step aims at recovering a clean estimate from a corrupted query, repeatedly applying the same denoiser defines a dynamical system whose long-term behavior can reveal the geometry encoded by the data distribution. In particular, \cite{smart2025} show that a denoising step can be interpreted as a gradient descent step associated with a dense associative memory network making the study of its iterates natural from a dynamical perspective. Consider the degenerate balanced mixture $X\in\{m_1,m_2\}$, so that $X$ emanates from a mixture of Dirac masses, with $\gamma=0$ and $\pi_1=\pi_2=1/2$, and consider the dynamics
$z_{n+1}=f^\star(z_n)$. One attention layer encoding $f^\star$ is obviously Bayes optimal, but iteration progressively
erases information about the initial query. We show that this loss of information has a
precise interpretation: depending on the separation of the mixture, the
denoiser either collapses all inputs to their common midpoint or partitions
the space into two basins of attraction.

Let $\lambda=\|m_1-m_2\|/(2\xi)$ so that $\lambda^2$ can be interpreted as a signal-to-noise ratio. After projecting onto the line joining the
means, the fixed-point equation reduces to $u=\tanh(\lambda^2u)$, yielding a transition at $\lambda=1$.

\begin{proposition}[Denoising-to-clustering transition]
\label{prop:phase-transition}
\label{prop:iterated-denoiser-clustering}
If $\lambda\leq1$ (referred to as weak separation), the iterated optimal denoiser dynamics $z_{n+1}=f^\star(z_n)$ converges to
$$z_n \xrightarrow[n\to+\infty]{}(m_1+m_2)/2.
$$

If $\lambda>1$ (referred to as strong separation), almost every initialization converges to one of two stable fixed points $m_{\lambda,1},m_{\lambda,2}$, being such that $m_{\lambda,k}\to m_k$ as $\lambda\to\infty$.
\end{proposition}
The proof is given in Appendix~\ref{app:phase-transition-proof}.
Figure~\ref{fig:phase} displays the two regimes described by Proposition \ref{prop:iterated-denoiser-clustering}. Under weak separation the
dynamic cannot distinguish the components and every trajectory
collapses to the midpoint. Above the threshold, two stable attractors emerge
and approach the component means as the separation grows. Repeated passages through the same layer therefore changes the statistical role of the same map: a one-step estimator becomes, under recurrence, a clustering mechanism.

\subsection{Residual attention denoises in-context}\label{sec:icl}

The previous analyses considered denoising for a fixed mixture,
with optimal attention parameters tailored to the mixture means. In-context denoising
asks for a single architecture that adapts to different mixtures using
only a prompt of noisy observations.

A task is therefore now specified by $\tau=(m_1,m_2,\pi_1)$, with clean and noisy
distributions
\[
\begin{aligned}
\mu_X^\tau
&=\pi_1\cN(m_1,\gamma^2\Id)
 +(1-\pi_1)\cN(m_2,\gamma^2\Id),\\
\mu_Z^\tau
&=\pi_1\cN(m_1,S)
 +(1-\pi_1)\cN(m_2,S),
\end{aligned}
\]
where $S=(\gamma^2+\xi^2)\Id$.
The means and mixture weights vary across tasks, while $\gamma^2$
and $\xi^2$ remain fixed.
For each task, a new parameter $\tau$ is sampled according to an admissible distribution (i.e., with  $m_1\neq m_2$, $\|m_1\|=\|m_2\|$, and $\pi_1\in(0,1)$ almost surely). The layer receives an infinite prompt
$Z_1,\ldots,Z_L,\ldots$ drawn independently from $\mu_Z^\tau$ and an
independent noisy query $Z=X+\varepsilon$ from the same task.
Its goal is to reconstruct the query's clean counterpart~$X$.
The predictor is
evaluated through the in-context risk
\[
\cR_{\rm ICL}(f)
=\mathbb E_{\tau}
  \mathbb E\!\left[\|X-f[\mu^\tau_Z](Z)\|^2\right].
\]
The attention weights are shared across tasks; adaptation occurs
through the context argument $\mu_Z^\tau$.
We therefore seek one set of weights satisfying
$\hat f[\mu_Z^\tau](z)=\mathbb E[X\mid Z=z,\tau]$ for every admissible
task $\tau$, without retraining when its means or proportions change.

To solve this problem, the attention mechanism must infer from the context both the barycentric combination of the component centroids and the linear contraction required to reproduce the Bayes-optimal denoiser~$f^\star$ in~\eqref{eq:bayes-denoiser}. Unlike in the previous setting, however, the attention weights can no longer depend on the task parameters $(m_1,m_2,\pi_1)$, as these vary across tasks. 
We therefore add a residual connection allowing the network to adjust the linear term and the barycentric mean separately. We define
$T^{\rm res}_{U,V}[\mu](z)=z+T_{U,V}[\mu](z)$
and combine two such residual heads through a linear output map.

\begin{proposition}[Exact in-context adaptation]
\label{prop:icl}
\label{prop:residual-icl-denoiser}
For every admissible prompt law (such that $m_1\neq m_2$, $\|m_1\|=\|m_2\|$, and $\pi_1\in(0,1)$ almost surely), the architecture
\[
\widehat f[\mu](z)=W_{\rm out}
\begin{pmatrix}T^{\rm res}_{U_1,V_1}[\mu](z)\\T^{\rm res}_{U_2,V_2}[\mu](z)\end{pmatrix}
\]
equals its Bayes denoiser with the task-independent parameters

\[
\begin{alignedat}{2}
U_1&=S^{-1}, \qquad &V_1&=\frac{\xi^2}{\gamma^2+\xi^2}\Id,\\
U_2&=V_2=0, \qquad &W_{\rm out}&=\left(\Id,-\frac{2\xi^2}{\gamma^2+\xi^2}\Id\right).
\end{alignedat}
\]
\end{proposition}
The proof is given in Appendix~\ref{app:icl-proof}. 
The task-dependent means and weights never enter the layer parameters. The
first head reads them through the context distribution and returns their weighted barycenter; the second head provide copies of the query, allowing the output map to adjust the linear shrinkage independently. In fact, the parameters encode invariants shared across tasks, while attention extracts the latent structure specific to the current prompt.
This proposal shows that a single-layer softmax attention architecture with residual connections can perform in-context denoising. Interestingly, it also highlights how residual connections can enhance the expressivity of attention-based architectures; in contrast to neural networks, where their primary role is to facilitate optimization.

\section{Conclusion}

Gaussian mixtures provide a tractable setting in which context
selection has a concrete statistical role. For classification,
the Bayes score is affine, yet the nonlinear attention
parameterization creates a rich optimization geometry.
For denoising, query-dependent selection becomes essential,
and residual connections allow the same parameters to serve
different mixture distributions. Together, these results connect
specific features of attention architectures to the statistical
operations they perform. Extending this understanding to finite
prompts and explaining how training discovers such adaptive
solutions are natural next steps.


\subsection*{AI use statement}

In this work, we used generative AI for the following tasks: literature review, including identifying articles relevant to our research; idea exploration, through prompts aimed at investigating some of our initial leads and assessing their potential relevance; calculation and proof assistance; programming assistance; and language polishing. All AI-assisted work was reviewed and verified by the authors. In particular, we independently reviewed the cited articles, further investigated the ideas we found relevant, and checked AI-assisted calculations and proofs ourselves. The problem studied and our research strategies were developed through discussions among the authors, rather than generated by AI. We take full responsibility for the final content of this work, including all text, claims, calculations, and other artifacts produced with the assistance of generative AI.

\subsection*{Acknowledgment}

This work was supported by the French Ministry of Higher Education, Research and Space (MESRE) through an AMX doctoral fellowship awarded by École Polytechnique.

\bibliographystyle{plainnat}
\bibliography{references_hal}

\begin{thebibliography}{21}
\providecommand{\natexlab}[1]{#1}
\providecommand{\url}[1]{\texttt{#1}}
\expandafter\ifx\csname urlstyle\endcsname\relax
  \providecommand{\doi}[1]{doi: #1}\else
  \providecommand{\doi}{doi: \begingroup \urlstyle{rm}\Url}\fi

\bibitem[Banyaga and Hurtubise(2004)]{BANYAGA2004365}
Augustin Banyaga and David~E. Hurtubise.
\newblock A proof of the morse--bott lemma.
\newblock \emph{Expositiones Mathematicae}, 22\penalty0 (4):\penalty0 365--373,
  2004.

\bibitem[Bohbot et~al.(2026)Bohbot, Letrouit, Peyr{\'e}, and
  Vialard]{bohbot2026}
L{\'e}a Bohbot, Cyril Letrouit, Gabriel Peyr{\'e}, and Fran{\c{c}}ois-Xavier
  Vialard.
\newblock Token sample complexity of attention.
\newblock In \emph{Forty-third International Conference on Machine Learning},
  2026.

\bibitem[Boursier and Boyer(2026)]{boursier2026}
Etienne Boursier and Claire Boyer.
\newblock Softmax as linear attention in the large-prompt regime: a
  measure-based perspective.
\newblock In \emph{Forty-third International Conference on Machine Learning},
  2026.

\bibitem[Castin et~al.(2025)Castin, Ablin, Carrillo, and Peyr{\'e}]{castin2025}
Val{\`e}re Castin, Pierre Ablin, Jos{\'e}~A. Carrillo, and Gabriel Peyr{\'e}.
\newblock A unified perspective on the dynamics of deep transformers.
\newblock \emph{Foundations of Computational Mathematics}, 2025.

\bibitem[Duranthon et~al.(2026)Duranthon, Marion, Boyer, Loureiro, and
  Zdeborov{\'a}]{duranthon2026statistical}
Odilon Duranthon, Pierre Marion, Claire Boyer, Bruno Loureiro, and Lenka
  Zdeborov{\'a}.
\newblock Statistical advantage of softmax attention: Insights from
  single-location regression.
\newblock In \emph{International Conference on Learning Representations}, 2026.

\bibitem[Frei and Vardi(2025)]{frei2025}
Spencer Frei and Gal Vardi.
\newblock Trained transformer classifiers generalize and exhibit benign
  overfitting in-context.
\newblock In \emph{International Conference on Learning Representations}, 2025.

\bibitem[Geshkovski et~al.(2023)Geshkovski, Letrouit, Polyanskiy, and
  Rigollet]{geshkovski2023}
Borjan Geshkovski, Cyril Letrouit, Yury Polyanskiy, and Philippe Rigollet.
\newblock The emergence of clusters in self-attention dynamics.
\newblock In \emph{Advances in Neural Information Processing Systems},
  volume~36, pages 57026--57037, 2023.

\bibitem[Goel et~al.(2026)Goel, Soltanolkotabi, and Bartlett]{goel2026training}
Gautam Goel, Mahdi Soltanolkotabi, and Peter Bartlett.
\newblock Training dynamics of softmax self-attention: Fast global convergence
  via preconditioning.
\newblock \emph{arXiv preprint arXiv:2603.01514}, 2026.

\bibitem[Guillemin and Pollack(1974)]{guillemin1974differential}
Victor Guillemin and Alan Pollack.
\newblock \emph{Differential Topology}.
\newblock Prentice-Hall, 1974.

\bibitem[Hastie et~al.(2009)Hastie, Tibshirani, Friedman, and
  Friedman]{hastie2009elements}
Trevor Hastie, Robert Tibshirani, Jerome~H Friedman, and Jerome~H Friedman.
\newblock \emph{The elements of statistical learning: data mining, inference,
  and prediction}, volume~2.
\newblock Springer, 2009.

\bibitem[Li et~al.(2026)Li, Min, and Vidal]{li2026}
Hongkang Li, Hancheng Min, and Rene Vidal.
\newblock Transformers learn the optimal {DDPM} denoiser for multi-token
  {GMMs}, 2026.

\bibitem[Maulen-Soto and Boyer(2026)]{maulen2026}
Rodrigo Maulen-Soto and Claire Boyer.
\newblock Attention-based {PCA}, 2026.

\bibitem[Rasmussen and Williams(2006)]{Rasmussen2006Gaussian}
Carl~Edward Rasmussen and Christopher K.~I. Williams.
\newblock \emph{Gaussian Processes for Machine Learning}.
\newblock MIT Press, 2006.

\bibitem[Rebjock and Boumal(2025)]{rebjock2024fastconvergencenonisolatedminima}
Quentin Rebjock and Nicolas Boumal.
\newblock Fast convergence to non-isolated minima: Four equivalent conditions
  for {$C^2$} functions.
\newblock \emph{Mathematical Programming}, 213\penalty0 (1):\penalty0 151--199,
  2025.

\bibitem[Rosu et~al.(2025)Rosu, Carin, and Cheng]{rosu2025}
Paul Rosu, Lawrence Carin, and Xiang Cheng.
\newblock From softmax to score: Transformers can effectively implement
  in-context denoising steps.
\newblock In \emph{Advances in Neural Information Processing Systems},
  volume~38, 2025.

\bibitem[Sander et~al.(2022)Sander, Ablin, Blondel, and Peyr{\'e}]{sander2022}
Michael~E. Sander, Pierre Ablin, Mathieu Blondel, and Gabriel Peyr{\'e}.
\newblock Sinkformers: Transformers with doubly stochastic attention.
\newblock In \emph{Proceedings of the 25th International Conference on
  Artificial Intelligence and Statistics}, volume 151, pages 3515--3530. PMLR,
  2022.

\bibitem[Shen et~al.(2025)Shen, Zhou, Yang, and Shen]{shen2025}
Wei Shen, Ruida Zhou, Jing Yang, and Cong Shen.
\newblock On the training convergence of transformers for in-context
  classification of gaussian mixtures.
\newblock In \emph{Proceedings of the 42nd International Conference on Machine
  Learning}, volume 267, pages 54732--54771. PMLR, 2025.

\bibitem[Smart et~al.(2025)Smart, Bietti, and Sengupta]{smart2025}
Matthew Smart, Alberto Bietti, and Anirvan~M. Sengupta.
\newblock In-context denoising with one-layer transformers: Connections between
  attention and associative memory retrieval.
\newblock In \emph{Proceedings of the 42nd International Conference on Machine
  Learning}, volume 267, pages 55950--55971. PMLR, 2025.

\bibitem[Smart et~al.(2026)Smart, Ganguly, Metya, Morozov, and
  Sengupta]{smart2026}
Matthew Smart, Soumya Ganguly, Nilava Metya, Alexandre~V. Morozov, and
  Anirvan~M. Sengupta.
\newblock Attention as in-context empirical bayes: A two-stage view via
  particle dynamics, 2026.

\bibitem[Vuckovic et~al.(2020)Vuckovic, Baratin, and Tachet~des
  Combes]{vuckovic2020}
James Vuckovic, Aristide Baratin, and Remi Tachet~des Combes.
\newblock A mathematical theory of attention, 2020.

\bibitem[Zhang and Cao(2026)]{zhang2026}
Chenyang Zhang and Yuan Cao.
\newblock Transformers efficiently perform in-context logistic regression via
  normalized gradient descent.
\newblock In \emph{Forty-third International Conference on Machine Learning},
  2026.

\end{thebibliography}

\clearpage
\appendix
\thispagestyle{empty}


\part*{Appendix}
\addcontentsline{toc}{part}{Appendix}
\renewcommand{\contentsname}{Table of Contents}

\etocsetnexttocdepth{subsection}
\begingroup
\setlength{\parskip}{0pt}

\etocsettocstyle
  {\section*{Table of Contents}
   \noindent\rule{\linewidth}{0.4pt}\par\vspace{-0.3em}}
  {}

\localtableofcontents

\noindent\rule{\linewidth}{0.4pt}
\endgroup

\paragraph{Where to find each main-text proof.}
The table below locates every numbered result in the main text.
Titles distinguish complete proofs, preparatory arguments, and complements.
For results assembled from several lemmas, the indicated location contains
an explicitly named proof that identifies the ingredients.

\medskip
\noindent\begin{tabularx}{\textwidth}{@{}lX@{}}
\toprule
\textbf{Main-text result} & \textbf{Proof location} \\
\midrule
Proposition~\ref{prop:gmm-attention} & Appendix~\ref{app:gmm-attention-proof}. \\
Theorem~\ref{thm:lda-landscape} & Appendix~\ref{app:lda-effective-proof-block}, using the minimizers identified in Appendix~\ref{app:lda-minimizers-block}. \\
Lemma~\ref{lem:lda-regularity-main} & Appendix~\ref{app:lda-regularity-proof-block}. \\
Theorem~\ref{thm:lda-landscape-initial} & Appendix~\ref{app:lda-factorized-proof-block}, using the minimizer preimage in Appendix~\ref{app:lda-minimizers-block}. \\
Proposition~\ref{prop:lda-escape-main} & Appendix~\ref{app:lda-escape-proof-block}. \\
Theorem~\ref{thm:lda-dynamics} & Appendix~\ref{app:lda-local-proof}. \\
Theorem~\ref{thm:lda-krylov} & Appendix~\ref{app:lda-krylov-proof}. \\
Theorem~\ref{thm:denoise-expressivity} & Appendix~\ref{app:denoise-expressivity-proof}; detailed arguments in Appendices~\ref{app:one-head-obstruction}, \ref{app:two-head-nondegenerate}, and~\ref{app:linear-attention-gap}. \\
Proposition~\ref{prop:phase-transition} & Appendix~\ref{app:phase-transition-proof}. \\
Proposition~\ref{prop:icl} & Appendix~\ref{app:icl-proof}. \\
\bottomrule
\end{tabularx}
\medskip

\section{Proofs and Complements for Section~\ref{sec:mixtures}}
\label{app:mixtures}

\subsection{Gaussian-mixture attention (proof of Proposition~\ref{prop:gmm-attention})}
\label{app:gmm-attention-proof}

\begin{proof}
Set \(u\coloneqq Uz\). Then
\[
T_{U,V}[\mu](z)
=
V\,
\frac{
\mathbb E\left[X e^{\langle u,X\rangle}\right]
}{
\mathbb E\left[e^{\langle u,X\rangle}\right]
}.
\]
It is well known that the moment-generating function of a Gaussian vector is given by $\mathbb E\left[e^{\langle u,X\rangle}\right]=\exp\left(
\langle u,m\rangle
+\frac12 u^\top \Gamma u
\right)$. By differentiation, we obtain $\mathbb E\left[ Xe^{\langle u,X\rangle}\right] = \left(m+\Gamma u\right)
\exp\left(
\langle u,m\rangle
+\frac12 u^\top \Gamma u
\right)$

Thus, for a Gaussian mixture \(X\sim \mu=\sum \pi_k \mu_k\), one has 
\[
\mathbb E\left[e^{\langle u,X\rangle}\right]
=
\sum_{k=1}^M
\pi_k
\exp\left(
\langle u,m_k\rangle
+\frac12 u^\top \Gamma_k u
\right),
\]
and
\[
\mathbb E\left[X e^{\langle u,X\rangle}\right]
=
\sum_{k=1}^M
\pi_k
\left(m_k+\Gamma_k u\right)
\exp\left(
\langle u,m_k\rangle
+\frac12 u^\top \Gamma_k u
\right).
\]
Therefore,
\[
T_{U,V}[\mu](z)
=
V
\sum_{k=1}^M
\alpha_k(z)
\left(m_k+\Gamma_k u\right),
\]
where
\(
\alpha_k(z)
=
\frac{\displaystyle 
\pi_k
\exp\left(
\langle u,m_k\rangle
+\frac12 u^\top \Gamma_k u
\right)
}{
\displaystyle
\sum_{\ell=1}^M
\pi_\ell
\exp\left(
\langle u,m_\ell\rangle
+\frac12 u^\top \Gamma_\ell u
\right)
}
\)
which concludes the proof.
\end{proof}

\subsection{General mixtures (extension of Proposition~\ref{prop:gmm-attention})}
\label{app:general-mixture-attention}

For a positive alignment function $a$, define the normalized attention map
\[
T_a[\mu](z)
=\frac{\int a(z,z')z'\,\mathrm d\mu(z')}
       {\int a(z,z')\,\mathrm d\mu(z')}.
\]

\begin{proposition}[Attention on a general mixture]
\label{prop:attention_op_for_general_mixture}
Let $\mu=\sum_{k=1}^M\pi_k\mu_k$. Then, for every $z\in\mathbb R^d$,
\[
T_a[\mu](z)=\sum_{k=1}^M r_k(z)T_a[\mu_k](z),
\qquad
r_k(z)=\pi_k
\frac{\int a(z,z')\,\mathrm d\mu_k(z')}
     {\int a(z,z')\,\mathrm d\mu(z')}.
\]
\end{proposition}

\begin{proof}
This proposition follows directly from the fact that
\[
\int f \mathrm{d}\mu=\int f \sum_k\pi_k\mathrm{d}\mu_k.
\]
\end{proof}

\section{Proofs and Complements for Section~\ref{sec:lda}}
\label{appendix:attention_for_LDA}

This appendix develops the representation, landscape, and convergence arguments underlying Section~\ref{sec:lda}. We first derive the attention predictor and identify the population target of logistic training. We then separate two questions: which parameters minimize the risk, and whether gradient flow approaches such parameters. The distinction matters because a favorable finite critical-point structure does not control trajectories at infinity.

The landscape argument is developed in Appendix~\ref{app:lda-landscape-proof}; the subsequent analysis establishes local convergence near the regular Bayes manifold and global convergence for a prescribed initialization in a symmetric invariant subspace. The auxiliary calculations are collected before the arguments that use them. Throughout, $\widetilde{\mathcal R}$ denotes the effective LDA risk and $\mathcal R=\widetilde{\mathcal R}\circ\Phi$ its expression in attention parameters.

\subsection{Effective LDA predictor (derivation of Equations~\eqref{eq:effective}--\eqref{eq:Phi})}
\label{appendix:aux_lemma_LDA_and_attention}
\label{app:lda-effective-model}
For parameter matrices $U,V\in \mathbb{R}^{(d+1)\times (d+1)}$, we consider the block structure:
\begin{equation}\label{eq:blockstructure}
U^\top=\left[
\begin{array}{c|c}
U_{11} & u_{12} \\[2pt] \hline \\[-8pt]
u_{21}^\top  & u_{22}
\end{array}\right],
\qquad
V=\left[
\begin{array}{c|c}
V_{11} & v_{12} \\[2pt] \hline \\[-8pt]
v_{21}^\top  & v_{22}
\end{array}\right],
\end{equation}
where $U_{11}, V_{11}\in\mathbb{R}^{d\times d}$, $u_{12},u_{21}, v_{12}, v_{21}\in\mathbb{R}^d$ are $U_{22},v_{22} \in \mathbb{R}$.

Equation~\eqref{eq:effective} is a direct consequence of the following lemma.

\begin{lemma}
\label{lem:explicit_form_attention_lda}
\label{lem:attention_predictor_lda}
Let $\mu_{\mathrm{LDA}}$ denote the joint law defined by
Equation~\eqref{eq:lda-model}. For every \((x,y)\in\mathbb R^d\times\{-1,0,1\}\), one has
\[
T_{U,V}[\mu_{\mathrm{LDA}}](x,y)
=
V
\begin{pmatrix}
\Gamma\big(U_{11}^\top x + y\,u_{21}\big)
+
(1-\alpha(x,y))\,m_{-1}
+
\alpha(x,y)\,m_1
\\[1mm]
2\alpha(x,y)-1
\end{pmatrix},
\]
where
\[
\alpha(x,y)
=
\sigma\Big(
\log\!\frac{\pi_+}{1-\pi_+}
+2\bigl(u_{12}^\top x+y\,u_{22}\bigr)
+\big\langle U_{11}^\top x+y\,u_{21},\,m_1-m_{-1}\big\rangle
\Big),
\]
with
\(
\sigma(t)\coloneqq \frac{1}{1+e^{-t}}
\) the sigmoid function.
\end{lemma}

\begin{proof}
For $z=(x,y)$ and $z'=(x',y')$, we then have
\[
z^\top U^\top z'
=
x^\top U_{11}x' + x^\top u_{12}\,y' + y\,u_{21}^\top x' + y\,u_{22}\,y'.
\]
Hence,
\[
z^\top U^\top z'
=
\langle s(x,y),x'\rangle + t(x,y)\,y',
\]
where the linear quantities are explicitly given by
\[
s(x,y)\coloneqq U_{11}^\top x + y\,u_{21}\in\mathbb R^d,
\qquad
t(x,y)\coloneqq u_{12}^\top x + y\,u_{22}\in\mathbb R.
\]

\paragraph{Denominator.}
For any $s\in\mathbb R^d$ and any $c\in\{-1,1\}$, the moment-generating function of a Gaussian vector gives
\[
\mathbb E\big[e^{\langle s,X\rangle}\mid Y=c\big]
=
\exp\!\left(\langle s,m_c\rangle + \frac12 s^\top\Gamma s\right).
\]
Therefore,
\begin{align*}
\int \exp\!\big( z^\top U^\top z\big)\,\mathrm{d}\mu_{\mathrm{LDA}}(z')
&=\mathbb{E}_{(X,Y)\sim\mu_{\mathrm{LDA}}}\!\left[\exp\!\big(\langle s(x,y),X\rangle + t(x,y)\,Y\big)\right]
\\
&=
\sum_{c\in\{-1,1\}} \pi_c
\exp\!\big(c\,t(x,y)\big)\,
\mathbb E\!\left[e^{\langle s(x,y),X\rangle}\mid Y=c\right]
\\
&=
\exp\!\left(\frac12 s(x,y)^\top \Gamma s(x,y)\right)
\Big(
(1-\pi_+)\,\exp\!\big(-t(x,y)+\langle s(x,y),m_{-1}\rangle\big)
\\
&\hspace{3.5cm}
+
\pi_+\,\exp\!\big(t(x,y)+\langle s(x,y),m_1\rangle\big)
\Big).
\end{align*}

\paragraph{Numerator.}
Similarly, for any $c\in\{-1,1\}$,
\[
\mathbb E\big[X\,e^{\langle s,X\rangle}\mid Y=c\big]
=
\big(m_c+\Gamma s\big)
\exp\!\left(\langle s,m_c\rangle + \frac12 s^\top\Gamma s\right).
\]
Therefore,
\begin{align*}
\int \exp\!\big( z^\top U ^\top z'\big)\,z'\,\mathrm{d}\mu_{\mathrm{LDA}}(z')
&=
\sum_{c\in\{-1,1\}} \pi_c
\exp\!\big(c\,t(x,y)\big)\,
\mathbb E\!\left[ \begin{pmatrix}
X\\
c
\end{pmatrix}e^{\langle s(x,y),X\rangle}\mid Y=c\right]
\\
&=
\exp\!\left(\frac12 s^\top\Gamma s\right)
\Bigg[
(1-\pi_+)\,\exp\!\big(-t+\langle s,m_{-1}\rangle\big)
\begin{pmatrix}
m_{-1}+\Gamma s\\
-1
\end{pmatrix}
\\
&\hspace{2.7cm}
+
\pi_+\,\exp\!\big(t+\langle s,m_1\rangle\big)
\begin{pmatrix}
m_1+\Gamma s\\
1
\end{pmatrix}
\Bigg],
\end{align*}
where we wrote $s=s(x,y)$ and $t=t(x,y)$ in the last inequality. 
Dividing numerator and denominator, we obtain
\[
\frac{
\displaystyle \int \exp\!\big(\langle Uz,z'\rangle\big)\,z'\,\mathrm{d}\mu(z')
}{
\displaystyle \int \exp\!\big(\langle Uz,z'\rangle\big)\,\mathrm{d}\mu(z')
}
=
\begin{pmatrix}
\Gamma s(x,y) + (1-\alpha(x,y))\,m_{-1} + \alpha(x,y)\,m_1\\
2\alpha(x,y)-1
\end{pmatrix},
\]
where
\[
\alpha(x,y)
\coloneqq 
\frac{
\pi_+\exp\!\big(t(x,y)+\langle s(x,y),m_1\rangle\big)
}{
(1-\pi_+)\,\exp\!\big(-t(x,y)+\langle s(x,y),m_{-1}\rangle\big)
+
\pi_+\exp\!\big(t(x,y)+\langle s(x,y),m_1\rangle\big)
}.
\]
Equivalently,
\begin{equation*}
    \alpha(x,y) =
\sigma\left(
\log\!\frac{\pi_+}{1-\pi_+}
+2\,t(x,y)
+\langle s(x,y),m_1-m_{-1}\rangle
\right),
\end{equation*}
which allows to conclude by plugging the values of $s(x,y)$ and $t(x,y)$.
\end{proof}

\subsection{Bayes realizability (preparation for Theorems~\ref{thm:lda-landscape} and~\ref{thm:lda-landscape-initial})}
\label{app:lda-realizability}

We recall the form of the optimal preidctor in the LDA model, which is a classical result in statistical learning theory, see e.g., \citet{hastie2009elements}.
\begin{proposition}[Bayes classifier in the LDA model]
\label{thm:bayes_classifier_lda}
Under the model~\eqref{eq:lda-model}, the Bayes classifier is
\(g^\star(x)=\operatorname{sign}(f^\star(x))\), where
\[
f^\star(x)=\langle w^\star,x\rangle+b^\star,
\qquad
w^\star=\Gamma^{-1}(m_1-m_{-1}),
\]
and
\[
b^\star=-\frac12\left(m_1^\top\Gamma^{-1}m_1
-m_{-1}^\top\Gamma^{-1}m_{-1}\right)
+\log\frac{\pi_+}{1-\pi_+}.
\]
\end{proposition}

\begin{lemma}[Existence of Bayes-realizing attention parameters]
\label{lem:existence_oracle_parameters}
There exists an attention parameter \(\theta^\star\) such that
\(f_{\theta^\star}(x)=f^\star(x)\) for every \(x\in\mathbb R^d\).
Consequently, the attention class attains the Bayes logistic risk.
\end{lemma}

\begin{proof}[Proof of Lemma~\ref{lem:existence_oracle_parameters}.]
We recall that, in the LDA model, the Bayes score is affine and can be written as
\[
f^\star(x)=\langle w^\star,x\rangle + b^\star.
\]
Equation~\eqref{eq:effective} gives \(f_\theta\) in terms of the effective parameters:
\[
f_\theta(x)
=
\underbrace{
v_{21}^\top \Gamma U_{11}^\top x
+ v_{21}^\top m_{-1}
- v_{22}
}_{\text{linear part}}
+
\underbrace{
\alpha(x)\Big(
v_{21}^\top (m_1-m_{-1}) + 2v_{22}
\Big)
}_{\text{nonlinear part}},
\]
where
\[
\alpha(x)
=
\sigma\left(
\log\!\frac{\pi_+}{1-\pi_+}
+
\left(2u_{12}+U_{11}(m_1-m_{-1})\right)^\top x
\right).
\]

We start by imposing the constraint
\begin{equation*}
\tag{C}
\label{eq:constraint_expressivity}
    2u_{12}+U_{11}(m_1-m_{-1})=0,
\qquad\text{that is,}\qquad
u_{12}=-\frac12\,U_{11}(m_1-m_{-1}).
\end{equation*}

With this choice, the argument of the sigmoid no longer depends on \(x\), and therefore
\[
\alpha(x)
=
\sigma\left(
\log\!\frac{\pi_+}{1-\pi_+}
+
\bigl(2u_{12}+U_{11}(m_1-m_{-1})\bigr)^\top x
\right)
=
\sigma\left(\log\!\frac{\pi_+}{1-\pi_+}\right)
=
\pi_+.
\]
Hence \(f_\theta\) becomes affine:
\[
f_\theta(x)
=
v_{21}^\top \Gamma U_{11}^\top x
+
v_{21}^\top\bar m_{\pi_+}
+
(2\pi_+-1)\,v_{22},
\qquad
\bar m_{\pi_+}\coloneqq (1-\pi_+) m_{-1}+\pi_+ m_1.
\]

\textbf{Case 1.} Assume first that \(\pi_+\neq \tfrac12\). Then, the coefficient of \(v_{22}\) is nonzero, so the intercept can be adjusted freely. More precisely, given any \(w^\star\in\mathbb R^d\) and \(b^\star\in\mathbb R\), one may choose
\[
U_{11}=I_d,
\qquad
u_{12}=-\frac12(m_1-m_{-1}),
\qquad
v_{21}=\Gamma^{-1}w^\star,
\]
and then define
\[
v_{22}
=
\frac{
b^\star-(\Gamma^{-1}w^\star)^\top\bar m_{\pi_+}
}{2\pi_+-1}.
\]
With this choice, we obtain
\[
f_\theta(x)=\langle w^\star,x\rangle+b^\star.
\]

\textbf{Case 2.} Assume now that \(\pi_+=\tfrac12\). In that case, still under constraint \eqref{eq:constraint_expressivity}
\[
\alpha(x)\equiv \frac12,
\qquad
\bar m_{\pi_+}=\frac{m_{-1}+m_1}{2},
\]
and the term involving \(v_{22}\) disappears:
\[
f_\theta(x)
=
v_{21}^\top \Gamma U_{11}^\top x
+
v_{21}^\top\frac{m_{-1}+m_1}{2}.
\]
\textbf{Case 2(a).} If \(m_1\neq -m_{-1}\), that is, if $\bar m\coloneqq \frac{m_{-1}+m_1}{2}\neq 0$, one may choose \(v_{21}\neq0\) so that $v_{21}^\top \bar m=b^\star$ and then choose \(U_{11}\) so that $U_{11}\Gamma v_{21}=w^\star.$
For instance, say $v_{21}=\frac{b^\star}{\|\bar m\|^2}\,\bar m$ (if $b^\star=0$, one may choose any $v_{21}\in \operatorname{Span}(\bar m)^\perp \backslash \{0\}$), and define \(a\coloneqq \Gamma v_{21}\neq 0\). One may take $U_{11}=\frac{w^\star\,a^\top}{\|a\|^2}$ so that we have
\[
f_\theta(x)=\langle w^\star,x\rangle+b^\star.
\]

\textbf{Case 2(b).} Finally, if \(\pi_+=\tfrac12\) and \(m_1=-m_{-1}\), then \(\bar m=0\) and
\[
f_\theta(x)=v_{21}^\top \Gamma U_{11}^\top x.
\]
In the symmetric balanced case considered here, the target intercept is
\[
b^\star
=
-\frac12\Bigl(m_1^\top\Gamma^{-1}m_1-m_{-1}^\top\Gamma^{-1}m_{-1}\Bigr)
+\log\!\frac{\pi_+}{1-\pi_+}
=0,
\]
so there is no obstruction: it is enough to choose
\[
U_{11}\Gamma v_{21}=w^\star,
\]
for instance $U_{11}=I_d$ and $v_{21}=\Gamma^{-1}w^\star$.

\textbf{Remark.} This choice of parameters is not unique. Another possible strategy, when \(m_1\neq -m_{-1}\), is to first cancel the amplitude in front of the sigmoid term and then adjust the remaining coefficients accordingly.

\end{proof}

\subsection{Population logistic-risk minimizer (preparation for Theorem~\ref{thm:lda-landscape})}
\label{app:lda-logistic-target}

Before studying critical points in parameter space, we identify the predictor that training seeks to recover. The logistic loss is particularly well suited to this purpose: its unrestricted population minimizer is the Bayes predictor, which the attention model can represent exactly. Thus the optimal risk is the same in the effective model, in the original attention model, and over all measurable scores.

\begin{lemma}
\label{lem:lda_logistic_risk}
Over all measurable scores, the population logistic risk is uniquely minimized
almost surely by the Bayes predictor $f^\star$ defined in
Equation~\eqref{eq:lda-bayes}.
\end{lemma}

\begin{proof}
 We want to minimize $\mathcal R_{\mathrm{LDA}}(f)$ over the set of measurable functions.
\begin{align*}
\mathcal R_{\mathrm{LDA}}(f)= \mathbb E_{(X,Y)\sim \mu_{\mathrm{LDA}}}
\bigl[\ell_{\log}(Y,f(X))\bigr] = \mathbb E_X\bigl[\ell_{\log}(1,f(X)) \eta(X)+ \ell_{\log}(-1,f(X)) (1-\eta(X))\bigr]
\end{align*}
We now define $g_X(t)\coloneqq \ell_{\log}(1,t) \eta(X)+ \ell_{\log}(-1,t) (1-\eta(X))$. Using that $
\ell'(y,t)
=
-\frac{y}{1+e^{yt}}
$, we obtain $g'_X(t)=\sigma(t)-\eta(X)$. Using the strict convexity and coercivity of $g_X$ (as long as $\eta(X)\notin\{0,1\}$ which is almost sure), we obtain by differentiation that :
\[
t(X)\in\arg \min g_X \quad \Longleftrightarrow \quad t(X)=\log\left(\frac{\eta(X)}{1-\eta(X)}\right)
\]

Finally :
\[
\mathcal R_{\mathrm{LDA}}(f)=  \mathbb E_X\bigl[g_X(f(X))\bigr]\geq \mathbb E_X(g_X(t(X)))=\mathcal R_{\mathrm{LDA}}(f^\star)
\]
where the inequality is an equality if and only if 
\[
f(X)=\log\left(\frac{\eta(X)}{1-\eta(X)}\right) \quad \text{a.s.}
\]

Moreover, this quantity is well-defined since $0<\eta(X)<1$ almost surely in the LDA setting.
\end{proof}

\begin{remark} 
This property is specific to the logistic loss in the present setting. Indeed, if one considers instead the quadratic loss, the corresponding population minimizer is
\[
f_{\mathrm{sq}}(x)=\mathbb E[Y\mid X=x]=2\eta(x)-1.
\]
Although $f_{\mathrm{sq}}$ yields the same Bayes classifier after thresholding at $0$, there is in general no reason for $f_{\mathrm{sq}}$ to belong to the class $\mathcal F\coloneqq \{f_\theta,\: \theta\in \Theta \}$. In other words, even if $\mathcal F$ contains the logistic target
$f_{\mathrm{log}}$, 
it does not necessarily contain the quadratic target $f_{\mathrm{sq}}$. Hence, minimizing the quadratic risk over $\mathcal F$ does not in general allow one to recover the global minimizer of the quadratic loss, thus the optimal classifier.
\end{remark}

\subsection{Differential identities (tools for the LDA landscape and convergence proofs)}
\label{app:lda-differential-tools}

The arguments below use three kinds of calculations: derivatives of the logistic risk, the Hessian at a Bayes realization, and the differential of the map $\Phi$. We collect them here to keep the subsequent landscape argument focused on their geometric consequences. Readers interested first in the structure of the critical set may proceed to Appendix~\ref{app:lda-landscape-proof} and return to these identities as needed.

The first Gaussian identity will also be used to prove coercivity of the linear risk in the symmetric setting. It quantifies the contribution of observations on the wrong side of a candidate linear decision boundary.

\begin{lemma}[Positive part of a Gaussian]
\label{lem:positive_part_gaussian}
Let $X\sim \mathcal N(\mu,\sigma^2)$ with $\sigma>0$, and let
$X_+ \coloneqq  \max(X,0)$. If $\varphi$ and $\Phi$ denote respectively the
density and distribution function of a standard normal random variable, then
\[
    \mathbb E[X_+]
    =
    \mu \Phi\!\left(\frac{\mu}{\sigma}\right)
    +
    \sigma \varphi\!\left(\frac{\mu}{\sigma}\right).
\]
\end{lemma}

\begin{proof}
Let $c=-\mu/\sigma$ and $Z=(X-\mu)/\sigma\sim \mathcal N(0,1)$.
\begin{align*}
\mathbb E[X_+]
&= \int_0^\infty x\,\frac{1}{\sigma}
   \varphi\!\left(\frac{x-\mu}{\sigma}\right)\,dx \\
&= \int_c^\infty (\mu+\sigma z)\varphi(z)\,dz \\
&= \mu\int_c^\infty \varphi(z)\,dz
   +\sigma\int_c^\infty z\varphi(z)\,dz \\
&= \mu(1-\Phi(c))
   +\sigma[-\varphi(z)]_{c}^{\infty},
   \qquad \text{since } \varphi'(z)=-z\varphi(z), \\
&= \mu\Phi\!\left(\frac{\mu}{\sigma}\right)
   +\sigma\varphi\!\left(\frac{\mu}{\sigma}\right),
\end{align*}
using $1-\Phi(-u)=\Phi(u)$ and $\varphi(-u)=\varphi(u)$.
\end{proof}

\paragraph{From prediction error to risk derivatives.}
Conditioning on the query turns the derivative of the logistic loss into the discrepancy between the model probability $\sigma(f_\theta(X))$ and the true posterior $\eta(X)$. This identity is the common starting point for the gradient and Hessian calculations.

\begin{lemma}[A useful identity for logistic loss]
\label{lem:identity_log_loss_derivative}
For the logistic loss, conditioning on $X$ gives
\[
\mathbb E\left[
\ell'\bigl(Y,f_\theta(X)\bigr)\mid X
\right]
=
\sigma(f_\theta(X))-\eta(X)
\]
\end{lemma}

\begin{proof}
For the logistic loss we have
$
\ell'(y,t)
=
-\frac{y}{1+e^{yt}}.
$
Since $Y\in\{-1,1\}$ and $\eta(X)=\mathbb P(Y=1\mid X)$, conditioning on $X$ yields
\begin{align*}
\mathbb E\left[\ell'\bigl(Y,f_\theta(X)\bigr)\mid X\right]
&=
-\eta(X)\frac{1}{1+e^{f_\theta(X)}}
+(1-\eta(X))\frac{1}{1+e^{-f_\theta(X)}} \\
&=
-\eta(X)\bigl(1-\sigma(f_\theta(X))\bigr)
+(1-\eta(X))\sigma(f_\theta(X)) \\
&=
\sigma(f_\theta(X))-\eta(X).
\end{align*}
\end{proof}

\begin{lemma}[First and second variations of the model]
\label{lem:first_second_variations_model}
Set
\[
\Delta m \coloneqq  m_1-m_{-1},
\qquad
z_\theta(x)\coloneqq s_0+s(\theta)^\top x = s_0+(2u_{12}+U_{11}\Delta m)^\top x.
\]
Recall that
\[
f_\theta(x)
=
v_{21}^\top \Gamma U_{11}^\top x
+
v_{21}^\top m_{-1}
-
v_{22}
+
\bigl(v_{21}^\top \Delta m+2v_{22}\bigr)
\sigma\bigl(z_\theta(x)\bigr).
\]
Let
\[
h=(\delta v_{21},\delta v_{22},\delta U_{11},\delta u_{12}).
\]
Then
\[
\begin{aligned}
Df_\theta(x)[h]
={}&
\delta v_{21}^\top \Gamma U_{11}^\top x
+
v_{21}^\top \Gamma \delta U_{11}^\top x
+
\delta v_{21}^\top m_{-1}
-
\delta v_{22}
\\
&+
\bigl(\delta v_{21}^\top \Delta m+2\delta v_{22}\bigr)
\sigma(z_\theta(x))
\\
&+
\bigl(v_{21}^\top \Delta m+2v_{22}\bigr)
\sigma'(z_\theta(x))
\bigl(2\delta u_{12}+\delta U_{11}\Delta m\bigr)^\top x.
\end{aligned}
\]
Moreover,
\[
\begin{aligned}
D^2f_\theta(x)[h,h]
={}&
2\delta v_{21}^\top \Gamma \delta U_{11}^\top x
\\
&+
2\bigl(\delta v_{21}^\top \Delta m+2\delta v_{22}\bigr)
\sigma'(z_\theta(x))
\bigl(2\delta u_{12}+\delta U_{11}\Delta m\bigr)^\top x
\\
&+
\bigl(v_{21}^\top \Delta m+2v_{22}\bigr)
\sigma''(z_\theta(x))
\Bigl[
\bigl(2\delta u_{12}+\delta U_{11}\Delta m\bigr)^\top x
\Bigr]^2.
\end{aligned}
\]
\end{lemma}

\begin{proof}
Direct computations
\end{proof}

\begin{lemma}\label{lem:gradients_tildeR}
For every \((w,b,r,s)\in \mathbb R^d\times \mathbb R\times \mathbb R\times \mathbb R^d\), the function \(\widetilde{\mathcal R}\) is differentiable and its gradients are given by
\[
\left\{
\begin{array}{l}
\nabla_w \widetilde{\mathcal R}(w,b,r,s)
=
\mathbb E\Big[
\big(
\sigma(g_{w,b,r,s}(X))-\eta(X)
\big)X
\Big], \\
\partial_b \widetilde{\mathcal R}(w,b,r,s)
=
\mathbb E\Big[
\sigma(g_{w,b,r,s}(X))-\eta(X)
\Big], \\
\partial_r \widetilde{\mathcal R}(w,b,r,s)
=
\mathbb E\Big[
\big(
\sigma(g_{w,b,r,s}(X))-\eta(X)
\big)\,
\sigma(s_0+s^\top X)
\Big], \\
\nabla_s \widetilde{\mathcal R}(w,b,r,s)
=
r\,\mathbb E\Big[
\big(
\sigma(g_{w,b,r,s}(X))-\eta(X)
\big)\,
\sigma(s_0+s^\top X)\bigl(1-\sigma(s_0+s^\top X)\bigr)\,X
\Big].
\end{array}
\right.
\]

This can also be written as
\[
\nabla_\zeta \widetilde{\mathcal R}(\zeta)
=
\mathbb E\Big[
\big(
\sigma(g_\zeta(X))-\eta(X)
\big)\,\nabla g_\zeta(X)
\Big],
\qquad
\zeta=(w,b,r,s).
\]
\end{lemma}

\begin{proof}
Let $g(X)=g_{w,b,r,s}(X)=w^\top x+b+r\,\sigma\left(s_0+s^\top x\right)$. One has,
\[
\nabla \widetilde{\mathcal R}(w,b,r,s)
=
\mathbb E\Big[
\ell'(Y,g(X))\,\nabla g(X)
\Big],
\]
Using Lemma~\ref{lem:identity_log_loss_derivative},
\[
\nabla \widetilde{\mathcal R}(w,b,r,s)
=
\mathbb E\Big[
\bigl(\sigma(g(X))-\eta(X)\bigr)\,\nabla g(X)
\Big].
\]
The formulas for the partial gradients then follow by differentiating
\(g_{w,b,r,s}\) with respect to \((w,b,r,s)\) (where we use the fact that $\sigma'=\sigma(1-\sigma)$ for the gradient with respect to $s$).
\end{proof}

\begin{lemma}[Explicit gradients of $\mathcal R$]
The gradients of $\mathcal R$ are given by
\label{lem:grad_of_R}
\[
\left\{
\begin{array}{l}
\nabla_{U_{11}}\mathcal R(\theta)
=
\mathbb E\left[
\bigl( \sigma(f_\theta(X))-\eta(X) \bigr)
\left(
X(\Gamma^\top v_{21})^\top
+
\bigl(v_{21}^\top\Delta m+2v_{22}\bigr)
\sigma'\bigl(s_0+(2u_{12}+U_{11}\Delta m)^\top X\bigr)
X\Delta m^\top
\right)
\right],\\
\nabla_{u_{12}}\mathcal R(\theta)
=
2\bigl(v_{21}^\top\Delta m+2v_{22}\bigr)
\mathbb E\left[
\bigl( \sigma(f_\theta(X))-\eta(X) \bigr)
\sigma'\bigl(s_0+(2u_{12}+U_{11}\Delta m)^\top X\bigr)
X
\right], \\
\nabla_{v_{21}}\mathcal R(\theta)
=
\mathbb E\left[
\bigl( \sigma(f_\theta(X))-\eta(X) \bigr)
\left(
\Gamma U_{11}^\top X
+
m_{-1}
+
\Delta m\,
\sigma\bigl(s_0+(2u_{12}+U_{11}\Delta m)^\top X\bigr)
\right)
\right], \\
\partial_{v_{22}}\mathcal R(\theta)
=
\mathbb E\left[
\bigl( \sigma(f_\theta(X))-\eta(X) \bigr)
\left(
-1
+
2\sigma\bigl(s_0+(2u_{12}+U_{11}\Delta m)^\top X\bigr)
\right)
\right].
\end{array}
\right.
\]

This can be written as 
\[
\left\{
\begin{array}{l}
\nabla_{U_{11}}\mathcal R(\theta)
=
\nabla_w\widetilde{R}(\zeta)(\Gamma^\top v_{21})^\top
+
\nabla_s\widetilde{R}(\zeta)\Delta m^\top,
\\
\nabla_{u_{12}}\mathcal R(\theta)
=
2\nabla_s\widetilde{R}(\zeta),
\\
\nabla_{v_{21}}\mathcal R(\theta)
=
\Gamma U_{11}^\top \nabla_w\widetilde{R}(\zeta)
+
\nabla_b\widetilde{R}(\zeta)m_{-1}
+
\nabla_r\widetilde{R}(\zeta)\Delta m,
\\
\partial_{v_{22}}\mathcal R(\theta)
=
-\nabla_b\widetilde{R}(\zeta)+2\nabla_r\widetilde{R}(\zeta).
\end{array}
\right.
\]
\end{lemma}

\begin{proof}
For the logistic loss, Lemma~\ref{lem:identity_log_loss_derivative} gives
\[
\mathbb E\left[
\ell'\bigl(Y,f_\theta(X)\bigr)\mid X
\right]
=
\sigma(f_\theta(X))-\eta(X)
\]
Hence, for any parameter component $\alpha$,
\[
\nabla_\alpha \mathcal R(\theta)
=
\mathbb E\left[
\bigl( \sigma(f_\theta(X))-\eta(X) \bigr)\nabla_\alpha f_\theta(X)
\right].
\]
The derivatives of $f_\theta$ are
\[
\nabla_{U_{11}}f_\theta(x)
=
x(\Gamma^\top v_{21})^\top
+
\bigl(v_{21}^\top\Delta m+2v_{22}\bigr)
\sigma'\bigl(s_0+(2u_{12}+U_{11}\Delta m)^\top x\bigr)
x\Delta m^\top,
\]
\[
\nabla_{u_{12}}f_\theta(x)
=
2\bigl(v_{21}^\top\Delta m+2v_{22}\bigr)
\sigma'\bigl(s_0+(2u_{12}+U_{11}\Delta m)^\top x\bigr)x,
\]
\[
\nabla_{v_{21}}f_\theta(x)
=
\Gamma U_{11}^\top x
+
m_{-1}
+
\Delta m\,
\sigma\bigl(s_0+(2u_{12}+U_{11}\Delta m)^\top x\bigr),
\]
and
\[
\partial_{v_{22}}f_\theta(x)
=
-1
+
2\sigma\bigl(s_0+(2u_{12}+U_{11}\Delta m)^\top x\bigr).
\]
Plugging these identities into the previous formula gives the result.
\end{proof}

\begin{lemma}[Second-order expansion of the risk]
\label{lem:second_order_risk}
Let
\[
\mathcal{R}(\theta)
=
\mathbb{E}\bigl[\ell(Y,f_\theta(X))\bigr],
\]
where \(\ell\) is the logistic loss, and recall that $\eta(x)=\mathbb{P}(Y=1\mid X=x)$. Then, for any
direction \(h\), one has
\[
D^2\mathcal{R}(\theta)[h,h]
=
\mathbb{E}\Bigl[
\sigma(f_\theta(X))(1-\sigma(f_\theta(X)))
\bigl(Df_\theta(X)[h]\bigr)^2
+
\bigl(\sigma(f_\theta(X))-\eta(X)\bigr)
D^2f_\theta(X)[h,h]
\Bigr].
\]
\end{lemma}

\begin{proof}
The first and second variations of \(\mathcal R\) are given by the chain rule:
\[
D\mathcal{R}(\theta)[h]
=
\mathbb{E}\Bigl[
\ell'\bigl(Y,f_\theta(X)\bigr)
Df_\theta(X)[h]
\Bigr],
\]
and
\[
D^2\mathcal{R}(\theta)[h,h]
=
\mathbb{E}\Bigl[
\ell''\bigl(Y,f_\theta(X)\bigr)
\bigl(Df_\theta(X)[h]\bigr)^2
+
\ell'\bigl(Y,f_\theta(X)\bigr)
D^2f_\theta(X)[h,h]
\Bigr].
\]
Lemma~\ref{lem:identity_log_loss_derivative} gives
\[
\mathbb{E}\bigl[\ell'(Y,f_\theta(X))\mid X\bigr]
=
\sigma(f_\theta(X))-\eta(X),
\]
so
\[
\mathbb{E}\bigl[\ell''(Y,f_\theta(X))\mid X\bigr]
=
\sigma(f_\theta(X))(1-\sigma(f_\theta(X))).
\]
Substituting these identities in the previous expression gives the result.
\end{proof}

\paragraph{Curvature at a Bayes realization.}
At an optimal predictor, the probability discrepancy vanishes. The Hessian therefore reduces to a weighted second moment of the first variation of the predictor. On $\mathcal E_1$, this reveals exactly which perturbations are invisible to second order and explains why the condition $s\neq0$ is needed for the local convergence argument.

\begin{proposition}\label{prop:hessian_tildeR_on_E1}
Let \(\mathcal E_1\) be as defined in Lemma~\ref{lem:argmin_Rtilde_E1_E2}, and let
\[
\zeta^\star=(w^\star,b^\star,0,s^\star)\in \mathcal E_1,
\qquad s^\star\in\mathbb R^d,
\]
where \(w^\star\) and \(b^\star\) are the coefficients of the Bayes score $f^\star(x)=\langle w^\star,x\rangle+b^\star$. Then the Hessian of \(\widetilde{\mathcal R}\) at \(\zeta^\star\) is given by
\[
\nabla^2 \widetilde{\mathcal R}(\zeta^\star)
=
\mathbb E\Big[
\eta(X)\bigl(1-\eta(X)\bigr)\,
\nabla g_{\zeta^\star}(X)\,\nabla g_{\zeta^\star}(X)^\top
\Big],
\]
where
\[
\nabla g_{\zeta^\star}(X)
=
\begin{pmatrix}
X\\[0.2em]
1\\[0.2em]
\sigma(s_0+s^{\star\top}X)\\[0.2em]
0
\end{pmatrix}.
\]

Equivalently, for every
\[
\delta \zeta=(\delta w,\delta b,\delta r,\delta s),
\qquad
\delta \zeta'=(\delta w',\delta b',\delta r',\delta s'),
\]
one has
\[
D^2\widetilde{\mathcal R}(\zeta^\star)(\delta \zeta,\delta \zeta')
=
\mathbb E\Big[
\eta(X)\bigl(1-\eta(X)\bigr)\,
\bigl(\delta w^\top X+\delta b+\delta r\,\sigma(s_0+s^{\star\top}X)\bigr)
\bigl(\delta w'^\top X+\delta b'+\delta r'\,\sigma(s_0+s^{\star\top}X)\bigr)
\Big].
\]
In particular, the Hessian does not depend on the components
\(\delta s\) and \(\delta s'\).
\end{proposition}

\begin{proof}
Let $\zeta^\star=(w^\star,b^\star,0,s^\star)\in \mathcal E_1$ and write $g_\zeta(X)=g_{w,b,r,s}(X)$. By differentiation under the expectation and the chain rule,
\[
\nabla^2 \widetilde{\mathcal R}(\zeta)
=
\mathbb E\Big[
\ell''(Y,g_\zeta(X))\,\nabla g_\zeta(X)\nabla g_\zeta(X)^\top
+
\ell'(Y,g_\zeta(X))\,\nabla^2 g_\zeta(X)
\Big].
\]

By Lemma~\ref{lem:identity_log_loss_derivative},
\[
\mathbb E\bigl[\ell'(Y,g_\zeta(X))\mid X\bigr]
=
\sigma(g_\zeta(X))-\eta(X).
\]
Since \(\zeta^\star\in\mathcal E_1\), we have $g_{\zeta^\star}(X)=f^\star(X)$ and hence $\sigma(g_{\zeta^\star}(X))=\sigma(f^\star(X))=\eta(X)$. Therefore, $\mathbb E\bigl[\ell'(Y,g_{\zeta^\star}(X))\mid X\bigr]=0,$ and the second term vanishes:
\[
\mathbb E\Big[
\ell'(Y,g_{\zeta^\star}(X))\,\nabla^2 g_{\zeta^\star}(X)
\Big]=0.
\]

Next, as in Lemma~\ref{lem:second_order_risk}, we get 
\[
\mathbb{E}\bigl[\ell''(Y,g_\zeta(X))\mid X\bigr]
=
\sigma(g_\zeta(X))(1-\sigma (g_\zeta(X))).
\]
Thus, since \(g_{\zeta^\star}(X)=f^\star(X)\) and \(\sigma(f^\star(X))=\eta(X)\), we get
\[
\ell''(Y,g_{\zeta^\star}(X))=\eta(X)\bigl(1-\eta(X)\bigr).
\]
Hence
\[
\nabla^2 \widetilde{\mathcal R}(\zeta^\star)
=
\mathbb E\Big[
\eta(X)\bigl(1-\eta(X)\bigr)\,
\nabla g_{\zeta^\star}(X)\nabla g_{\zeta^\star}(X)^\top
\Big].
\]

Finally, since \(\zeta^\star\in\mathcal E_1\), we have \(r^\star=0\), so
\[
\nabla_s g_{\zeta^\star}(X)
=
r^\star\,\sigma(s_0+s^{\star\top}X)\bigl(1-\sigma(s_0+s^{\star\top}X)\bigr)X
=
0.
\]
Therefore,
\[
\nabla g_{\zeta^\star}(X)
=
\begin{pmatrix}
X\\[0.2em]
1\\[0.2em]
\sigma(s_0+s^{\star\top}X)\\[0.2em]
0
\end{pmatrix},
\]
\end{proof}

\begin{lemma}\label{lem:ker_hessian_tildeR_on_E1}
Let
\[
\zeta^\star=(w^\star,b^\star,0,s^\star)\in \mathcal E_1.
\]
Then
\[
\ker \nabla^2\widetilde{\mathcal R}(\zeta^\star)
=
\begin{cases}
\{(0,0,0,\delta s):\ \delta s\in\mathbb R^d\},
& \text{if } s^\star\neq 0,\\[0.4em]
\{(0,-\delta r\,\sigma(s_0),\delta r,\delta s):\ (\delta r,\delta s)\in\mathbb R\times\mathbb R^d\},
& \text{if } s^\star=0.
\end{cases}
\]
\end{lemma}

\begin{proof}
Recall that for every $\delta \zeta=(\delta w,\delta b,\delta r,\delta s)$ and $\delta \zeta'=(\delta w',\delta b',\delta r',\delta s')$, we have
\[
D^2\widetilde{\mathcal R}(\zeta^\star)(\delta \zeta,\delta \zeta')
=
\mathbb E\Big[
\eta(X)\bigl(1-\eta(X)\bigr)\,
\bigl(\delta w^\top X+\delta b+\delta r\,\sigma(s_0+s^{\star\top}X)\bigr)
\bigl(\delta w'^\top X+\delta b'+\delta r'\,\sigma(s_0+s^{\star\top}X)\bigr)
\Big].
\]
Hence, if \(\delta \zeta\in \ker \nabla^2\widetilde{\mathcal R}(\zeta^\star)\), then
\[
D^2\widetilde{\mathcal R}(\zeta^\star)(\delta \zeta,\delta \zeta)=0,
\]
so that
\[
\mathbb E\Big[
\eta(X)\bigl(1-\eta(X)\bigr)\,
\bigl(\delta w^\top X+\delta b+\delta r\,\sigma(s_0+s^{\star\top}X)\bigr)^2
\Big]=0.
\]
Since \(\eta(X)(1-\eta(X))>0\) almost surely, this implies
\[
\delta w^\top X+\delta b+\delta r\,\sigma(s_0+s^{\star\top}X)=0
\qquad\text{a.s.}
\]

If \(s^\star\neq 0\), since \(X\) is a non-degenerate Gaussian mixture, and
\(x\mapsto \sigma(s_0+s^{\star\top}x)\) is not affine when \(s^\star\neq 0\),
this forces
\[
\delta w=0,\qquad \delta b=0,\qquad \delta r=0.
\]
Conversely, every vector of the form \((0,0,0,\delta s)\) clearly belongs to
the kernel, since the Hessian does not depend on \(\delta s\). Therefore
\[
\ker \nabla^2\widetilde{\mathcal R}(\zeta^\star)
=
\{(0,0,0,\delta s):\ \delta s\in\mathbb R^d\}.
\]

If \(s^\star=0\), then
\[
\sigma(s_0+s^{\star\top}x)=\sigma(s_0),
\]
and the condition becomes $\delta w^\top X+\delta b+\delta r\,\sigma(s_0)=0$ almost surely. Thus, we obtain
\[
\delta w=0,
\qquad
\delta b=-\delta r\,\sigma(s_0).
\]
Conversely, every vector of the form
\[
(0,-\delta r\,\sigma(s_0),\delta r,\delta s)
\]
belongs to the kernel. This proves the second case.
\end{proof}

\paragraph{Returning to attention parameters.}
The effective coordinates expose the predictor, but training evolves the original matrix and vector parameters. The differential of $\Phi$ connects these two descriptions. In particular, the chain rule below identifies the rank condition needed to transfer stationarity from the original risk to the effective risk.

\begin{lemma}\label{lem:differential_Phi}
Let
\[
\theta=(U_{11},u_{12},v_{21},v_{22})\in\Theta,
\qquad
\delta\theta=(\delta U_{11},\delta u_{12},\delta v_{21},\delta v_{22})\in\Theta.
\]
Recall that
\[
\Phi:\Theta \to \mathbb R^d\times \mathbb R\times \mathbb R\times \mathbb R^d,
\qquad
\Phi(\theta)=\bigl(w(\theta),b(\theta),r(\theta),s(\theta)\bigr),
\]
with
\[
w = U_{11}\Gamma v_{21},
\qquad
b = v_{21}^\top m_{-1} - v_{22},
\qquad
r = v_{21}^\top (m_1-m_{-1}) + 2v_{22},
\qquad
s = 2u_{12} + U_{11}(m_1-m_{-1}).
\]
Then
\[
D\Phi(\theta)[\delta\theta]
=
\bigl(\delta w,\delta b,\delta r,\delta s\bigr),
\]
where
\[
\delta w
=
\delta U_{11}\Gamma v_{21}
+
U_{11}\Gamma \delta v_{21},
\]
\[
\delta b
=
\delta v_{21}^\top m_{-1}
-
\delta v_{22},
\]
\[
\delta r
=
\delta v_{21}^\top (m_1-m_{-1})
+
2\,\delta v_{22},
\]
and
\[
\delta s
=
2\,\delta u_{12}
+
\delta U_{11}(m_1-m_{-1}).
\]
\end{lemma}

\begin{proof}
Direct computation.
\end{proof}

\begin{lemma}
\label{lem:gradR_from_gradR~}
    \[
\nabla \mathcal R(\theta)
=
D\Phi(\theta)^*
\bigl(
\nabla \widetilde{\mathcal R}(\Phi(\theta))
\bigr).
\]
\end{lemma}

\paragraph{The singular set.}
When $v_{21}=0$, the differential of $\Phi$ loses rank, so stationarity of the original risk need not imply stationarity of the effective risk. The next lemma identifies an effective descent direction that survives at these exceptional critical points. It is the key input to the strict-saddle construction in Appendix~\ref{app:lda-landscape-proof}.

\begin{lemma}[Non-vanishing of the reduced $w$-gradient]
\label{lem:non_vanishing_omega_gradient}
Let $\theta_{\mathrm{sp}}\in \operatorname{Crit}(\mathcal R) \cap\{v_{21}=0\}$. Assume that $m_1+m_{-1}\neq 0$. Then
\[
\nabla_w \widetilde{\mathcal R}
\bigl(\Phi(\theta_{\mathrm{sp}})\bigr)
\neq 0.
\]
\end{lemma}

\begin{proof}
We argue by contradiction, assume that
$\nabla_w \widetilde{\mathcal R}
\bigl(\Phi(\theta_{\mathrm{sp}})\bigr)
=0$.

Since \(\theta_{\mathrm{sp}}\) is a critical point of \(\mathcal R\), using the expression of $\nabla \mathcal{R}$ in terms of $\nabla\widetilde{R}$ in lemma~\ref{lem:grad_of_R} the
criticality identities give
\[
2\partial_r\widetilde{\mathcal R}
\bigl(\Phi(\theta_{\mathrm{sp}})\bigr)
-
\partial_b\widetilde{\mathcal R}
\bigl(\Phi(\theta_{\mathrm{sp}})\bigr)
=0,
\]
and
\[
\Gamma U_{11}^\top
\nabla_w \widetilde{\mathcal R}
\bigl(\Phi(\theta_{\mathrm{sp}})\bigr)
+
m_{-1}
\partial_b\widetilde{\mathcal R}
\bigl(\Phi(\theta_{\mathrm{sp}})\bigr)
+
\Delta m\,
\partial_r\widetilde{\mathcal R}
\bigl(\Phi(\theta_{\mathrm{sp}})\bigr)
=0.
\]
Using $\nabla_w \widetilde{\mathcal R}
\bigl(\Phi(\theta_{\mathrm{sp}})\bigr)
=0$, we obtain
\[
2m_{-1}
\partial_r\widetilde{\mathcal R}
\bigl(\Phi(\theta_{\mathrm{sp}})\bigr)
+
\Delta m\,
\partial_r\widetilde{\mathcal R}
\bigl(\Phi(\theta_{\mathrm{sp}})\bigr)
=0.
\]
This becomes
\[
(m_1+m_{-1})
\partial_r\widetilde{\mathcal R}
\bigl(\Phi(\theta_{\mathrm{sp}})\bigr)
=0.
\]
By the assumption \(m_1+m_{-1}\neq 0\), we get $\partial_r\widetilde{\mathcal R}
\bigl(\Phi(\theta_{\mathrm{sp}})\bigr)
=0$. Therefore,
\[
\partial_b\widetilde{\mathcal R}
\bigl(\Phi(\theta_{\mathrm{sp}})\bigr)
=0.
\]

Now, since \(v_{21}=0\) at \(\theta_{\mathrm{sp}}\), we have
\[
f_{\theta_{\mathrm{sp}}}(x)
=
-v_{22}
+
2v_{22}\sigma(s_0+s(\theta_{\mathrm{sp}})^\top x).
\]
Hence
\[
\begin{aligned}
\mathbb{E}\Bigl[
\bigl(\sigma(f_{\theta_{\mathrm{sp}}}(X))-\eta(X)\bigr)
f_{\theta_{\mathrm{sp}}}(X)
\Bigr]
&=-v_{22} \mathbb{E}\Bigl[
\sigma(f_{\theta_{\mathrm{sp}}}(X))-\eta(X)\Bigr]\\
&\quad+2v_{22}\mathbb{E}\Bigl[
\bigl(\sigma(f_{\theta_{\mathrm{sp}}}(X))-\eta(X)\bigr)\\
&\hspace{42mm}\cdot\sigma(s_0+s(\theta_{\mathrm{sp}})^\top X)\Bigr]
\\
&=
v_{22}\bigl(-\partial_b\widetilde{\mathcal R}
\bigl(\Phi(\theta_{\mathrm{sp}})\bigr)
+
2\partial_r\widetilde{\mathcal R}
\bigl(\Phi(\theta_{\mathrm{sp}})\bigr) \bigr)
\\
&=
0.
\end{aligned}
\]

Moreover, since
\[
f^\star(x)=w^{\star\top}x+b^\star,
\]
we also have
\[
\begin{aligned}
\mathbb{E}\Bigl[
\bigl(\sigma(f_{\theta_{\mathrm{sp}}}(X))-\eta(X)\bigr)
f^\star(X)
\Bigr]
=
w^{\star\top}
\nabla_w \widetilde{\mathcal R}
\bigl(\Phi(\theta_{\mathrm{sp}})\bigr)
+
b^\star
\partial_b\widetilde{\mathcal R}
\bigl(\Phi(\theta_{\mathrm{sp}})\bigr)
=
0.
\end{aligned}
\]

Using \(\eta(x)=\sigma(f^\star(x))\), we obtain
\[
\mathbb{E}\Bigl[
\bigl(\sigma(f_{\theta_{\mathrm{sp}}}(X))-\sigma(f^\star(X))\bigr)
\bigl(f_{\theta_{\mathrm{sp}}}(X)-f^\star(X)\bigr)
\Bigr]
=0.
\]
Since \(\sigma\) is strictly increasing, the integrand is nonnegative and
vanishes if and only if
\[
f_{\theta_{\mathrm{sp}}}(X)=f^\star(X)
\quad\text{a.s.}
\]
Moreover no function of the form
$
x\longmapsto c_1+c_2\sigma(s_0+s^\top x)
$
coincides almost surely with \(f^\star\) so that last equality is impossible. Thus, our assumption must be false and
\[
\nabla_w \widetilde{\mathcal R}
\bigl(\Phi(\theta_{\mathrm{sp}})\bigr)
\neq 0.
\]
\end{proof}

\subsection{Optimization landscape (proofs of Theorems~\ref{thm:lda-landscape} and~\ref{thm:lda-landscape-initial}, Lemma~\ref{lem:lda-regularity-main}, and Proposition~\ref{prop:lda-escape-main})}
\label{appendix:proofs_lemmas_GF}
\label{app:lda-landscape-proof}

We now prove Theorems~\ref{thm:lda-landscape} and~\ref{thm:lda-landscape-initial}, Lemma~\ref{lem:lda-regularity-main}, and Proposition~\ref{prop:lda-escape-main}. The argument proceeds from the predictor to its parameterization: we identify all effective Bayes realizations, show that every effective stationary point is optimal, and determine what changes under the attention factorization. We then address a separate issue---loss of compactness---which explains why these landscape results alone do not establish global convergence.

Each main-text result has an explicitly identified proof below. The proof
of Theorem~\ref{thm:lda-landscape} combines the characterization of the
effective minimizers with effective stationarity. The proof of
Theorem~\ref{thm:lda-landscape-initial} combines the preimage description
with the regular-point and strict-saddle lemmas. The proofs of
Lemma~\ref{lem:lda-regularity-main} and
Proposition~\ref{prop:lda-escape-main} are given directly, after restating
their claims for use elsewhere in the appendix.

Recall that $g_{w,b,r,s}(x)=w^\top x+b+r\sigma(s_0+s^\top x)$ and $\mathcal R=\widetilde{\mathcal R}\circ\Phi$. The coordinates $(w,b)$ describe the affine part, while $r$ and $s$ control the amplitude and direction of the sigmoid. This decomposition makes the two ways of representing an affine Bayes score transparent.

\subsubsection{Bayes minimizers (Theorems~\ref{thm:lda-landscape} and~\ref{thm:lda-landscape-initial}, minimizer claims)}
\label{app:lda-minimizers-block}

\paragraph{Two mechanisms for attaining the Bayes risk.}
By Lemma~\ref{lem:lda_logistic_risk}, a parameter is optimal exactly when its predictor equals $f^\star$. This can happen by setting $r=0$, which removes the sigmoid, or by setting $s=0$, which makes it constant and allows its contribution to be absorbed into the bias. The following lemma shows that these mechanisms exhaust the effective minimizers.

\begin{lemma}[Effective minimizers]
\label{lem:argmin_Rtilde_E1_E2}
The minimizer set of $\widetilde{\mathcal R}$ is
\[
\mathcal E^\star=\mathcal E_1\cup\mathcal E_2,
\quad
\mathcal E_1=\{(w^\star,b^\star,0,s):s\in\mathbb R^d\},
\]
\[
\mathcal E_2=\{(w^\star,b^\star-r\sigma(s_0),r,0):r\in\mathbb R\}.
\]
\end{lemma}

\begin{proof}[Proof of Lemma~\ref{lem:argmin_Rtilde_E1_E2}]

By Lemma~\ref{lem:lda_logistic_risk}, the set of minimizers of
\(\widetilde{\mathcal R}\) is given by
\[
\mathcal E^\star
=
\Bigl\{
(w,b,r,s)\in \mathbb R^d\times \mathbb R\times \mathbb R\times \mathbb R^d
\;:\;
w^\top x+b+r\,\sigma(s_0+s^\top x)=f^\star(x)
\quad \text{a.s.}
\Bigr\},
\]
where the Bayes score is
$
f^\star(x)=w^{\star\top}x+b^\star,
\qquad
w^\star=\Gamma^{-1}(m_1-m_{-1})
$. 
Therefore, if \((w,b,r,s)\in\mathcal E^\star\), then
\[
w^\top x+b+r\,\sigma(s_0+s^\top x)
=
w^{*\top}x+b^\star
\qquad \text{a.s.}
\]
This is equivalent to saying that one of the following two situations occurs:

\begin{itemize}
    \item either
    \[
    w=w^\star,
    \qquad
    s=0,
    \qquad
    b+r\,\sigma(s_0)=b^\star,
    \]
    that is,
    \[
    w=w^\star,
    \qquad
    s=0,
    \qquad
    b=b^\star-r\,\sigma(s_0).
    \]
    We denote this set by
    \[
    \mathcal E_2
    =
    \Bigl\{
    (w^\star,b^\star-r\,\sigma(s_0),r,0)
    \;:\;
    r\in\mathbb R
    \Bigr\}.
    \]

    \item or
    \[
    w=w^\star,
    \qquad
    b=b^\star,
    \qquad
    r=0,
    \qquad
    s\in\mathbb R^d \text{ arbitrary}.
    \]
    We denote this set by
    \[
    \mathcal E_1
    =
    \Bigl\{
    (w^\star,b^\star,0,s)
    \;:\;
    s\in\mathbb R^d
    \Bigr\}.
    \]
\end{itemize}

Hence,
\[
\mathcal E^\star
=
\mathcal E_1\cup \mathcal E_2
=
\Bigl\{
(w^\star,b^\star,0,s)
\;:\;
s\in\mathbb R^d
\Bigr\}
\cup
\Bigl\{
(w^\star,b^\star-r\,\sigma(s_0),r,0)
\;:\;
r\in\mathbb R
\Bigr\}.
\]
\end{proof}

The corresponding attention minimizers are obtained by taking preimages under $\Phi$. This step describes the optimal set without assuming that every effective parameter is attainable. In particular, the existence of a branch in effective coordinates does not by itself guarantee that its preimage is nonempty; the realizability discussion in Appendix~\ref{app:lda-realizability} addresses this distinction.

\begin{corollary}[Minimizers in attention coordinates]
\label{cor:argmin_R_preimage_Estar}
The minimizers of $\mathcal R$ satisfy
$\Theta^\star=\Theta_1\cup\Theta_2$, where
$\Theta_i=\Phi^{-1}(\mathcal E_i)$.
\end{corollary}

\begin{proof}[Proof of Corollary~\ref{cor:argmin_R_preimage_Estar}]
Since
$
\mathcal R=\widetilde{\mathcal R}\circ \Phi,
$
we have
\[
\theta\in \operatorname*{argmin}\mathcal R
\quad\Longleftrightarrow\quad
\widetilde{\mathcal R}(\Phi(\theta))=\min \mathcal R.
\]
Using
$
\min \mathcal R=\min \widetilde{\mathcal R}
$,
this is equivalent to
\[
\Phi(\theta)\in \operatorname*{argmin}\widetilde{\mathcal R}
=
\mathcal E^\star.
\]
Hence
\[
\operatorname*{argmin}\mathcal R
=
\Phi^{-1}(\mathcal E^\star)=\underbrace{\Phi^{-1}(\mathcal E_1)}_{\Theta_1}\cup \underbrace{\Phi^{-1}(\mathcal E_2)}_{\Theta_2}.
\]

\end{proof}

\paragraph{Why effective stationarity implies optimality.}
The reduced risk is not jointly convex in $(w,b,r,s)$. Nevertheless, fixing $s$ leaves a convex problem in $(w,b,r)$, and this restricted problem still contains the Bayes predictor through the choice $r=0$. Thus stationarity in the affine coefficients and sigmoid amplitude already forces the risk to attain its global minimum.

\subsubsection{Effective stationary points (completion of Theorem~\ref{thm:lda-landscape})}
\label{app:lda-effective-proof-block}

\begin{lemma}[No spurious effective critical points]
\label{lem:crit_Rtilde_equals_argmin}
$\operatorname{Crit}(\widetilde{\mathcal R})
=\operatorname*{argmin}\widetilde{\mathcal R}$.
\end{lemma}

\begin{proof}[Proof of Lemma \ref{lem:crit_Rtilde_equals_argmin}]
For each fixed \(s\in\mathbb R^d\), we define the partial risk
\[
\widetilde{\mathcal R}_s(w,b,r)
\coloneqq 
\widetilde{\mathcal R}(w,b,r,s)
=
\mathbb E\Bigl[
\ell_{\log}\bigl(Y,w^\top X+b+r\,\sigma(s_0+s^\top X)\bigr)
\Bigr].
\]

Then, \(\widetilde{\mathcal R}_s\) is convex as a function of
\((w,b,r)\in \mathbb R^d\times\mathbb R\times\mathbb R\) (as the expectation of a convex function).

So its critical points coincide with its global minimizers. Moreover, $\min_{(w',b',r')}\widetilde{\mathcal R}_s(w',b',r')\geq\mathcal R^\star$ and the choice \((w,b,r)=(w^\star,b^\star,0)\) yields ${\mathcal R}_s(w,b,r)=\mathcal R^\star$. This leads to
\[
\min_{(w',b',r')}\widetilde{\mathcal R}_s(w',b',r')=\mathcal R^\star.
\]

Let now \((w,b,r,s)\in \operatorname{Crit}(\widetilde{\mathcal R})\).
Then, in particular,
\[
\nabla_w\widetilde{\mathcal R}(w,b,r,s)=0,
\qquad
\partial_b\widetilde{\mathcal R}(w,b,r,s)=0,
\qquad
\partial_r\widetilde{\mathcal R}(w,b,r,s)=0.
\]
Equivalently,
\[
\nabla_{(w,b,r)}\widetilde{\mathcal R}_s(w,b,r)=0.
\]
Hence \((w,b,r)\) is a critical point of the convex function
\(\widetilde{\mathcal R}_s\), and therefore a global minimizer of
\(\widetilde{\mathcal R}_s\). It follows that
\[
\widetilde{\mathcal R}(w,b,r,s)=\widetilde{\mathcal R}_s(w,b,r)
=
\min_{(w',b',r')}\widetilde{\mathcal R}_s(w',b',r')=\mathcal R^\star.
\]

which shows that
\[
\operatorname{Crit}(\widetilde{\mathcal R})
\subset
\operatorname*{argmin}\widetilde{\mathcal R}.
\]

The reverse inclusion is immediate: since \(\widetilde{\mathcal R}\) is
differentiable on an open set, every global minimizer is a critical point. We conclude that
\[
\operatorname{Crit}(\widetilde{\mathcal R})
=
\operatorname*{argmin}\widetilde{\mathcal R}.
\]
\end{proof}

\begin{proof}[Proof of Theorem~\ref{thm:lda-landscape}]
Lemma~\ref{lem:argmin_Rtilde_E1_E2} identifies the minimizer set as
$\mathcal E_1\cup\mathcal E_2$, and
Lemma~\ref{lem:crit_Rtilde_equals_argmin} shows that it coincides with
$\operatorname{Crit}(\widetilde{\mathcal R})$. Together, these give both
claims of the theorem.
\end{proof}

\paragraph{From effective to attention parameters.}
The favorable effective landscape does not automatically transfer to the attention parameters. The chain rule gives $\nabla\mathcal R(\theta)=D\Phi(\theta)^*\nabla\widetilde{\mathcal R}(\Phi(\theta))$: a singular differential may annihilate a nonzero effective gradient. The next result identifies precisely where this obstruction is absent. Here, asymmetry means $m_1\neq-m_{-1}$, rather than merely distinct class means.

\subsubsection{Regularity of the parameter map (proof of Lemma~\ref{lem:lda-regularity-main})}
\label{app:lda-regularity-proof-block}

\begin{lemma}[Regularity of the reparametrization]
\label{lem:Dphi_surjective}
The adjoint $D\Phi(\theta)^*$ is injective if and only if
$m_1\neq-m_{-1}$ and $v_{21}\neq0$.
\end{lemma}

\begin{proof}[Proof of Lemma~\ref{lem:lda-regularity-main}]
This is the statement restated in Lemma~\ref{lem:Dphi_surjective}.
Let
\[
\delta\theta=(\delta U_{11},\delta u_{12},\delta v_{21},\delta v_{22})\in\Theta,
\]
and let
\[
(\delta w,\delta b,\delta r,\delta s)
=
D\Phi(\theta)[\delta\theta].
\]
By Lemma~\ref{lem:differential_Phi}, this means that
\begin{align}
\delta w &= \delta U_{11}\Gamma v_{21}+U_{11}\Gamma \delta v_{21}, \label{eq:Dphi_sys_1}\\
\delta b &= \delta v_{21}^\top m_{-1}-\delta v_{22}, \label{eq:Dphi_sys_2}\\
\delta r &= \delta v_{21}^\top (m_1-m_{-1})+2\delta v_{22}, \label{eq:Dphi_sys_3}\\
\delta s &= 2\delta u_{12}+\delta U_{11}(m_1-m_{-1}). \label{eq:Dphi_sys_4}
\end{align}

\textbf{Proof of the necessary conditions.} Assume first that \(m_1=-m_{-1}\). Then
$
m_1-m_{-1}=-2m_{-1},
$
and combining \eqref{eq:Dphi_sys_2} and \eqref{eq:Dphi_sys_3} yields
\[
\delta r
=
-2\,\delta v_{21}^\top m_{-1}+2\delta v_{22}
=
-2\delta b.
\]
Hence every vector in the image of \(D\Phi(\theta)\) satisfies
\[
2\delta b+\delta r=0,
\]
so \(D\Phi(\theta)\) cannot be surjective.

Assume now that \(v_{21}=0\). Then \eqref{eq:Dphi_sys_1} becomes
\[
\delta w=U_{11}\Gamma \delta v_{21}.
\]
Therefore \((\delta w,\delta b,\delta r)\) depend only on
\((\delta v_{21},\delta v_{22})\) through a linear map
\[
h:\mathbb R^{d+1}\to \mathbb R^{d+2}.
\]
Since \(\dim(\mathbb R^{d+1})<\dim(\mathbb R^{d+2})\), the map \(h\) cannot be
surjective, and neither can \(D\Phi(\theta)\).

This proves the necessity of the two conditions.

\textbf{Proof of the sufficient conditions.} Conversely, assume that
\[
m_1\neq -m_{-1}
\qquad\text{and}\qquad
v_{21}\neq 0.
\]
Let us fix
\[
(\delta w,\delta b,\delta r,\delta s)\in \mathbb R^d\times \mathbb R\times \mathbb R\times \mathbb R^d
\]
be arbitrary. We show that the system
\eqref{eq:Dphi_sys_1} - \eqref{eq:Dphi_sys_4} can be solved.

First, using \eqref{eq:Dphi_sys_2} and \eqref{eq:Dphi_sys_3}, we want to choose
\(\delta v_{21}\) and \(\delta v_{22}\) so that
\[
\delta b=\delta v_{21}^\top m_{-1}-\delta v_{22},
\qquad
\delta r=\delta v_{21}^\top (m_1-m_{-1})+2\delta v_{22}.
\]
Adding the second equation to twice the first one gives
\[
\delta v_{21}^\top (m_1+m_{-1})=\delta r+2\delta b.
\]
Since \(m_1+m_{-1}\neq 0\), one can choose \(\delta v_{21}\in\mathbb R^d\) such that
\[
\delta v_{21}^\top (m_1+m_{-1})=\delta r+2\delta b.
\]
Then define
\[
\delta v_{22}\coloneqq \delta v_{21}^\top m_{-1}-\delta b.
\]
With this choice, both \eqref{eq:Dphi_sys_2} and \eqref{eq:Dphi_sys_3} are satisfied.

Next, we solve \eqref{eq:Dphi_sys_1}. Since \(v_{21}\neq 0\) and \(\Gamma\) is
invertible, we have $\Gamma v_{21}\neq0$. So there exists a matrix \(\delta U_{11}\in\mathbb R^{d\times d}\) such that
\[
\delta U_{11}\Gamma v_{21}
=
\delta w-U_{11}\Gamma \delta v_{21}.
\]
With this choice, \eqref{eq:Dphi_sys_1} holds.

Finally, once \(\delta U_{11}\) is fixed, we define
\[
\delta u_{12}
=
\frac12\Bigl(\delta s-\delta U_{11}(m_1-m_{-1})\Bigr),
\]
which ensures \eqref{eq:Dphi_sys_4}.

This proves that every
$(\delta w,\delta b,\delta r,\delta s)$
belongs to the image of \(D\Phi(\theta)\), so \(D\Phi(\theta)\) is surjective.
\end{proof}

Away from the singular set $\{v_{21}=0\}$, injectivity of the adjoint differential transfers stationarity to the effective risk. The preceding characterization then rules out any nonoptimal critical point in this regular region.

\subsubsection{Critical points and strict saddles (completion of Theorem~\ref{thm:lda-landscape-initial})}
\label{app:lda-factorized-proof-block}

\begin{lemma}[Critical points away from the singular set]
\label{lem:crit_R_outside_v21_zero}
If $m_1\neq-m_{-1}$, then every critical point of $\mathcal R$ with
$v_{21}\neq0$ is globally minimizing.
\end{lemma}

\begin{proof}[Proof of Lemma \ref{lem:crit_R_outside_v21_zero}]
We assume that $m_1\neq-m_{-1}$. Let $\theta\in \operatorname{Crit}(\mathcal R)$ and assume that $v_{21}\neq 0$.
Since
\[
\nabla \mathcal R(\theta)
=
D\Phi(\theta)^\star\nabla \widetilde{\mathcal R}(\Phi(\theta)),
\]
we obtain
\[
D\Phi(\theta)^\star\nabla \widetilde{\mathcal R}(\Phi(\theta))=0.
\]
By Lemma~\ref{lem:Dphi_surjective}, $D\Phi(\theta)$ is surjective so $D\Phi(\theta)^\star$ is injective, and then
\[
\nabla \widetilde{\mathcal R}(\Phi(\theta))=0.
\]
By Lemma~\ref{lem:crit_Rtilde_equals_argmin}, this implies that
\[
\Phi(\theta)\in \operatorname*{argmin}\widetilde{\mathcal R},
\]
Since
$
\mathcal R(\theta)=\widetilde{\mathcal R}(\Phi(\theta))
$,
we get
\[
\mathcal R(\theta)=R^\star,
\]
so that $\theta\in \operatorname*{argmin}\mathcal R$.

This shows that every critical point $\theta$ with $v_{21}\neq 0$
is a global minimizer of $\mathcal R$. Equivalently,
\[
\operatorname{Crit}(\mathcal R)
\subset
\operatorname*{argmin}\mathcal R
\;\sqcup\;
\Bigl(
\operatorname{Crit}(\mathcal R)\cap\{\theta\in\Theta:\ v_{21}=0\}
\Bigr).
\]
The reverse inclusion is immediate, which concludes the proof.
\end{proof}

\paragraph{The exceptional critical points are strict saddles.}
It remains to examine $\{v_{21}=0\}$, where the transfer argument fails. Lemma~\ref{lem:non_vanishing_omega_gradient} supplies a nonzero effective $w$-gradient at these points. By coupling variations of the two factors while keeping the gate fixed to first order, the proof turns this gradient into directions of opposite curvature. The attention parameterization therefore introduces saddles, but no additional local minima.

\begin{lemma}[Strict-saddle property]
\label{lem:no_spurious_min_R}
Assume $m_1\neq-m_{-1}$. Every nonoptimal critical point of $\mathcal R$
admits both a positive- and a negative-curvature direction. In particular,
every critical point is either a global minimizer or a strict saddle.
\end{lemma}

\begin{proof}[Proof of Lemma~\ref{lem:no_spurious_min_R}]

Let $\theta_{\mathrm{sp}}\in \operatorname{Crit}\mathcal{R} \cap \{v_{21}=0\}$. Then
\[
f_{\theta_{\mathrm{sp}}}(x)
=
-v_{22}
+
2v_{22}\sigma\bigl(s_0+s({\theta_{\mathrm{sp}}})^\top x\bigr).
\]

We choose a direction $h$ such that
\[
\delta u_{12}
=
-\frac12 \delta U_{11}\Delta m,
\qquad
\delta v_{22}=0.
\]
So that, at $\theta_{\mathrm{sp}}$, we get (using Lemma~\ref{lem:first_second_variations_model})
\[
Df_{\theta_{\mathrm{sp}}}(x)[h]
=
\delta v_{21}^\top
\Bigl(
\Gamma U_{11}^\top x
+
m_{-1}
+
\Delta m\, \sigma(s_0+s(\theta_{\mathrm{sp}})^\top x)
\Bigr),
\]
and
\[
D^2f_{\theta_{\mathrm{sp}}}(x)[h,h]
=
2\delta v_{21}^\top \Gamma \delta U_{11}^\top x.
\]

By Lemma~\ref{lem:second_order_risk} (and Lemma~\ref{lem:gradients_tildeR}), we can deduce that 
\begin{align*} 
D^2\mathcal{R}(\theta_{\mathrm{sp}})[h,h]
&=
\mathbb{E}\Bigl[
\sigma(f_\theta(X))(1-\sigma(f_\theta(X)))
\bigl(Df_\theta(X)[h]\bigr)^2
+
\bigl(\sigma(f_\theta(X))-\eta(X)\bigr)
D^2f_\theta(X)[h,h]
\Bigr] 
\\
&=C
+
2\delta v_{21}^\top
\Gamma
\delta U_{11}^\top
\nabla_w \widetilde{\mathcal R}\bigl(\Phi(\theta_{\mathrm{sp}})\bigr)
\end{align*}
where $C$ does not depend on $\delta U_{11}$.

Using Lemma~\ref{lem:non_vanishing_omega_gradient}, we know that
$\nabla_w \widetilde{\mathcal R}\bigl(\Phi(\theta_{\mathrm{sp}})\bigr)
\neq 0.$
We can thus choose a fixed $\delta v_{21}$ such that $\Gamma\delta v_{21}\neq 0$ and set
\[
\delta U_{11}^{(t)}
=
-t\,
\nabla_w \widetilde{\mathcal R}\bigl(\Phi(\theta_{\mathrm{sp}})\bigr)
(\Gamma \delta v_{21})^\top,
\qquad t>0.
\]
Then
\[
\begin{aligned}
\delta v_{21}^\top
\Gamma
\bigl(\delta U_{11}^{(t)}\bigr)^\top
\nabla_w \widetilde{\mathcal R}\bigl(\Phi(\theta_{\mathrm{sp}})\bigr)
&=
-t
\left\|
\nabla_w \widetilde{\mathcal R}\bigl(\Phi(\theta_{\mathrm{sp}})\bigr)
\right\|^2
\left\|
\Gamma\delta v_{21}
\right\|^2.
\end{aligned}
\]
Thus
\[
D^2\mathcal{R}(\theta_{\mathrm{sp}})[h_t,h_t]
=
C
-
2t
\left\|
\nabla_w \widetilde{\mathcal R}\bigl(\Phi(\theta_{\mathrm{sp}})\bigr)
\right\|^2
\left\|
\Gamma \delta v_{21}
\right\|^2.
\]
For $t$ large enough,
\[
D^2\mathcal{R}(\theta_{\mathrm{sp}})[h_t,h_t]<0.
\]
Taking 
\[
\delta U_{11}^{(t)}
=
+t\,
\nabla_w \widetilde{\mathcal R}\bigl(\Phi(\theta_{\mathrm{sp}})\bigr)
(\Gamma\delta v_{21})^\top,
\qquad t>0.
\]
leads to 
\[
D^2\mathcal{R}(\theta_{\mathrm{sp}})[h_t,h_t]>0.
\]
for t large enough.
Hence $\theta_{\mathrm{sp}}$ is a strict saddle point.

\end{proof}

\begin{proof}[Proof of Theorem~\ref{thm:lda-landscape-initial}]
Corollary~\ref{cor:argmin_R_preimage_Estar} gives
$\Theta^\star=\Theta_1\cup\Theta_2$. Under $m_1\neq-m_{-1}$,
Lemma~\ref{lem:crit_R_outside_v21_zero} shows that every critical point
outside $\{v_{21}=0\}$ is globally optimal, while
Lemma~\ref{lem:no_spurious_min_R} supplies a negative-curvature direction
at every nonoptimal critical point. Such a point cannot be a local
minimum. Hence every critical point is either globally minimizing or a
strict saddle, and all local minima are global.
\end{proof}

\paragraph{Why a benign landscape is not a global convergence theorem.}
We have classified finite critical points, but this does not bound a gradient-flow trajectory. The next construction exhibits noncompact near-optimal directions: as the sigmoid amplitude grows and its direction shrinks, its constant and linear contributions are canceled by the affine part. The predictor approaches $f^\star$ even though the effective parameters diverge.

\subsubsection{Minimizing sequences at infinity (proof of Proposition~\ref{prop:lda-escape-main})}
\label{app:lda-escape-proof-block}

\begin{proposition}[Minimizers at infinity]
\label{prop:infinite_mininum}
For $c\in\mathbb R^d$ and $R>0$, let
\[
\zeta_R=(w^\star-\sigma'(s_0)c,
b^\star-\sigma(s_0)R,R,c/R).
\]
Then $g_{\zeta_R}\to f^\star$ in $L^1(\mu_X)$ and
$\widetilde{\mathcal R}(\zeta_R)\to\mathcal R^\star$ as $R\to\infty$.
\end{proposition}

\begin{proof}[Proof of Proposition~\ref{prop:lda-escape-main}]
This is the statement restated in Proposition~\ref{prop:infinite_mininum}.
Let \(c\in\mathbb R^d\), and for \(r>0\) define
\[
\zeta_R
=
\left(
w^\star-\sigma'(s_0)c,\,
b^\star-\sigma(s_0)R,\,
R,\,
\frac cR
\right).
\]
For \(x\in\mathbb R^d\), we have
\[
g_{\zeta_R}(x)
=
(w^\star-\sigma'(s_0)c)^\top x
+
b^\star-\sigma(s_0)R
+
R\,\sigma\left(s_0+\frac{c^\top x}{R}\right).
\]
Therefore
\[
g_{\zeta_r}(x)-f^\star(x)
=
R\left[
\sigma\left(s_0+\frac{c^\top x}{R}\right)
-\sigma(s_0)
-\sigma'(s_0)\frac{c^\top x}{R}
\right].
\]

Using Taylor expansion at order \(2\) around \(s_0\) for fixed
\(x\), as \(R\to+\infty\), we get
\[
\sigma\left(s_0+\frac{c^\top x}{R}\right)
=
\sigma(s_0)
+
\sigma'(s_0)\frac{c^\top x}{R}
+
\frac12\sigma''(s_0)\frac{(c^\top x)^2}{R^2}
+
o\left(\frac1{R^2}\right).
\]
Hence
\[
g_{\zeta_R}(x)-f^\star(x)
=
\frac12\sigma''(s_0)\frac{(c^\top x)^2}{R}
+
o\left(\frac1R\right).
\]
In particular,
\[
g_{\zeta_R}(x)\longrightarrow f^\star(x)
\qquad
\text{for every }x\in\mathbb R^d.
\]

We now justify convergence in \(L^1(\mu_X)\). Since the sigmoid is smooth
and \(\sigma''\) is bounded, Taylor's formula with integral remainder gives
a constant \(C>0\) such that, for every \(R\geq 1\) and every \(x\),
\[
\left|
g_{\zeta_R}(x)-f^\star(x)
\right|
\leq
\frac{C}{R}(c^\top x)^2\leq C(c^\top x)^2
\]

Since \(X\) is a mixture of non-degenerate Gaussian distributions, it has
finite moments of all orders. In particular, $\mathbb E[(c^\top X)^2]<+\infty$. By dominated convergence, we obtain
\[
\|g_{\zeta_R}-f^\star\|_{L^1(\mu_X)}
\longrightarrow 0.
\]

It remains to pass to the risk. The
log-exp function is \(1\)-Lipschitz. Hence
\[
\left|
\widetilde{\mathcal R}(\zeta_R)
-
\mathcal R^\star
\right|
=
\left|
\mathbb E\left[
\log\bigl( 1+\exp(-Yg_{\zeta_R}(X))\bigr)-\log\bigl(1+\exp(-Yf^\star(X))\bigr)
\right]
\right|
\leq
\mathbb E\left[
|g_{\zeta_R}(X)-f^\star(X)|
\right].
\]
Since the right-hand side converges to \(0\), we obtain
\[
\widetilde{\mathcal R}(\zeta_R)
\longrightarrow
\mathcal R^\star .
\]
This proves the claim.
\end{proof}

The preceding results establish Theorems~\ref{thm:lda-landscape} and~\ref{thm:lda-landscape-initial}, Lemma~\ref{lem:lda-regularity-main}, and Proposition~\ref{prop:lda-escape-main}. Their dynamical implication is deliberately limited: the absence of nonoptimal local minima does not supply the boundedness needed for a global convergence argument. Conversely, the escaping sequence above is not itself a gradient-flow trajectory and does not prove that such trajectories diverge. We therefore turn to regimes where the dynamics can be controlled: a neighborhood of the regular component $\Theta_1$, and an invariant subspace in the balanced symmetric setting.

\subsection{Normal curvature (preparation for Theorem~\ref{thm:lda-dynamics})}
\label{appendix:morse-bott-theory}

Throughout this subsection, assume $m_1\neq-m_{-1}$. We now move from the classification of critical points to the geometry needed for local convergence. Since optimal parameters are not isolated, positive definiteness of the full Hessian is neither available nor desirable: directions tangent to a manifold of equivalent Bayes realizations must be flat. The relevant question is whether these are the only flat directions.

We apply the Morse--Bott lemma at a point
\[
\theta^\star \in \Theta_1^{\mathrm{reg}}
\coloneqq 
\bigl\{
\theta\in\Theta_1:\ s(\theta)\neq 0
\bigr\},
\]
in order to study the local convergence of the gradient flow near this critical manifold.

The key point is that \(\Theta_1^{\mathrm{reg}}\) is a Morse-Bott critical manifold for the risk \(\mathcal R\). In other words, at every point \(\theta^\star\in\Theta_1^{\mathrm{reg}}\), the tangent space to the manifold coincides with the kernel of the Hessian:
\[
T_{\theta^\star}\Theta_1^{\mathrm{reg}}
=
\ker \nabla^2 \mathcal R(\theta^\star).
\]
Equivalently, the Hessian is positive definite in the directions normal to the critical manifold. This means that, although the risk is flat along the manifold of minimizers, it has a quadratic behavior transversally to it.

The Morse--Bott lemma, recalled in Lemma~\ref{lem:morse_bott_theor_local},
then formalizes that, up to a local diffeomorphism, the risk \(\mathcal R\) is
locally equivalent to a quadratic function in the normal directions to
\(\Theta_1^{\mathrm{reg}}\). This normal form is the main ingredient in the
proof of local convergence of the gradient flow toward the manifold of global
minimizers (Theorem~\ref{thm:lda-dynamics}).

\medskip

We first state the geometric criterion and its local normal form, adapted from \cite{rebjock2024fastconvergencenonisolatedminima} and \cite{BANYAGA2004365} to our local setting.

\begin{definition}[Morse--Bott property at a point from \citealt{rebjock2024fastconvergencenonisolatedminima}]\label{def:morse_bott_point}
Let $f:M\to\mathbb R$ be a $\mathcal C^2$ function on a smooth manifold $M$, and let
$\bar x\in M$ be a local minimum of $f$. Let
\[
S\coloneqq \{x\in M:\ \nabla f(x)=0 \:,\: f(x)=f(\bar x)\}
\]
denote the associated critical set. We say that $f$ satisfies the
\emph{Morse--Bott property at $\bar x$} if, in a neighborhood of $\bar x$,
the set $S$ is a $\mathcal C^1$ submanifold and
\[
\ker \nabla^2 f(\bar x)=T_{\bar x}S.
\]
\end{definition}

In other words, the Hessian is non-degenerate in the normal directions to the local critical submanifold.

\begin{lemma}[Morse--Bott lemma, local version adapted from \citealt{BANYAGA2004365}]\label{lem:morse_bott_theor_local}
Let $M$ be a smooth manifold of dimension $m$, let $f:M\to\mathbb R$ be a smooth function,
and let $p\in M$ be a critical point of $f$.
Assume that $f$ satisfies the Morse--Bott property at $p$.
Let $C$ be the local critical submanifold through $p$, and let $n\coloneqq \dim C$. Then there exist an open neighborhood $U$ of $p$ in $M$ and a smooth chart
\[
\psi:U\to \mathbb R^n\times \mathbb R^{m-n}
\]
such that:

\begin{itemize}
    \item[(a)] $\psi(p)=0$;

    \item[(b)] $    \psi(U\cap C)
    =
    \{(x,y)\in \mathbb R^n\times \mathbb R^{m-n}:\ y=0\}$;

    \item[(c)] there exists an integer $k\le m-n$ such that
    \[
    (f\circ \psi^{-1})(x,y)
    =
    f(p)
    -
    y_1^2-\cdots-y_k^2
    +
    y_{k+1}^2+\cdots+y_{m-n}^2.
    \]
\end{itemize}

Here $k$ is the index of the Hessian of $f$ at $p$ restricted to the normal directions to $C$.
\end{lemma}

\begin{proof}
    The proof can be found in \cite{BANYAGA2004365}. The version we enounced can be deduced by considering the function $f$ restricted to the neighborhood where it satisfies the Morse-Bott property.
\end{proof}

\medskip

We recall the following definitions :
\[
\Theta_1\coloneqq \Phi^{-1}(\mathcal E_1),
\qquad
\Theta_2\coloneqq \Phi^{-1}(\mathcal E_2).
\]
\[
\Theta_1^{\mathrm{reg}}
\coloneqq 
\{\theta\in \Theta_1:\ s(\theta)\neq 0\}.
\]
Here, the condition $s(\theta)\neq 0$ will play an important role in the sequel,
as it ensures the non-degeneracy of the Hessian at the point under consideration.

\begin{lemma}\label{lem:R_morse_bott_on_Theta1_reg}
Let $\theta^\star\in \Theta_1^{\mathrm{reg}}$, and let $\mathcal S$ denote the connected component of $\Theta_1^{\mathrm{reg}}$ containing $\theta^\star$. Then $\mathcal R$ satisfies the Morse--Bott property
at $\theta^\star$. More precisely:

\begin{itemize}
    \item[(i)] $\mathcal S$ is a $\mathcal C^1$ submanifold in a neighborhood of $\theta^\star$;

    \item[(ii)]
    \[
    \ker \nabla^2 \mathcal R(\theta^\star)
    =
    T_{\theta^\star}\mathcal S.
    \]
\end{itemize}
\end{lemma}

The proof has two parts. The preimage theorem first identifies $\Theta_1^{\mathrm{reg}}$ as a smooth constraint manifold and computes its tangent space. The Hessian identities from Appendix~\ref{app:lda-differential-tools} then show that precisely these tangent directions have zero curvature. We recall the differential-topological statement needed for the first step.

\begin{definition}\label{def:regular_value}
Let $X$ and $Y$ be smooth manifolds, and let $f:X\to Y$ be a smooth map. We say that a point $y\in Y$ is a \emph{regular value} of $f$ if for every $x\in f^{-1}(y)$, the differential $df_x:T_xX\to T_yY$ is surjective.
\end{definition}

\begin{lemma}[Preimage theorem]\label{thm:preimage}
Let $X$ and $Y$ be smooth manifolds, let $f:X\to Y$ be a smooth map, and let $y\in Y$ be a regular value of $f$. Then the preimage $f^{-1}(y)$ is a smooth submanifold of $X$.
If moreover $y\in \operatorname{Im}(f)$, then the codimension of $f^{-1}(y)$
in $X$ is equal to $\dim Y$.
Furthermore, for every $x\in f^{-1}(y)$, the tangent space satisfies
\[
T_x\bigl(f^{-1}(y)\bigr)=\ker(df_x).
\]
\end{lemma}

\begin{proof}[Proof of Lemma~\ref{lem:R_morse_bott_on_Theta1_reg}]
We first show that the connected component \(\mathcal S\) of \(\Theta_1^{\mathrm{reg}}\)
containing \(\theta^\star\) is a smooth submanifold. Let
$
\mathcal O
\coloneqq 
\{\theta\in \Theta:\ s(\theta)\neq 0\}.
$
Since \(s:\Theta\to\mathbb R^d\) is continuous and $\mathbb R^d\setminus\{0\}$ is open, we have that $\mathcal O=s^{-1}(\mathbb R^d\setminus\{0\})$ is an open subset of \(\Theta\)  (and thus a smooth manifold).

Next, define
\[
C:\Theta\to \mathbb R^d\times \mathbb R\times \mathbb R,
\qquad
C(\theta)=\bigl(w(\theta)-w^\star,\ b(\theta)-b^\star,\ r(\theta)\bigr).
\]
By definition of \(\Theta_1\), we have
\[
\Theta_1
=
\{\theta\in\Theta:\ w(\theta)=w^\star,\ b(\theta)=b^\star,\ r(\theta)=0\}
=
C^{-1}(0).
\]
Therefore,
\[
\Theta_1^{\mathrm{reg}}
=
\Theta_1\cap \mathcal O
=
C^{-1}(0)\cap \mathcal O
=
\bigl(C_{|\mathcal O}\bigr)^{-1}(0),
\]
where $C_{|\mathcal{O}}$ denotes the restriction of \(C\) to the smooth open submanifold \(\mathcal O\).

Let now \(\theta\in \Theta_1^{\mathrm{reg}}= \bigl(C_{|\mathcal O}\bigr)^{-1}(0)\).
Since \(\mathcal O\) is open, we have
$
D(C_{|\mathcal O})(\theta)=DC(\theta).
$
Moreover, \(DC(\theta)\) is surjective. Indeed, \(C\) is obtained by keeping only the
\((w,b,r)\)-components of \(\Phi\), so \(DC(\theta)\) is the composition of
\(D\Phi(\theta)\) with the canonical projection onto $\mathbb R^d\times \mathbb R\times \mathbb R$.
By Lemma~\ref{lem:Dphi_surjective}, \(D\Phi(\theta)\) is surjective (because $m_1\neq -m_{-1}$ by assumption and $v_{21}\neq0$ since $U_{11}\Gamma v_{21}=w=w^\star\neq0$). Therefore \(DC(\theta)\) is surjective as well. Thus, \(0\) is a regular value of \(C_{|\mathcal O}\). 

By the preimage theorem~\ref{thm:preimage} we can then deduce that
$
\Theta_1^{\mathrm{reg}}
=
\bigl(C_{|\mathcal O}\bigr)^{-1}(0)
$
is a smooth submanifold of \(\mathcal O\), hence of \(\Theta\).
Its codimension is
$
\dim(\mathbb R^d\times \mathbb R\times \mathbb R)=d+2.
$
In particular, for every \(\theta\in \Theta_1^{\mathrm{reg}}\),
\[
T_\theta \Theta_1^{\mathrm{reg}}=\ker DC(\theta).
\]

Finally, since \(\mathcal S\) is a connected component of the smooth manifold
\(\Theta_1^{\mathrm{reg}}\), it is itself a smooth submanifold in a neighborhood of
\(\theta^\star\). In particular, \(\operatorname{codim}\mathcal{S}=d+2\) and for every \(\theta\in \mathcal{S}\),
\[
T_\theta \mathcal{S}=\ker DC(\theta).
\]

It remains to prove that
\[
\ker \nabla^2 \mathcal R(\theta^\star)=T_{\theta^\star}\mathcal S.
\]
We already know that $T_{\theta^\star}\mathcal S=\ker DC(\theta^\star)$, thus it is enough to show that
\[
\ker \nabla^2 \mathcal R(\theta^\star)=\ker DC(\theta^\star).
\]

Set $\zeta^\star\coloneqq \Phi(\theta^\star)\in \mathcal E_1.$ Since $\mathcal R=\widetilde{\mathcal R}\circ \Phi,$ the chain rule yields that for every \(h,h'\in \Theta\),
\[
D^2\mathcal R(\theta^\star)(h,h')
=
D^2\widetilde{\mathcal R}(\zeta^\star)\bigl(D\Phi(\theta^\star)h,D\Phi(\theta^\star)h'\bigr)
+
D\widetilde{\mathcal R}(\zeta^\star)\bigl(D^2\Phi(\theta^\star)(h,h')\bigr).
\]
Now \(\zeta^\star=\Phi(\theta^\star)\in \mathcal E^\star=\operatorname*{argmin}\widetilde{\mathcal R}\), we have $D\widetilde{\mathcal R}(\zeta^\star)=0$ and therefore,
\[
D^2\mathcal R(\theta^\star)(h,h')
=
D^2\widetilde{\mathcal R}(\zeta^\star)\bigl(D\Phi(\theta^\star)h,D\Phi(\theta^\star)h'\bigr).
\]

We now prove the first inclusion
\[
\ker DC(\theta^\star)\subset \ker \nabla^2\mathcal R(\theta^\star).
\]
Let \(h\in \ker DC(\theta^\star)\). The condition \(DC(\theta^\star)h=0\) means that the first three components of
\(D\Phi(\theta^\star)h\) vanish. Hence
\[
D\Phi(\theta^\star)h=(0,0,0,\delta s)
\]
for some \(\delta s\in\mathbb R^d\).
By Lemma~\ref{lem:ker_hessian_tildeR_on_E1}, since \(\theta^\star\in \Theta_1^{\mathrm{reg}}\) and thus
\(s(\theta^\star)\neq 0\), we have
$
(0,0,0,\delta s)\in \ker \nabla^2\widetilde{\mathcal R}(\zeta^\star).
$
Therefore,
\[
D^2\mathcal R(\theta^\star)(h,h)
=
D^2\widetilde{\mathcal R}(\zeta^\star)\bigl(D\Phi(\theta^\star)h,D\Phi(\theta^\star)h\bigr)
=0,
\]
so $h\in \ker \nabla^2\mathcal R(\theta^\star)$. This proves
\[
\ker DC(\theta^\star)\subset \ker \nabla^2\mathcal R(\theta^\star).
\]

Conversely, let
\[
h\in \ker \nabla^2\mathcal R(\theta^\star).
\]
Then
\[
0=D^2\mathcal R(\theta^\star)(h,h)
=
D^2\widetilde{\mathcal R}(\zeta^\star)\bigl(D\Phi(\theta^\star)h,D\Phi(\theta^\star)h\bigr).
\]
Since \(\zeta^\star\in \mathcal E_1\) with \(s(\theta^\star)\neq 0\), Lemma~\ref{lem:ker_hessian_tildeR_on_E1} implies that
\[
D\Phi(\theta^\star)h=(0,0,0,\delta s')
\]
for some \(\delta s'\in\mathbb R^d\).
Hence the first three components vanish, that is,
\[
DC(\theta^\star)h=0.
\]
Thus $h\in \ker DC(\theta^\star)$ and we obtain
\[
\ker \nabla^2\mathcal R(\theta^\star)\subset \ker DC(\theta^\star).
\]

We conclude that
\[
\ker \nabla^2\mathcal R(\theta^\star)=\ker DC(\theta^\star)=T_{\theta^\star}\mathcal S.
\]
\end{proof}

\subsection{Local gradient-flow convergence (proof of Theorem~\ref{thm:lda-dynamics})}
\label{appendix:proof_thm_local_cv_theta1}
\label{app:lda-local-proof}

The Morse--Bott structure provides local coordinates in which the excess risk is quadratic in the normal variables and independent of the tangential variables. To obtain convergence of the original Euclidean gradient flow, however, we must also account for the metric induced by this coordinate change. The argument below shows exponential decay in the normal directions, controls the total tangential displacement, and chooses an initial neighborhood from which the trajectory cannot leave the coordinate chart. This yields convergence to a single Bayes-optimal parameter, not merely approach to the set of minimizers.

The following statement makes the neighborhoods and the normal decay estimate in Theorem~\ref{thm:lda-dynamics} explicit.

\paragraph{Notation.}
We recall the parameterization and the regular component used in the local argument.
\[
\theta=(U_{11},u_{12},v_{21},v_{22})
\in \mathbb{R}^{d\times d}\times \mathbb{R}^d\times \mathbb{R}^d\times \mathbb{R},
\qquad
\Delta m\coloneqq m_1-m_{-1},
\]
and define
\[
w(\theta)\coloneqq U_{11}\Gamma v_{21},\qquad
b(\theta)\coloneqq v_{21}^{\top}m_{-1}-v_{22},
\]
\[
r(\theta)\coloneqq v_{21}^{\top}\Delta m+2v_{22},
\qquad
s(\theta)\coloneqq 2u_{12}+U_{11}\Delta m .
\]
We consider the population risk
\[
\mathcal R(\theta)=\widetilde{\mathcal R}\bigl(w(\theta),b(\theta),r(\theta),s(\theta)\bigr).
\]
We also consider
\[
\Theta_1
\coloneqq 
\Bigl\{
\theta:\ w(\theta)=w^\star,\ \ b(\theta)=b^\star,\ \ r(\theta)=0
\Bigr\},
\]
and its regular part
\[
\Theta_1^{\mathrm{reg}}
\coloneqq 
\Bigl\{
\theta\in \Theta_1:\ \ s(\theta)\neq 0
\Bigr\}.
\]
For a fixed point $\theta^\star\in \Theta_1^{\mathrm{reg}}$, we denote by $\mathcal{S}$ the connected component of \(\Theta_1^{\mathrm{reg}}\) containing \(\theta^*\).

\begin{theorem}[Local convergence of the gradient flow near a regular critical component]
\label{thm:local_convergence_theta1}
Assume that $m_1\neq -\,m_{-1}$, $\theta^\star\in \Theta_1^{\mathrm{reg}}$. Let \(\mathcal S\) be the connected component of \(\Theta_1^{\mathrm{reg}}\) containing \(\theta^\star\).
Then there exist:

\begin{itemize}
\item an open neighborhood \(U\) of \(\theta^\star\) in the parameter space,
\item an open neighborhood \(V\subset \mathbb{R}^{n}\times \mathbb{R}^{d+2}\) of \(0\), where $n=\dim \mathcal S=d^2+d-1$,
\item a smooth diffeomorphism
\[
\psi:U\longrightarrow V,
\qquad
\psi(\theta^\star)=0,
\]
\end{itemize}
such that
\[
\psi(U\cap \mathcal S)=V\cap \{(x,y)\in \mathbb{R}^n\times \mathbb{R}^{d+2}: y=0\},
\]
and
\[
(\mathcal R\circ \psi^{-1})(x,y)=\mathcal R(\mathcal S)+|y|^2
\qquad\text{for all }(x,y)\in V.
\]

Moreover, there exists an open neighborhood \(U_0\subset U\) of \(\theta^\star\) such that, for every initial condition $\theta_0\in U_0$, the Euclidean gradient flow
\[
\dot\theta_t=-\nabla \mathcal R(\theta_t),
\qquad
\theta_{t=0}=\theta_0,
\]
is defined for all \(t\ge 0\), remains in \(U\), and converges to a point $\theta_\infty\in \mathcal S$.

More precisely, if we write
\[
z_t\coloneqq (x_t,y_t)\coloneqq \psi(\theta_t),
\]
then \(z_t\in V\) for all \(t\ge 0\), \(x_t\) converges in \(\mathbb{R}^n\), and there exists a constant \(c>0\) such that
\[
|y_t|\le e^{-ct}|y_0|
\qquad\text{for all }t\ge 0.
\]
In particular,
\[
\operatorname{dist}(\theta_t,\mathcal S)\longrightarrow 0
\qquad\text{and}\qquad
\theta_t\longrightarrow \theta_\infty\in \mathcal S
\quad\text{as }t\to+\infty.
\]
where the convergence rates are exponentials. 
\end{theorem}

\begin{proof}[Proof of Theorem~\ref{thm:lda-dynamics}]
We prove the detailed version stated in
Theorem~\ref{thm:local_convergence_theta1}, which implies the main-text
result. The argument is divided into the following steps.

\medskip

\noindent\textbf{Step 1: Rewriting the flow in Morse--Bott coordinates.}
By Lemma \ref{lem:R_morse_bott_on_Theta1_reg}, we can use the Morse--Bott theorem (Lemma \ref{lem:morse_bott_theor_local})
Let
\[
\psi:U\to V\subset \mathbb R^n\times \mathbb R^{d+2}
\]
be the local diffeomorphism given by the Morse--Bott lemma at $\theta^*$, and define
\[
F\coloneqq \mathcal R\circ \psi^{-1}.
\]
By construction, $F$ is the expression of the risk in local coordinates adapted to the critical manifold $\mathcal S$: the variable
\[
z=(x,y)\in \mathbb R^n\times \mathbb R^{d+2}
\]
splits into tangent coordinates $x$ along $\mathcal S$ and normal coordinates $y$, along which the risk has a quadratic structure.

Now let $(\theta_t)_{t\ge 0}$ be a solution of the Euclidean gradient flow
\[
\dot \theta_t=-\nabla \mathcal R(\theta_t),
\qquad
\theta_{t=0}=\theta_0,
\]
as long as $\theta_t\in U$, and define
\[
z_t\coloneqq \psi(\theta_t).
\]
Differentiating the identity $z_t=\psi(\theta_t)$, we obtain
\[
\dot z_t
=
D\psi(\theta_t)\dot\theta_t
=
-\,D\psi(\theta_t)\nabla \mathcal R(\theta_t).
\]
Set
\[
J(z)\coloneqq D\psi(\psi^{-1}(z)),
\qquad
K(z)\coloneqq D\psi^{-1}(z)=J(z)^{-1}.
\]

Next, by the chain rule, we have that
\[
\nabla \mathcal R(\psi^{-1}(z))
=
\nabla (F\circ\psi)(\psi^{-1}(z))
=
D\psi(\psi^{-1}(z))^\top \nabla F(z).
\]
Substituting this into the previous identity gives
\[
\dot z_t
=
-\,J(z_t)J(z_t)^\top \nabla F(z_t).
\]
Therefore, defining $G(z)\coloneqq J(z)J(z)^\top$, the flow in the $z$-coordinates takes the form
\[
\dot z_t=-\,G(z_t)\nabla F(z_t).
\]

Finally, the map $G$ is continuous on $V$, and for every $z\in V$, the matrix $G(z)$ is symmetric positive definite, with inverse
\[
G(z)^{-1}
=
K(z)^\top K(z)
=
D\psi^{-1}(z)^\top D\psi^{-1}(z).
\]
This shows that the original Euclidean gradient flow for $\mathcal R$ is equivalent, in the coordinates given by $\psi$, to a gradient-type flow for $F$ in the tangent/normal splitting.

\medskip

\noindent\textbf{Step 2: Writing the flow in the \((x,y)\)-coordinates and preliminary estimates.}

By the Morse--Bott normal form, for every
$
(x,y)\in V\subset \mathbb R^n\times \mathbb R^{d+2},
$
we have
\[
F(x,y)=\mathcal R(\mathcal S)+|y|^2,
\]
hence
\[
\nabla F(x,y)=\binom{0}{2y}.
\]
From Step~1, the flow in the \(z\)-coordinates satisfies
\[
\dot z_t=-\,G(z_t)\nabla F(z_t),
\qquad
z_t=(x_t,y_t).
\]
We write the symmetric positive definite matrix \(G(x,y)\) in block form as
\[
G(x,y)=
\begin{pmatrix}
A(x,y) & B(x,y)\\
B(x,y)^\top & C(x,y)
\end{pmatrix},
\]
where
\[
A(x,y)\in \mathbb R^{n\times n},
\qquad
B(x,y)\in \mathbb R^{n\times(d+2)},
\qquad
C(x,y)\in \mathbb R^{(d+2)\times(d+2)}.
\]
Therefore,
\[
\binom{\dot x_t}{\dot y_t}
=
-
\begin{pmatrix}
A(x_t,y_t) & B(x_t,y_t)\\
B(x_t,y_t)^\top & C(x_t,y_t)
\end{pmatrix}
\binom{0}{2y_t},
\]
that is,
\[
\dot x_t=-2\,B(x_t,y_t)\,y_t,
\qquad
\dot y_t=-2\,C(x_t,y_t)\,y_t.
\]

We now fix a compact set $K\coloneqq \overline B_{\mathbb R^n}(0,\rho_x)\times \overline B_{\mathbb R^{d+2}}(0,\rho_y) \subset V$. Our goal will be to show that, on the compact set $K$, the eigenvalues of B and C are bounded below, which will allow us to use Grönwall's lemma to obtain exponential convergence on this compact set, and then to adjust the size of the permitted initialization perimeter to ensure that convergence occurs before leaving the compact set.

Since $C$ is continuous on \(V\) and takes values in the set of symmetric positive definite matrices,
and since the map $M\longmapsto \lambda_{\min}(M)$ is continuous on the space of symmetric matrices, the function
$
z\longmapsto \lambda_{\min}(C(z))
$
is continuous on \(K\). Moreover, it is strictly positive on \(K\). Hence we may define
\[
\lambda\coloneqq \min_{z\in K}\lambda_{\min}(C(z))>0.
\]
We also set
\[
M\coloneqq \sup_{z\in K}\|B(z)\|_{\mathrm{op}}<+\infty.
\]

For the remainder of the proof, we choose \(\varepsilon_x,\varepsilon_y>0\) such that
\[
\varepsilon_x+\frac{M}{\lambda}\varepsilon_y<\rho_x/2,
\qquad
\varepsilon_y< \rho_y/2,
\]
and we consider an initial condition $\theta_{t=0}=\theta_0=\psi^{-1}(x_0,y_0)$ such that
\[
|x_0|<\varepsilon_x,
\qquad
|y_0|\le \varepsilon_y.
\]
Finally, we define
\[
T^*
\coloneqq 
\sup\Bigl\{
T>0:\ z_t\in \operatorname{int}(K)\ \text{for all }t\in[0,T]
\Bigr\}.
\]
We will show below that
$
T^*=+\infty.
$

\medskip

\noindent\textbf{Step 3: Control in the normal directions.}

For \(t\in[0,T^*)\), we have
\[
\dot y_t=-2\,C(x_t,y_t)\,y_t.
\]
Define
$
V(y)\coloneqq \|y\|^2
$
so that
\[
\frac{d}{dt}V(y_t)
=
2\langle y_t,\dot y_t\rangle
=
-4\langle y_t,C(x_t,y_t)y_t\rangle.
\]
Since \(z_t=(x_t,y_t)\in \operatorname{int}(K)\subset K\) for every \(t<T^*\), and by definition of \(\lambda\),
\[
\langle y,C(z)y\rangle\ge \lambda \|y\|^2
\qquad
\text{for all }z\in K,\ y\in\mathbb R^{d+2},
\]
we obtain
\[
\frac{d}{dt}V(y_t)\le -4\lambda V(y_t).
\]
By Grönwall's lemma,
\[
V(y_t)\le e^{-4\lambda t}V(y_0)
\qquad\text{for all }t<T^*.
\]
Equivalently,
\[
\|y_t\|\le e^{-2\lambda t}\|y_0\|
\qquad\text{for all }t<T^*.
\]

\medskip

\noindent\textbf{Step 4: Control in the tangent directions.}

For \(t\in[0,T^*)\), we also have
\[
\dot x_t=-2\,B(x_t,y_t)\,y_t.
\]
Hence, using the definition of \(M\),
\[
\|\dot x_t\|
\le 2\,\|B(x_t,y_t)\|_{\mathrm{op}}\,\|y_t\|
\le 2M\|y_t\|.
\]
Combining this with the estimate from Step~3 gives
\[
\|\dot x_t\|
\le 2M e^{-2\lambda t}\|y_0\|.
\]
Integrating on \([0,t]\subset [0,T^*)\), we obtain
\[
\|x_t\|
\le
\|x_0\|+\int_0^t \|\dot x_s\|\,ds
\le
\|x_0\|+2M\|y_0\|\int_0^t e^{-2\lambda s}\,ds.
\]
Therefore,
\[
\|x_t\|
\le
\|x_0\|+\frac{M}{\lambda}\bigl(1-e^{-2\lambda t}\bigr)\|y_0\|
\le
\|x_0\|+\frac{M}{\lambda}\|y_0\|
\qquad\text{for all }t<T^*.
\]
In particular,
\[
\|x_t\|
<
\varepsilon_x+\frac{M}{\lambda}\varepsilon_y
\le \rho_x/2,
\]
and from Step~3,
\[
\|y_t\|\le \|y_0\|\le \varepsilon_y\le \rho_y/2.
\]
Thus \(z_t=(x_t,y_t)\) cannot reach the boundary of \(K\) on \([0,T^*)\).
By continuity and the definition of \(T^*\), this implies
\[
T^*=+\infty.
\]

\medskip

\noindent\textbf{Conclusion.}
By Step~3, we have
\[
\|y_t\|\le e^{-2\lambda t}\|y_0\|
\qquad\text{for all }t\ge 0,
\]
hence
\[
y_t\longrightarrow 0
\qquad\text{as }t\to+\infty.
\]
Moreover, by Step~4,
\[
\|x_t\|\le \|x_0\|+\frac{M}{\lambda}\|y_0\|<\rho_x
\qquad\text{for all }t\ge 0,
\]
so \((x_t)\) remains in the compact ball \(\overline B(0,\rho_x)\). Since
\[
\int_0^{+\infty}\|\dot x_t\|\,dt
\le
2M\|y_0\|\int_0^{+\infty}e^{-2\lambda t}\,dt
<+\infty,
\]
the curve \((x_t)\) has finite length and therefore converges to some point
\[
x_\infty\in \overline B(0,\rho_x).
\]
Consequently,
\[
z_t=(x_t,y_t)\longrightarrow (x_\infty,0)
\qquad\text{as }t\to+\infty.
\]
Since \(\psi\) is a diffeomorphism, it follows that
\[
\theta_t=\psi^{-1}(z_t)\longrightarrow \theta_\infty\coloneqq \psi^{-1}(x_\infty,0).
\]
Finally, because
\[
(x_\infty,0)\in V\cap\{(x,y):y=0\}=\psi(U\cap \mathcal S),
\]
we obtain
\[
\theta_\infty\in \mathcal S\cap U.
\]
The exponential convergence in the original coordinate $\theta_t$ can be deduced from the fact that $\psi$ is a locally lipschitz, hence
$\|\theta_t-\theta_\infty\|\leq L\|z_t-z_\infty\|$.
This concludes the proof.
\end{proof}

\subsection{Constant-gate geometry (complement to Theorem~\ref{thm:lda-dynamics})}
\label{app:lda-theta2-geometry}

The two Bayes representation mechanisms need not yield the same local geometry. On the regular part of $\Theta_1$, a nonconstant gate has zero amplitude, and the preceding argument identifies all flat directions with motion along the minimizer manifold. On $\mathcal E_2$, the gate is constant: a perturbation of its direction can be compensated to first order by a change in the affine coefficient. The following calculation identifies the resulting additional Hessian kernel.

\begin{proposition}[Failure of the Morse--Bott property]
\label{prop:failure_morse_bott_E2}
Let
$\zeta^\star=(w^\star,b^\star-r^\star\sigma(s_0),r^\star,0)\in\mathcal E_2$ such that $r^\star\neq0$.
Then
\[
\ker\nabla^2\widetilde{\mathcal R}(\zeta^\star)
=T_{\zeta^\star}\mathcal E_2\oplus E_{\rm flat},
\]
where
\[
E_{\rm flat}=\{(-r^\star\sigma'(s_0)\delta s,0,0,\delta s):
\delta s\in\mathbb R^d\}.
\]
In particular, the Hessian kernel has dimension $d+1$, whereas
$T_{\zeta^\star}\mathcal E_2$ has dimension one.
\end{proposition}

\begin{corollary}[Failure of the Morse--Bott property in attention coordinates]
\label{cor:failure_morse_bott_theta2}
Assume \(m_1\neq-m_{-1}\) and let \(\theta^\star\in\Theta_2\) such that $r(\theta^\star)\neq0$. Then
\[
\dim\ker\nabla^2\mathcal R(\theta^\star)
-\dim T_{\theta^\star}\Theta_2=d>0.
\]
Thus the constant-gate component \(\Theta_2\) is not Morse--Bott.
\end{corollary}

\begin{proof}[Proof of Proposition~\ref{prop:failure_morse_bott_E2} and Corollary~\ref{cor:failure_morse_bott_theta2}]
Let $r^\star\in\mathbb{R}\setminus \{0\}$, and
\[
\zeta^\star =(w^\star,b^\star-r^\star\,\sigma(s_0),r^\star,0)\in \mathcal{E}_2
\]
A direct computation gives:
\[
T_{\zeta^\star}\mathcal E_2
=
\left\{
(0,-\sigma(s_0)\delta r,\delta r,0)
\;:\;
\delta r\in\mathbb R
\right\}.
\].
Moreover, the same computation as Proposition~\ref{prop:hessian_tildeR_on_E1} gives :
\[
\nabla^2 \widetilde{\mathcal R}(\zeta^\star)
=
\mathbb E\Big[
\eta(X)\bigl(1-\eta(X)\bigr)\,
\nabla g_{\zeta^\star}(X)\,\nabla g_{\zeta^\star}(X)^\top
\Big],
\]
Since $g_{\zeta}(x)=
w^\top x+b+r\,\sigma(s_0+s^\top x),$ we obtain:
\[
\nabla g_{\zeta^\star}(X)
=
\begin{pmatrix}
X\\[0.2em]
1\\[0.2em]
\sigma(s_0)\\[0.2em]
r^\star\sigma'(s_0)X
\end{pmatrix}.
\]
One can deduce :
\[
D^2 \widetilde{\mathcal R}(\zeta^\star)[\delta\zeta,\delta\zeta]
=E\Big[
\eta(X)\bigl(1-\eta(X)\bigr)\,
\bigl( \left( \delta{w}+ r^\star\sigma'(s_0)\delta s\right)^\top X +\delta b +\sigma(s_0)\delta r\bigr)^2
\Big]
\]
Hence,
\[
\ker \nabla^2 \widetilde{\mathcal R}(\zeta^\star) =\bigl\{ (-r^\star\sigma'(s_0)\delta s,-\sigma(s_0)\delta r, \delta r, \delta s)\: : \: \delta r\in \mathbb{R},\: \delta s\in \mathbb{R} ^d\bigr\}
\]
This proves the proposition.

For the corollary, Lemma~\ref{lem:Dphi_surjective} combined with the fact that $v_{21}(\theta^\star)\neq0$ shows that $D\Phi(\theta^\star)$ is surjective. We have $\nabla^2\mathcal{R}(\theta^\star)=D\Phi(\theta^\star)^\top\nabla^2 \widetilde{\mathcal R}(\zeta^\star)D\Phi(\theta^\star)$, so $\ker\nabla^2\mathcal{R}(\theta^\star)=D\Phi(\theta^\star)^{-1}\ker\nabla^2 \widetilde{\mathcal R}(\zeta^\star)$. Moreover, because $D\Phi(\theta^\star)$ is surjective, we have $T_{\theta^\star}\Theta_2=D\Phi(\theta^\star)^{-1}T_{\zeta^\star}\mathcal{E}_2$ (see \cite{guillemin1974differential} Chap. 1 $\S5$ exercise 5) . We can conclude using Rank-nullity theorem. Indeed, for any subspace $V\subset\mathbb R^{2d+2}=\operatorname{Im}D\Phi(\theta^\star)$:
\[
\dim D\Phi(\theta^\star)^{-1}(V)= \dim \ker D\Phi(\theta^\star) +\dim V
\]
Thus
\[
\dim \ker \nabla^2 \mathcal R(\theta^\star)
    -
    \dim T_{\theta^\star}\Theta_2
=
\dim \ker \nabla^2 \widetilde{\mathcal R}(\zeta^\star)
-
\dim T_{\zeta^\star}\mathcal E_2 
= d
\]

\end{proof}

These extra directions distinguish the constant-gate branch from the regular zero-amplitude branch. At points of $\mathcal E_2$ with $r^\star\neq0$, the local minimizer set is the constant-gate branch, yet the Hessian kernel is strictly larger than its tangent space. At $r^\star=0$, the two effective branches meet, so the regular-manifold picture also requires separate care. Under the asymmetry assumption of the corollary, the same excess of flat directions transfers to attention coordinates.

This calculation explains why the preceding Morse--Bott argument cannot be reused near $\Theta_2$. It does not, by itself, rule out convergence there or determine a slower rate: those questions require control beyond the Hessian. The diagnostics in Figure~\ref{fig:lda-local} illustrate the contrast between the two branches without supplying such a convergence theorem.

\subsection{Reduced-risk estimates and conserved balance (preparation for Theorem~\ref{thm:lda-krylov})}
\label{appendix:aux_lemma_invariant_submani_sym}

We now assume $m_1=-m_{-1}\eqqcolon m$ and $\pi_+=1/2$. The regularity argument based on $D\Phi$ no longer applies, but symmetry offers another route: on $\mathcal I_K$, the bias and the sigmoid contribution vanish, and the risk reduces to the linear logistic objective $L(w)$ with $w=U_{11}\Gamma v_{21}$.

Convexity in $w$ alone is not enough, because training still evolves two factors. The lemmas below supply the additional controls: coercivity bounds the effective predictor, positive curvature on the resulting compact region gives a uniform convexity constant, and a conserved balance between the factors prevents the effective dynamics from stalling under the prescribed initialization. These ingredients are assembled in Appendix~\ref{app:lda-krylov-proof}.

\begin{lemma}[Coercivity of the reduced linear risk]
\label{lem:coercivity_linear_reduced_risk}
Let
\[
L(w)
\coloneqq 
\widetilde{\mathcal R}(w,0,0,0)
=
\mathbb E\bigl[\ell_{\log}(Y,w^\top X)\bigr].
\]
Then \(L\) is coercive, with a linear lower bound. More precisely, there exists a constant
\(C_1>0\) such that, for every
\(w\in\mathbb R^d\),
\[
L(w)
\geq
C_1\|w\|.
\]
In particular,
\[
L(w)\longrightarrow +\infty
\qquad
\text{as }\|w\|\to+\infty.
\]
\end{lemma}

\begin{proof}
Since $m_{-1}=-m_1$, we first observe that $YX \sim \mathcal N(m_1,\Gamma).$ For every $(y,t)\in\{-1,1\}\times\mathbb R$, the logistic loss satisfies
\[
\ell_{\log}(y,t)
=
\log(1+\exp(-yt))
\geq
(-yt)_+.
\]
Therefore, for every $w\in\mathbb R^d$,
\[
L(w)
=
\mathbb E\bigl[\ell_{\log}(Y,w^\top X)\bigr]
\geq
\mathbb E\bigl[(-w^\top YX)_+\bigr].
\]

Set
\[
Z_w
\sim
\mathcal N(\mu_w,\sigma_w^2),
\qquad
\mu_w\coloneqq -w^\top m_1,
\qquad
\sigma_w^2\coloneqq w^\top\Gamma w.
\]
Thus, for $w\neq 0$, using Lemma~\ref{lem:positive_part_gaussian},
\[
\mathbb E\bigl[(-w^\top YX)_+\bigr]
=
\mathbb E[(Z_w)_+]
=
\mu_w
\Phi\!\left(\frac{\mu_w}{\sigma_w}\right)
+
\sigma_w
\varphi\!\left(\frac{\mu_w}{\sigma_w}\right).
\]
So
\[
\mathbb E[(Z_w)_+]
=
\sigma_w
\left[
q_w \Phi(q_w)+\varphi(q_w)
\right],
\qquad
q_w\coloneqq \frac{\mu_w}{\sigma_w}.
\]

We show that the term in brackets is bounded from below by a positive
constant. Since $\Gamma$ is symmetric positive definite, by Cauchy Schwarz
inequality we get
\[
|w^\top m_1|
=
\left<\Gamma^{1/2} w,\Gamma^{-1/2}m_1 \right>
\leq
(w^\top\Gamma w)^{1/2}
(m_1^\top\Gamma^{-1}m_1)^{1/2}.
\]
Hence
\[
|q_w|
\leq
M,
\qquad
M\coloneqq (m_1^\top\Gamma^{-1}m_1)^{1/2}.
\]
Let $g(t)\coloneqq t\Phi(t)+\varphi(t)$, the function $g$ is continuous and strictly positive on $\mathbb R$, since (still by Lemma~\ref{lem:positive_part_gaussian})
\[
g(t)=\mathbb E[G_+],
\qquad G\sim\mathcal N(t,1).
\]
Therefore, we can define
\[
c_0\coloneqq  \inf_{w\neq0}q_w \Phi(q_w)+\varphi(q_w)=\min_{t\in[-M,M]} g(t)>0.
\]
It follows that, for every $w\neq 0$,
\[
\mathbb E[(Z_w)_+]
\geq
c_0 \sigma_w.
\]
Moreover, since $\Gamma$ is symmetric positive definite,
\[
\sigma_w
=
(w^\top\Gamma w)^{1/2}
\geq
\sqrt{\lambda_{\min}(\Gamma)}\,\|w\|.
\]
Consequently,
\[
L(w)
\geq
\mathbb E[(Z_w)_+]
\geq
c_0\sqrt{\lambda_{\min}(\Gamma)}\,\|w\|.
\]
Thus the desired inequality holds with
\[
C_1\coloneqq c_0\sqrt{\lambda_{\min}(\Gamma)}>0
\]
The case $w=0$ is immediate since $L(0)=\log2\geq 0$.
\end{proof}

\begin{lemma}[Local strong convexity along the effective trajectory]
\label{lem:local_strong_convexity_along_trajectory}
Let \((\theta_t)_{t\geq 0}\) be a gradient-flow trajectory initialized at
\(\theta_0\in\mathcal I_K\), and let
\[
w_t=U_{11}(t)\Gamma v_{21}(t).
\]
Define
\[
\mathcal C_{\theta_0}
\coloneqq 
\overline{
\operatorname{Conv}
\left(
\{w_t:t\geq 0\}\cup\{w^\star\}
\right)
}.
\]
Then there exists a constant
$
\mu_0=\mu_0(\theta_0)>0
$
such that \(L\) is \(\mu_0\)-strongly convex on \(\mathcal C_{\theta_0}\). In
other words, for every \(w,w'\in\mathcal C_{\theta_0}\),
\[
L(w')
\geq
L(w)
+
\langle \nabla L(w),w'-w\rangle
+
\frac{\mu_0}{2}\|w'-w\|^2.
\]
\end{lemma}

\begin{proof}
Since the trajectory remains in \(\mathcal I_K\), one has
\[
\mathcal R(\theta_t)=L(w_t)
\qquad
\text{for every }t\geq 0.
\]
Moreover, \(\mathcal R\) is non-increasing along the gradient flow. Hence
\[
L(w_t)
=
\mathcal R(\theta_t)
\leq
\mathcal R(\theta_0)
=
L(w_0)
\qquad
\text{for every }t\geq 0.
\]
This implies that \((w_t)_{t\geq 0}\) is bounded. Indeed, if
\[
\limsup_{t\to+\infty}\|w_t\|=+\infty,
\]
then, by Lemma~\ref{lem:coercivity_linear_reduced_risk}, one would have
\[
\limsup_{t\to+\infty} L(w_t)=+\infty,
\]
which contradicts the uniform bound
\[
L(w_t)\leq L(w_0).
\]
Thus \((w_t)_{t\geq 0}\) is bounded. It follows that
$
\{w_t:t\geq 0\}\cup\{w^\star\}
$
is bounded, and therefore, because we work in finite dimension,
$
\mathcal C_{\theta_0}
=
\overline{
\operatorname{Conv}
\left(
\{w_t:t\geq 0\}\cup\{w^\star\}
\right)
}
$
is compact.

Now \(L\) is the population logistic risk restricted to linear scores,
\[
L(w)=\mathbb E\bigl[\ell_{\log}(Y,w^\top X)\bigr].
\]
Its Hessian is
\[
\nabla^2L(w)
=
\mathbb E\left[
\sigma'(w^\top X)XX^\top
\right].
\]
Under the non-degeneracy of the LDA distribution, this Hessian is positive
definite for every \(w\in\mathbb R^d\). Since \(\nabla^2L\) is continuous
and \(\mathcal C_{\theta_0}\) is compact, the map
\[
w\mapsto \lambda_{\min}\bigl(\nabla^2L(w)\bigr)
\]
attains a positive minimum on \(\mathcal C_{\theta_0}\). We define
\[
\mu_0
\coloneqq 
\min_{w\in\mathcal C_{\theta_0}}
\lambda_{\min}\bigl(\nabla^2L(w)\bigr)
>0.
\]
Therefore,
\[
\nabla^2L(w)\succeq \mu_0 I
\qquad
\text{for every }w\in\mathcal C_{\theta_0}.
\]
This is exactly the \(\mu_0\)-strong convexity of \(L\) on
\(\mathcal C_{\theta_0}\).
\end{proof}

\begin{lemma}[Conservation of the balance on \(\mathcal I_K\)]
\label{lem:conservation_balance_NK}
Let
\[
\theta_t=(U_{11}(t),u_{12}(t),v_{21}(t),v_{22}(t))
\]
be a gradient-flow trajectory such that
\[
\theta_t\in\mathcal I_K
\qquad
\text{for every }t\geq 0.
\]
Then
\[
\|U_{11}(t)\|_F^2-\|v_{21}(t)\|^2
=
\|U_{11}(0)\|_F^2-\|v_{21}(0)\|^2
\qquad
\text{for every }t\geq 0.
\]
\end{lemma}

\begin{proof}
On \(\mathcal I_K\), the flow reduces to
\[
\dot U_{11}(t)=-G_{w_t}(\Gamma v_{21}(t))^\top,
\qquad
\dot v_{21}(t)=-\Gamma U_{11}(t)^\top G_{w_t},
\]
where
\[
G_{w_t}\coloneqq \mathbb E\left[
\bigl(\sigma(w^\top X)-\eta(X)\bigr)X
\right]
=
\nabla L(w_t),
\qquad
w_t=U_{11}(t)\Gamma v_{21}(t).
\]
(see proof of Lemma~\ref{lem:stability_krylov_submanifold} for more details)

Set
\[
F(t)\coloneqq \|U_{11}(t)\|_F^2-\|v_{21}(t)\|^2.
\]
Then, using \(\Gamma=\Gamma^\top\),
\[
\begin{aligned}
F'(t)
&=
2\langle U_{11}(t),\dot U_{11}(t)\rangle_F
-
2\langle v_{21}(t),\dot v_{21}(t)\rangle \\
&=
-2G_{w_t}^\top U_{11}(t)\Gamma v_{21}(t)
+
2v_{21}(t)^\top\Gamma U_{11}(t)^\top G_{w_t} \\
&=
-2G_{w_t}^\top w_t
+
2G_{w_t}^\top w_t \\
&=0.
\end{aligned}
\]
Thus \(F(t)\) is constant along the flow, which proves the claim.
\end{proof}

\subsection{Feasibility of the Krylov initialization (complement to Theorem~\ref{thm:lda-krylov})}

The convergence theorem requires a nonzero initialization in $K^\perp$, so it is useful to determine when this space is available. The answer is spectral: successive applications of $\Gamma$ to $m$ separate the distinct eigenvalues represented in $m$, but cannot generate every direction inside a repeated eigenspace. The following proposition states the resulting criterion.

\begin{proposition}
\label{prop:krylov_generated_space}
For symmetric positive-definite $\Gamma$, the Krylov space
$K=\operatorname{span}\{m,\Gamma m,\ldots,\Gamma^{d-1}m\}$ equals
$\mathbb R^d$ if and only if $\Gamma$ has $d$ distinct eigenvalues and the
projection of $m$ onto every eigenspace of $\Gamma$ is nonzero.
\end{proposition}

\begin{proof}
Since $\Gamma$ is symmetric, there exists an orthonormal basis
$(e_1,\ldots,e_d)$ of $\mathbb R^d$ made of eigenvectors of $\Gamma$. We write
\[
\Gamma e_i=\lambda_i e_i,
\qquad i=1,\ldots,d,
\]
and
\[
m=\sum_{i=1}^d m_i e_i.
\]

We first prove the direct implication by contraposition. Assume first that
$\Gamma$ does not have $d$ distinct eigenvalues. Without loss of generality,
suppose that $\lambda_1=\lambda_2$. Let $P$ be the orthogonal projection onto $E\coloneqq \operatorname{span}(e_1,e_2)$. Then, for every $k\geq 0$,
\[
P\Gamma^k m
=
\lambda_1^k(m_1e_1+m_2e_2).
\]
Therefore,
\[
P K
=
\operatorname{span}\{P\Gamma^k m : k=0,\ldots,d-1\}
\subset
\operatorname{span}\{m_1e_1+m_2e_2\}.
\]
Hence $\dim(PK)\leq 1$, whereas $\dim(E)=2$. Thus $K$ cannot be equal to
$\mathbb R^d$.

Assume now that the component of $m$ in one eigenspace is zero. For
instance, suppose that the projection of $\Delta m$ onto the eigenspace
associated with $\lambda_1$ is zero, which means $m_1=0$. Let $P_1$ be the orthogonal projection
onto that eigenspace. We have
for every $k\geq 0$,
\[
P_1\Gamma^k m
=
\Gamma^k P_1m
=
0.
\]
Therefore,
\[
P_1 K=\{0\},
\]
so $K$ cannot be equal to $\mathbb R^d$. This proves that if
$K=\mathbb R^d$, then $\Gamma$ has $d$ distinct eigenvalues and $m$ has
a nonzero component in each eigenspace of $\Gamma$.

Conversely, assume that $\Gamma$ has $d$ distinct eigenvalues
$\lambda_1,\ldots,\lambda_d$ and that the components of $ m$ in the
corresponding eigendirections are all nonzero (which means $m_i\neq 0 $ for $ i=1,\ldots,d$).

We prove that the family
\[
 m,\Gamma m,\ldots,\Gamma^{d-1} m
\]
is linearly independent. Let $a_1,\ldots,a_d\in\mathbb R$ be such that
\[
\sum_{\ell=1}^d a_\ell \Gamma^{\ell-1} m=0.
\]
We have
\[
\Gamma^{\ell-1} m
=
\sum_{i=1}^d m_i \lambda_i^{\ell-1} e_i.
\]
Therefore,
\[
0
=
\sum_{\ell=1}^d a_\ell \Gamma^{\ell-1} m
=
\sum_{i=1}^d
m_i
\left(
\sum_{\ell=1}^d a_\ell \lambda_i^{\ell-1}
\right)
e_i.
\]
Since $(e_1,\ldots,e_d)$ is a basis, it follows that for every
$i=1,\ldots,d$,
\[
m_i
\left(
\sum_{\ell=1}^d a_\ell \lambda_i^{\ell-1}
\right)
=0.
\]
Since $m_i\neq 0$, we obtain
\[
\sum_{\ell=1}^d a_\ell \lambda_i^{\ell-1}=0,
\qquad i=1,\ldots,d.
\]
Now define the polynomial
\[
P(X)\coloneqq \sum_{\ell=1}^d a_\ell X^{\ell-1}.
\]
Then
\[
P(\lambda_i)=0,
\qquad i=1,\ldots,d.
\]
Hence $P$ is a polynomial of degree at most $d-1$ having $d$ distinct roots.
Therefore $P$ is identically zero. Consequently,
\[
a_1=\cdots=a_d=0.
\]
Thus the family
\[
m,\Gamma m,\ldots,\Gamma^{d-1}m
\]
is linearly independent. Since it contains $d$ vectors in $\mathbb R^d$, it is
a basis of $\mathbb R^d$. This proves $K=\mathbb R^d$.

\end{proof}

\begin{remark}[Initialization and the codimension of the Krylov space]
\label{remark:init_sym_setting}
The only nontrivial initialization constraint is
\(v_{21}(0)\in K^\perp\setminus\{0\}\). When \(\dim K\ll d\), an isotropic
random initialization has only a fraction of order \(\dim(K)/d\) of its
energy in \(K\). In particular, if \(\Gamma=\gamma^2 I_d\), then
\(K=\operatorname{span}(m)\), so a normalized Gaussian initialization lies at
distance \(O_{\mathbb P}(d^{-1/2})\) from \(K^\perp\).
\end{remark}

The proximity of a random initialization to $K^\perp$ should not be confused with exact membership: the convergence theorem below assumes the latter. When $K\neq\mathbb R^d$, this constraint can be enforced by choosing a nonzero vector in $K^\perp$.

\subsection{Invariance of the Krylov subspace (preparation for Theorem~\ref{thm:lda-krylov})}

The Krylov constraint is designed to be preserved by the flow. Orthogonality to $m$ alone would not suffice, since the factor dynamics also involve multiplication by $\Gamma$. By construction, $K$ and $K^\perp$ are $\Gamma$-invariant. Together with $U_{11}P_K=0$, this closes the dynamics on $\mathcal I_K$ and keeps the predictor linear throughout training.

\begin{lemma}
\label{lem:stability_krylov_submanifold}
In the symmetric balanced setting, if $\theta_0\in\mathcal I_K$, then the
gradient-flow trajectory remains in $\mathcal I_K$ for every $t\geq0$.
\end{lemma}

\begin{proof}
Let
$
\theta^K=(U_{11},u_{12},v_{21},v_{22})\in\mathcal I_K.
$
By definition of \(\mathcal I_K\), one has
\[
u_{12}=0,
\qquad
v_{22}=0,
\qquad
v_{21}\in K^\perp,
\qquad
U_{11}P_K=0.
\]
Since \(\Delta m\in K\), this implies $v_{21}^\top\Delta m=0$ and $U_{11}\Delta m=0.$ Therefore, on \(\mathcal I_K\),
\[
b(\theta^K)=0,
\qquad
r(\theta^K)=0,
\qquad
s(\theta^K)=0.
\]

We now set $G_w
\coloneqq 
\mathbb E\left[
\bigl(\sigma(w^\top X)-\eta(X)\bigr)X
\right].
$
Using Lemma~\ref{lem:grad_of_R} (and the fact that $s_0=0$ in the symmetric setting), we immediately obtain
\[
\nabla_{u_{12}}\mathcal R(\theta^K)=0, \qquad \nabla_{v_{22}}\mathcal R(\theta^K)=0, \qquad
\nabla_{U_{11}}\mathcal R(\theta^K)=G_w(\Gamma v_{21})^\top, \qquad
\nabla_{v_{21}}\mathcal R(\theta^K)=\Gamma U_{11}^\top G_w.
\]
We now show that the vector field of the gradient flow stays in
\(\mathcal I_K\). We already know that at \(\theta^K\in\mathcal I_K\), $\dot u_{12}=0$ and $\dot v_{22}=0$. Let \(x_K\in K\). Since \(v_{21}\in K^\perp\) and \(K^\perp\) is stable under
\(\Gamma\), one has $\Gamma v_{21}\in K^\perp.$
Therefore
\[
\dot U_{11}x_K
=
-\nabla_{U_{11}}\mathcal R(\theta^K)x_K
=
-G_w(\Gamma v_{21})^\top x_K
=
0.
\]
Since this holds for every \(x_K\in K\), we obtain $\dot U_{11}P_K=0.$

It remains to check the constraint on \(v_{21}\). Since \(U_{11}P_K=0\), one
has $\operatorname{Im}(U_{11}^\top)\subset K^\perp$ and thus $U_{11}^\top G_w\in K^\perp$. Because \(K^\perp\) is stable under \(\Gamma\), it follows that
\[
\dot v_{21}
=
-\nabla_{v_{21}}\mathcal R(\theta^K)
=
-\Gamma U_{11}^\top G_w
\in K^\perp.
\]

We have therefore shown that, whenever \(\theta_t\in\mathcal I_K\),
\[
\dot u_{12}(t)=0,
\qquad
\dot v_{22}(t)=0,
\qquad
\dot U_{11}(t)P_K=0,
\qquad
\dot v_{21}(t)\in K^\perp.
\]
Equivalently,
\[
\dot\theta_t\in \mathcal I_K 
\]

It follows that if $\theta_0\in\mathcal I_K$, then the whole trajectory remains in \(\mathcal I_K\).
\end{proof}

\subsection{Global convergence on the invariant subspace (proof of Theorem~\ref{thm:lda-krylov})}
\label{app:lda-krylov-proof}

We now combine invariance with the controls established above. The proof first derives a preconditioned gradient flow for $w_t=U_{11}(t)\Gamma v_{21}(t)$. The conserved balance and nonzero initialization provide a uniform lower bound on its preconditioner, while coercivity keeps $w_t$ in a region of uniform positive curvature. These facts yield exponential convergence of the risk and effective predictor. A final estimate then bounds the factors and makes their velocities integrable, establishing convergence in the original attention parameters as well.

\begin{proof}[Proof of Theorem~\ref{thm:lda-krylov}]
Let
$
\theta_0=(U_{11}(0),u_{12}(0),v_{21}(0),v_{22}(0))
$
satisfy the initialization assumptions of the theorem, namely
\[
U_{11}(0)=0,\qquad u_{12}(0)=0,\qquad v_{22}(0)=0,
\qquad v_{21}(0)\in K^\perp\setminus\{0\}.
\]
Then \(\theta_0\in\mathcal I_K\). By
Lemma~\ref{lem:stability_krylov_submanifold}, the whole trajectory remains in
\(\mathcal I_K\). In particular,
\[
u_{12}(t)=0,\qquad v_{22}(t)=0,\qquad b_t=r_t=s_t=0.
\]
We first prove the convergence of the effective variable
$
w_t\coloneqq U_{11}(t)\Gamma v_{21}(t).
$
Set
\[
L(w)\coloneqq \widetilde{\mathcal R}(w,0,0,0).
\]
On \(\mathcal I_K\), one has
$
\mathcal R(\theta_t)=L(w_t).
$
Moreover, the gradient flow reduces to
\[
\dot U_{11}(t)= -\nabla L(w_t)(\Gamma v_{21}(t))^\top,
\qquad
\dot v_{21}(t)= -\Gamma U_{11}(t)^\top \nabla L(w_t).
\]
Therefore,
\[
\begin{aligned}
\dot w_t
&=
\dot U_{11}(t)\Gamma v_{21}(t)
+
U_{11}(t)\Gamma \dot v_{21}(t)\\
&=
-\|\Gamma v_{21}(t)\|^2\nabla L(w_t)
-
U_{11}(t)\Gamma^2U_{11}(t)^\top\nabla L(w_t).
\end{aligned}
\]
Thus the effective dynamics can be written as
\[
\dot w_t=-G_t\nabla L(w_t),
\]
where
\[
G_t
\coloneqq 
\|\Gamma v_{21}(t)\|^2 I_d
+
U_{11}(t)\Gamma^2U_{11}(t)^\top.
\]
By Lemma~\ref{lem:conservation_balance_NK},
\[
\|U_{11}(t)\|_F^2-\|v_{21}(t)\|^2
=
\|U_{11}(0)\|_F^2-\|v_{21}(0)\|^2
=-\|v_{21}(0)\|^2
\eqqcolon C_0<0
\]
Hence, for every \(t\geq0\),
\[
\|v_{21}(t)\|^2
=
\|U_{11}(t)\|_F^2-C_0
\geq -C_0.
\]
We set $\nu\coloneqq \sqrt{-C_0}>0$ so that
\[
\|v_{21}(t)\|\geq\nu
\qquad\text{for every }t\geq0.
\]
Since \(\Gamma=\Gamma^\top\succ0\),
\[
\|\Gamma v_{21}(t)\|
\geq
\lambda_{\min}(\Gamma)\|v_{21}(t)\|
\geq
\lambda_{\min}(\Gamma)\nu.
\]
Moreover, \(U_{11}(t)\Gamma^2U_{11}(t)^\top\) is positive semidefinite. Hence
\[
G_t
\succeq
\|\Gamma v_{21}(t)\|^2 I_d
\succeq
\alpha I_d,
\]
where
$
\alpha\coloneqq 
\lambda_{\min}(\Gamma)^2\nu^2>0.
$

By Lemma~\ref{lem:local_strong_convexity_along_trajectory}, \(L\) is
\(\mu_0\)-strongly convex on
\[
\mathcal C_{\theta_0}
=
\overline{
\operatorname{Conv}
\left(
\{w_t:t\geq0\}\cup\{w^\star\}
\right)
}.
\]
Thus, for every \(t\geq0\),
\[
\|\nabla L(w_t)\|^2
\geq
2\mu_0\bigl(L(w_t)-L(w^\star)\bigr).
\]
Using the effective dynamics, we get
\[
\begin{aligned}
\frac{d}{dt}\bigl(L(w_t)-L(w^\star)\bigr)
&=
\langle \nabla L(w_t),\dot w_t\rangle\\
&=
-\langle \nabla L(w_t),G_t\nabla L(w_t)\rangle\\
&\leq
-\alpha \|\nabla L(w_t)\|^2\\
&\leq
-2\alpha\mu_0\bigl(L(w_t)-L(w^\star)\bigr).
\end{aligned}
\]
By Grönwall's inequality,
\[
L(w_t)-L(w^\star)
\leq
e^{-2\alpha\mu_0 t}
\bigl(L(w_0)-L(w^\star)\bigr).
\]
In particular, the risk converges exponentially:
\[
\mathcal R(\theta_t)-\mathcal R^\star
=
L(w_t)-L(w^\star)
\lesssim
e^{-2\alpha\mu_0 t}.
\]
Since \(L\) is \(\mu_0\)-strongly convex on \(\mathcal C_{\theta_0}\), we also
have
\[
L(w_t)-L(w^\star)
\geq
\frac{\mu_0}{2}\|w_t-w^\star\|^2.
\]
Therefore
\[
\|w_t-w^\star\|
\lesssim
e^{-\alpha\mu_0 t}.
\]

We finally prove the convergence of the factors \(U_{11}(t)\) and
\(v_{21}(t)\). We first show that they are bounded. The function
\(L\) is smooth. Indeed,
\[
\nabla^2L(w)
=
\mathbb E\left[\sigma'(w^\top X)XX^\top\right],
\]
and since \(0\leq \sigma'\leq 1/4\), one has
\[
\nabla^2L(w)
\preceq
\frac14\mathbb E[XX^\top].
\]
Thus \(L\) is \(M\)-smooth with
$
M\coloneqq \frac14\left\|\mathbb E[XX^\top]\right\|_{\mathrm{op}}.
$
In particular,
\[
\|\nabla L(w_t)\|
=
\|\nabla L(w_t)-\nabla L(w^\star)\|
\leq
M\|w_t-w^\star\|
\lesssim
e^{-\alpha\mu_0 t}.
\]

Set
\[
S(t)\coloneqq \|U_{11}(t)\|_F^2+\|v_{21}(t)\|^2.
\]
Using the reduced dynamics on \(\mathcal I_K\), we get
\[
\begin{aligned}
S'(t)
&=
2\langle U_{11}(t),\dot U_{11}(t)\rangle_F
+
2\langle v_{21}(t),\dot v_{21}(t)\rangle\\
&= -2\langle U_{11}(t),\nabla L(w_t)(\Gamma v_{21}(t))^\top\rangle_F
-
2\langle v_{21}(t),\Gamma U_{11}(t)^\top \nabla L(w_t)\rangle\\
&=
-4\langle w_t,\nabla L(w_t)\rangle.
\end{aligned}
\]
Hence
\[
|S'(t)|
\leq
4\|w_t\|\,\|\nabla L(w_t)\|
\lesssim
e^{-\alpha\mu_0 t},
\]
because \((w_t)_{t\geq0}\) is bounded. Therefore \(S'\in L^1(\mathbb R_+)\),
and
\[
S(t)
=
S(0)+\int_0^t S'(\tau)\,d\tau
\]
converges to a finite limit as \(t\to+\infty\). In particular,
\[
\sup_{t\geq0}\|U_{11}(t)\|_F<+\infty,
\qquad
\sup_{t\geq0}\|v_{21}(t)\|<+\infty.
\]

We now prove convergence of the factors. Since
\[
\dot v_{21}(t)
=
-\Gamma U_{11}(t)^\top\nabla L(w_t),
\]
we have
\[
\|\dot v_{21}(t)\|
\leq
\|\Gamma\|_{\mathrm{op}}\|U_{11}(t)\|_{\mathrm{op}}
\|\nabla L(w_t)\|
\lesssim
e^{-\alpha\mu_0 t}.
\]
Thus \(\dot v_{21}\in L^1(\mathbb R_+)\), and \(v_{21}(t)\) converges. More
precisely, there exists \(v_\infty\in\mathbb R^d\) such that
\[
v_{21}(t)\to v_\infty,
\]
and
\[
\|v_{21}(t)-v_\infty\|
\leq
\int_t^{+\infty}\|\dot v_{21}(s)\|\,ds
\lesssim
e^{-\alpha\mu_0 t}.
\]
Similarly,
\[
\dot U_{11}(t)
=
-\nabla L(w_t)(\Gamma v_{21}(t))^\top,
\]
so
\[
\|\dot U_{11}(t)\|_F
\leq
\|\nabla L(w_t)\|\,
\|\Gamma\|_{\mathrm{op}}\|v_{21}(t)\|
\lesssim
e^{-\alpha\mu_0 t}.
\]
Hence \(\dot U_{11}\in L^1(\mathbb R_+)\), and there exists
\(U_\infty\in\mathbb R^{d\times d}\) such that
\[
U_{11}(t)\to U_\infty,
\]
with
\[
\|U_{11}(t)-U_\infty\|_F
\leq
\int_t^{+\infty}\|\dot U_{11}(s)\|_F\,ds
\lesssim
e^{-\alpha\mu_0 t}.
\]

Finally, we obtain
\[
w^\star
=
\lim_{t\to+\infty}w_t
=
\lim_{t\to+\infty}U_{11}(t)\Gamma v_{21}(t)
=
U_\infty\Gamma v_\infty.
\]
\end{proof}

\subsection{LDA experiments (illustrations of Theorems~\ref{thm:lda-dynamics} and~\ref{thm:lda-krylov})}
\label{appendix:LDA_experimental_details}

\begin{table}[h]
    \centering
    \begin{tabular}{c c c c c c}
        \hline
        Configuration
        & dimension
        & $m_1$
        & $m_{-1}$
        & $\Gamma$
        & $\pi_+$ \\
        \hline
        C1
        & $d=2$
        & $(1,0)^\top$
        & $(0,1)^\top$
        & $I_2$
        & $0.6$ \\[0.4cm]
        C2
        & $d=20$
        & $e_1$
        & $-e_1$
        & $I_{20}$
        & $0.5$ \\
        
        \hline
    \end{tabular}
    \caption{LDA configurations used in the numerical experiments.}
    \label{tab:configurations}
\end{table}

\begin{table}[h]
    \centering
    \renewcommand{\arraystretch}{1.2}
    \begin{tabular}{c c c c}
        \toprule
        \textbf{Figure} &
        \textbf{Problem} &
        
        \textbf{Dimension} &
        \textbf{Numerical method}
        \\
        \midrule
        Fig.~\ref{fig:lda-local} 
        & Config. C1
        & $d=2$
        & GH, $\mathrm{step\_size=10^{-4}}$ \\
        
        Fig.~\ref{fig:risk-sym-invariant-random-init-d=20}
        & Config. C2
        & $d=20$
        & SGD, $\mathrm{step\_size=10^{-2}}$ \\

        Fig.~\ref{fig:risk-sym-v21-random-d=20}
        & Config. C2
        & $d=20$
        & SGD, $\mathrm{step\_size=10^{-2}}$ \\

    \end{tabular}
    
    \caption{Technical details of the numerical experiments associated with each figure.}
    \label{tab:numerical_details}
\end{table}

\begin{table}[h]
    \centering
    \small
    
    \renewcommand{\arraystretch}{2.5}

    \begin{tabularx}{\textwidth}{
        >{\centering\arraybackslash}p{1.5cm}
        >{\centering\arraybackslash}p{2.5cm}
        >{\centering\arraybackslash}p{1.5cm}
        >{\centering\arraybackslash}p{1.5cm}
        >{\centering\arraybackslash}p{5.0cm}
        >{\centering\arraybackslash}p{1.5cm}
    }
        \toprule
        \textbf{Figure}
        & $\mathbf{U}_{11}$
        & $\mathbf{u}_{12}$
        & $\mathbf{v}_{21}$
        & $v_{22}$
        & \textbf{Noise} \\
        \midrule

        Fig.~\ref{fig:lda-local} ($\theta^{(1)}_t$)
        &
        $\begin{pmatrix}
        \dfrac{1}{2b_\star} & 0 \\
        -\dfrac{1}{2b_\star} & 1
        \end{pmatrix}$
        &
        $\begin{pmatrix}10\\10\end{pmatrix}$
        &
        $\begin{pmatrix}
        2b_\star\\
        0
        \end{pmatrix}$
        &
        $-b_\star$
        &
        $\varepsilon=10^{-6}$
        \\[2ex]

        Fig.~\ref{fig:lda-local} ($\theta^{(2)}_t$)
        &
        $\begin{pmatrix}
        1 & 0\\
        0 & 1
        \end{pmatrix}$
        &
        $\begin{pmatrix}
        -\frac12\\
        \frac12
        \end{pmatrix}$
        &
        $\begin{pmatrix}
        1\\
        -1
        \end{pmatrix}$
        &
        $\displaystyle
        \frac{
        {v}_{21}^{\top}{m}_{-1}
        - b_\star
        + \pi_{+}{v}_{21}^{\top}
        ({m}_{1}-{m}_{-1})
        }{
        1-2\pi_{+}
        }$
        &
        $\varepsilon=10^{-6}$
        \\[2ex]

        \bottomrule
    \end{tabularx}
    \caption{Specific initializations used in numerical experiments.}
    \label{tab:initializations}
\end{table}

\paragraph{Initialization in the balanced symmetric setting.}
We compare two initialization strategies for stochastic gradient
descent in dimension $d=20$, with balanced classes and opposite means.
With Gaussian initialization of all parameters,
Figure~\ref{fig:risk-sym-invariant-random-init-d=20} illustrates
the difficulty of approaching the Bayes risk in the displayed runs.
By contrast, setting
$U_{11}(0)=0$, $u_{12}(0)=0$, and $v_{22}(0)=0$,
while sampling $v_{21}(0)\sim\mathcal N(0,I_d)$,
leads to a rapid, approximately exponential decrease toward a
small excess risk
(Figure~\ref{fig:risk-sym-v21-random-d=20}).
This initialization requires no knowledge of $m$ or $\Gamma$.
Although it does not enforce the exact constraint
$v_{21}(0)\in K^\perp$ required by
Theorem~\ref{thm:lda-krylov}, these experiments suggest that
the favorable behavior can persist beyond the invariant subspace.

\begin{figure}[htbp]
    \centering
    \includegraphics[width=0.6\linewidth]
    {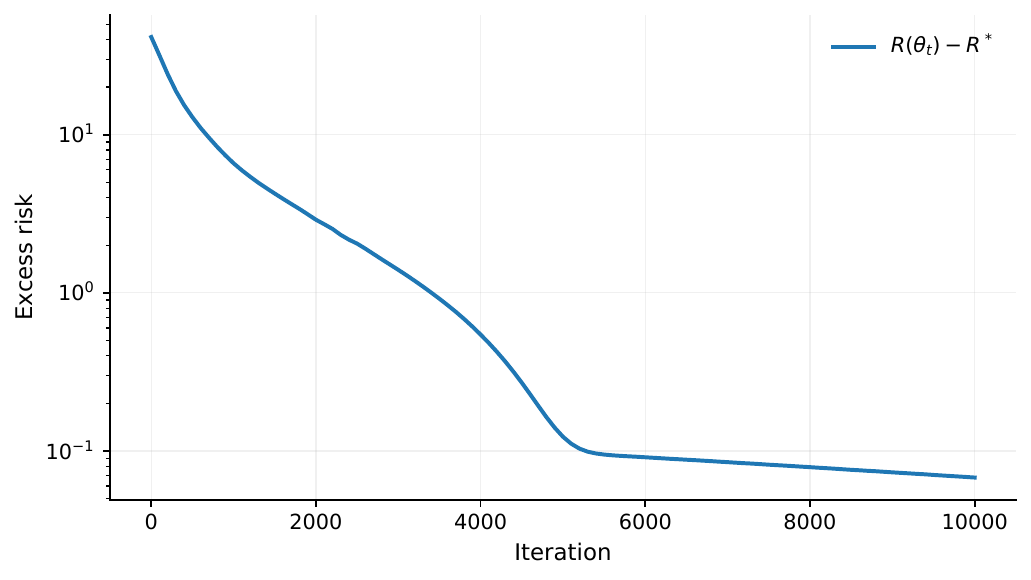}
    \caption{Risk evolution under stochastic gradient descent
    in the balanced symmetric setting with $d=20$.
    All parameter entries are initialized from a standard
    Gaussian distribution.}
    \label{fig:risk-sym-invariant-random-init-d=20}
\end{figure}

\begin{figure}[htbp]
    \centering
    \includegraphics[width=0.6\linewidth]
    {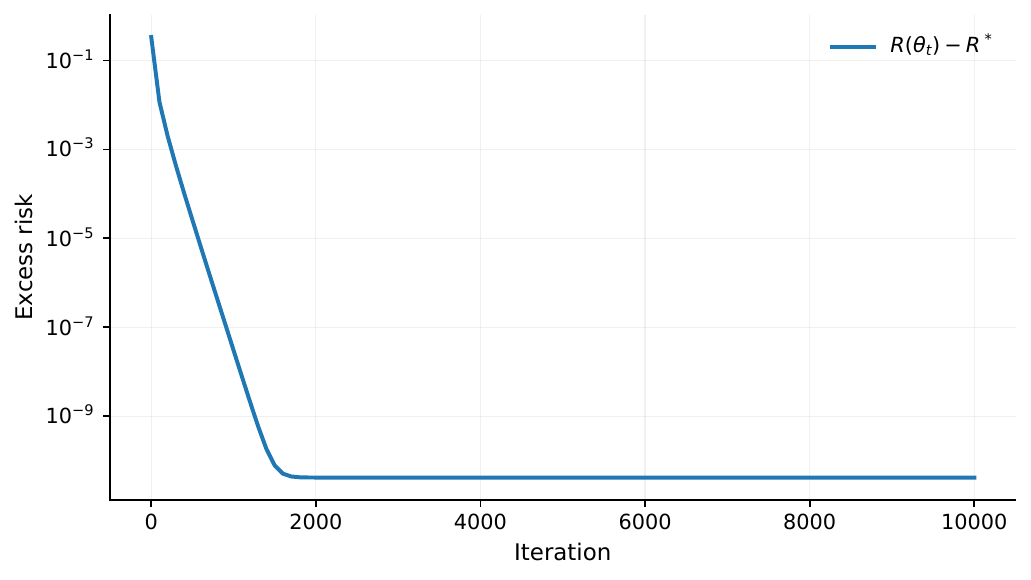}
    \caption{Risk evolution under stochastic gradient descent
    in the balanced symmetric setting with $d=20$.
    Initialization:
    $U_{11}(0)=0$, $u_{12}(0)=0$, $v_{22}(0)=0$,
    and $v_{21}(0)\sim\mathcal N(0,I_d)$.}
    \label{fig:risk-sym-v21-random-d=20}
\end{figure}

\section{Proofs and Complements for Section~\ref{sec:denoising}}
\label{appendix:denoising}

\subsection{Bayes denoiser (derivation of Equation~\eqref{eq:bayes-denoiser})}
\label{app:bayes-denoiser-proof}

\begin{proposition}[Bayes denoiser for a Gaussian mixture]
\label{prop:optimal-denoiser-gmm}
Let \(X\sim\sum_{k=1}^M\pi_k\mathcal N(m_k,\Gamma)\),
\(Z=X+\varepsilon\) with
\(\varepsilon\sim\mathcal N(0,\xi^2I_d)\) independent of \(X\), and
\(S=\Gamma+\xi^2I_d\). Then
\[
\mathbb E[X\mid Z=z]
=\Gamma S^{-1}z+\xi^2S^{-1}\sum_{k=1}^M\beta_k(z)m_k,
\]
where
\[
\beta_k(z)=
\frac{\pi_k\mathcal N(z;m_k,S)}
{\sum_{\ell=1}^M\pi_\ell\mathcal N(z;m_\ell,S)}.
\]
\end{proposition}

\begin{proof}
Recalling that the conditional expectation is the minimizer of the mean squared error, the function $f^\star : z \mapsto \mathbb{E}[X \mid Z=z]$ is an optimal denoiser. It remains to compute this conditional expectation explicitly.

Let $C$ denote the latent cluster variable, so that $\mathbb{P}(C=k)=\pi_k$ and $X\mid C=k \sim \mathcal{N}(m_k,\Gamma)$. By Bayes' formula,
\[
\mathbb{P}(C=k\mid Z=z)
=
\frac{\pi_k \, \mathcal{N}(z;m_k,S)}
{\sum_{\ell=1}^M \pi_\ell \, \mathcal{N}(z;m_\ell,S)}
\eqqcolon 
\beta_k(z).
\]

Using the law of total expectation with respect to the latent variable $C$, we obtain
\[
\mathbb{E}[X\mid Z=z]
=
\sum_{k=1}^M
\mathbb{E}[X\mid Z=z,C=k]\,
\mathbb{P}(C=k\mid Z=z).
\]

It remains to compute $\mathbb{E}[X\mid Z=z,C=k]$. Conditionally on $C=k$, the pair $(X,Z)$ is jointly Gaussian. Moreover, $\operatorname{Cov}(X,Z\mid C=k)
=
\operatorname{Cov}(X,X+\varepsilon\mid C=k)
=
\Gamma$. The Gaussian conditioning formula (e.g. see \cite{Rasmussen2006Gaussian}) therefore gives
\[
\mathbb{E}[X\mid Z=z,C=k]
=
m_k+\Gamma S^{-1}(z-m_k).
\]
Equivalently,
\[
\mathbb{E}[X\mid Z=z,C=k]
=
\Gamma S^{-1}z+\xi^2 S^{-1}m_k.
\]
Indeed,
\[
I_d-\Gamma S^{-1}=(S-\Gamma)S^{-1}=\xi^2S^{-1},
\]
which explains the noise-variance coefficient $\xi^2$.
Plugging this expression into the previous decomposition and using that $\sum_{k=1}^M \beta_k(z)=1$, we finally obtain
\[
\mathbb{E}[X\mid Z=z]
=
\Gamma S^{-1}z
+
\sum_{k=1}^M
\beta_k(z)\,\xi^2 S^{-1}m_k.
\]
This proves the result.

It is worth noting that this result could also be derived from a multivariate version of Tweedie's formula, which states that
\[
\mathbb{E}[X\mid Z=z]
=
z+\xi^2 \nabla_z \log p_Z(z).
\]

\end{proof}

\begin{remark}[High- and low-noise limits]
For fixed $z$, means, and covariance, as $\xi\to\infty$ the posterior
weights $\beta_k(z)$ converge to the constant priors $\pi_k$ , $\Gamma S^{-1}\to0$, and
$\xi^2S^{-1}\to I_d$. Hence $f^\star(z)\to\sum_k\pi_km_k$.
If $\Gamma$ is positive definite, then as $\xi\downarrow0$,
$\Gamma S^{-1}\to I_d$ and $\xi^2S^{-1}\to0$, so $f^\star(z)\to z$.
The latter assertion does not apply to the degenerate clean mixture
$\Gamma=0$.
\end{remark}

\subsection{Expressivity and risk separation (complete proof of Theorem~\ref{thm:denoise-expressivity})}
\label{app:denoise-expressivity-proof}

\begin{proof}[Proof of Theorem~\ref{thm:denoise-expressivity}]
Proposition~\ref{prop:two-head-nondegenerate-denoiser} below constructs
the two heads, including the case of opposite means.
Proposition~\ref{prop:one-head_attention_cant_rpz_non-degenerate_denoising}
proves that a single head requires $\gamma^2=\xi^2$. Conversely, at this
balance $S=2\xi^2I_d$, and choosing $U=S^{-1}$, $V=I_d/2$ gives
\[
T_{U,V}[\mu_Z](z)=\tfrac12z+\tfrac12\sum_k\beta_k(z)m_k=f^\star(z).
\]
Finally, Proposition~\ref{prop:softmax_beat_linear_for_denoising}
proves the strict excess-risk gap, not merely failure of exact linear
representation. Together these results establish all claims of the theorem.
The degenerate complement stated after the theorem follows from
Proposition~\ref{prop:two-head-degenerate-denoiser} and the same gap argument.
\end{proof}

\subsection{Degenerate mixtures (complement to Theorem~\ref{thm:denoise-expressivity})}
\label{app:two-head-degenerate}

\begin{proposition}[Degenerate mixture: two heads are sufficient]
\label{prop:two-head-degenerate-denoiser}
Assume $d\geq2$, $\xi>0$, $\Gamma=0$, $M=2$, positive mixture weights,
and distinct means with $\|m_1\|=\|m_2\|$. Then a
two-head softmax-attention predictor realizes the Bayes denoiser exactly. If,
in addition, \(d=2\) and \(m_1,m_2\) are linearly independent, no single
head realizes it.
\end{proposition}

\begin{proof}
We first prove the single-head obstruction. Since \(\Gamma=0\), the Bayes
denoiser is the bounded posterior barycenter
\(f^\star(z)=\beta_1(z)m_1+\beta_2(z)m_2\). If one head represented it,
comparison with
\[
T_{U,V}[\mu_Z](z)=\xi^2VU z
+V\bigl(\alpha_1(z)m_1+\alpha_2(z)m_2\bigr)
\]
would force \(VU=0\), since the first term is linear and the target is
bounded. Matching the two barycenters gives, up to exchanging the
components, \(Vm_i=m_i\) for \(i=1,2\), while matching the posterior weights
gives
\(U^\top(m_1-m_2)=\xi^{-2}(m_1-m_2)\). In dimension two, linear
independence of $m_1,m_2$ implies $V=I_2$ in the first orientation.
In the exchanged orientation, $V$ swaps the two basis vectors and is
also invertible. Thus, in either case $VU=0$ forces $U=0$, making the
attention weights constant, a contradiction.

We now construct two heads. We begin with the case where $m_1$ and $m_2$ are linearly independent. The idea is to let each head be responsible for one cluster. More precisely, we want the first head to recover the contribution of $m_1$ and the second head to recover the contribution of $m_2$.

Let $\Delta m \coloneqq  m_1-m_2$. We choose two value matrices $V_1,V_2$ such that
\[
V_1m_1=m_1,
\qquad
V_1m_2=0,
\]
and
\[
V_2m_1=0,
\qquad
V_2m_2=m_2.
\]
In other words, $V_1$ extracts the contribution of the first cluster, while $V_2$ extracts the contribution of the second cluster.

We then choose $U_1$ and $U_2$ so that their images are killed by the corresponding value matrices, while their attention scores still reproduce the optimal posterior weights. For instance, one may take
\[
U_1
=
\frac{1}{\xi^2}
\frac{m_2 \Delta m^\top}{\langle m_2,\Delta m\rangle},
\qquad
U_2
=
\frac{1}{\xi^2}
\frac{m_1 \Delta m^\top}{\langle m_1,\Delta m\rangle}.
\]
Since $\|m_1\|=\|m_2\|$ and $m_1\neq m_2$, the denominators are nonzero.

With this choice, we have
\[
\operatorname{Im}(U_1)\subset \operatorname{span}(m_2),
\qquad
\operatorname{Im}(U_2)\subset \operatorname{span}(m_1),
\]
and therefore
\[
V_1U_1=0,
\qquad
V_2U_2=0.
\]
Thus, the linear parts of both attention heads vanish. Moreover,
\[
U_1^\top \Delta m
=
\frac{1}{\xi^2}\Delta m,
\qquad
U_2^\top \Delta m
=
\frac{1}{\xi^2}\Delta m.
\]
Hence both heads produce the same posterior weight $\beta_1(z)$ as the optimal denoiser.

We now address the case $m_1=-m_2\eqqcolon m$. Here, a single head is sufficient. Let $n\neq0$ satisfy $\langle n,m\rangle=0$, which is possible because $d\geq2$. One can choose $U=\frac{(m+n)m^\top}{\xi^2\|m\|^2}$ and $V=m\left(\frac{m^\top}{\|m\|^2}-\frac{n^\top}{\|n\|^2}\right)$. Indeed, $Vm=m$, $Vn=-m$, and hence $VU=0$, while
$U^\top m=\xi^{-2}m$. This gives the optimal denoiser. Setting a second head to zero
also gives the asserted two-head representation.
\end{proof}

\subsection{Denoising-to-clustering transition (proof of Proposition~\ref{prop:phase-transition})}
\label{app:phase-transition-proof}
\label{appendix:proof_iterated_denoising}

\begin{proof}[Proof of Proposition~\ref{prop:phase-transition}]
Let $c=(m_1+m_2)/2$, $\delta=(m_1-m_2)/2$, and
$\lambda=\|\delta\|/\xi$. Equal norms between $m_1$ and $m_2$ imply $\langle c,\delta\rangle=0$.
Since the mixture is balanced and $\gamma=0$, $f^\star(z)=m_2+\Delta m\cdot\sigma(\langle z, \Delta m\rangle/\xi^2)$. Using
$2\sigma(2t)-1=\tanh(t)$ gives
\[
f^\star(z)=c+\delta\tanh\!\left(
\frac{\langle z-c,\delta\rangle}{\xi^2}\right).
\]
In particular, one step maps every query into the open segment joining the means. More precisely, for every $n\geq1$, there exists $u_n\in(-1,1)$ such that $z_n=c+u_n\delta$. Set $u_0=\langle z_0-c,\delta\rangle/\|\delta\|^2$ and
$F(u)=\tanh(\lambda^2u)$. We can rewrite the dynamic $z_{n+1}=f^\star(z_n)$ in term of $u_n$ as follows :
\[
u_{n+1}=F(u_n),\qquad z_n=c+u_n\delta\quad(n\geq1),
\]
and $|u_n|<1$ for $n\geq1$. The map $F$ is odd, strictly increasing,
and strictly concave on $(0,\infty)$.

\paragraph{Weak separation, including the threshold.}
If $\lambda\leq1$, then for every $u>0$,
$0<F(u)=\tanh(\lambda^2u)<\lambda^2u\leq u$.
Thus a positive trajectory decreases and remains bounded below by zero.
Its limit is a fixed point by continuity, and the same strict inequality
shows that zero is the only nonnegative fixed point. Oddness gives the
same conclusion for negative initial values. Therefore $u_n\to0$ and
$z_n\to c$, including when $\lambda=1$. Boundedness alone would not
suffice here; monotonicity supplies the convergence argument.

\paragraph{Strong separation and the basins of attraction.}
If $\lambda>1$, let $H(u)=F(u)-u$. We have $H(0)=0$,
$H'(0)=\lambda^2-1>0$, and $H(1)<0$.
Strict concavity on $(0,\infty)$ implies that $H$ has exactly one
positive zero $u_\lambda\in(0,1)$, with $H>0$ on $(0,u_\lambda)$
and $H<0$ on $(u_\lambda,\infty)$.
For $0<u<u_\lambda$, monotonicity of $F$ gives
$u<F(u)<u_\lambda$; for $u>u_\lambda$, it gives
$u_\lambda<F(u)<u$. Every positive trajectory is therefore monotone,
bounded, and converges to $u_\lambda$. By oddness, every negative
trajectory converges to $-u_\lambda$. The initial value zero remains zero.
Consequently the two limiting points are
\[
m_{\lambda,1}=c+u_\lambda\delta,\qquad
m_{\lambda,2}=c-u_\lambda\delta.
\]
Their basins are respectively
$\{z:\langle z-c,\delta\rangle>0\}$ and
$\{z:\langle z-c,\delta\rangle<0\}$.
The separating hyperplane maps directly to $c$ and has Lebesgue measure zero.
Moreover, strict concavity yields $F'(u_\lambda)<F(u_\lambda)/u_\lambda=1$,
while $F'>0$. The two outer fixed points are stable; the midpoint is
unstable along the line of the means because $F'(0)=\lambda^2>1$.

\paragraph{Limit of large separation.}
For any $r\in(0,1)$, $F(r)=\tanh(\lambda^2r)\to1$ as
$\lambda\to\infty$. Hence eventually $H(r)>0$, so $u_\lambda>r$.
Since $u_\lambda<1$, it follows that $u_\lambda\to1$.
For fixed means and $\xi\downarrow0$, this proves
$m_{\lambda,1}\to m_1$ and $m_{\lambda,2}\to m_2$.
\end{proof}

\subsection{Single-head necessity (Theorem~\ref{thm:denoise-expressivity}, single-head claim)}
\label{app:one-head-obstruction}

\begin{proposition}[A single head is generically insufficient]
\label{prop:one-head_attention_cant_rpz_non-degenerate_denoising}
Assume \(\Gamma=\gamma^2I_d\) with \(\gamma>0\), and retain the
two-component equal-norm setting of Section~\ref{sec:denoising}. If
\(\gamma^2\neq\xi^2\), there is no pair \((U,V)\) such that
\(T_{U,V}[\mu_Z]=f^\star\).
\end{proposition}

\begin{proof}
A single attention head takes the form
\[
T_{U,V}[\mu_Z](z)
=
(\gamma^2+\xi^2)VUz
+
V\left(
\alpha_1(z)m_1+\bigl(1-\alpha_1(z)\bigr)m_2
\right)
\quad \text{with} \quad 
\alpha_1(z)
=
\sigma\left(
\log\frac{\pi_1}{\pi_2}
+
\langle Uz,m_1-m_2\rangle
\right)
\]
If we try to impose
\[
T_{U,V}[\mu_Z](z)=f^\star(z)
\qquad
\text{for all }z\in\mathbb{R}^d,
\]
then boundedness of the gated terms first forces equality of the linear
parts (take $z=tx$ and divide by $t$). Put $\Delta m=m_1-m_2$ and
$a=\xi^2/(\gamma^2+\xi^2)$. The remaining identity is between two
nonconstant sigmoid maps. Their level hyperplanes coincide, so their score
vectors are parallel. Taking the two limits along $\Delta m$ identifies
the endpoint values as $am_1$ and $am_2$, possibly exchanged.
In the same orientation, strict monotonicity of the sigmoid then identifies
the linear argument within the sigmoid exactly. Thus the parameters must satisfy the following three conditions\footnote{When $\pi_1=\pi_2$, conditions~(2) and~(3) could alternatively be replaced by the sign-reversed score matching condition together with a permutation of the cluster means, but this does not affect the argument.}.
\begin{itemize}
\label{itemize:constraint_denoising_non-degenerate}
    \item[(1)] {Matching of the linear part:}
    \[
    (\gamma^2+\xi^2)VU
    =
    \frac{\gamma^2}{\gamma^2+\xi^2}I_d,
    \]

    \item[(2)] {Matching of the attention scores with the posterior weights:}
    \[
    U^\top(m_1-m_2)
    =
    \frac{1}{\gamma^2+\xi^2}(m_1-m_2).
    \]

    \item[(3)] {Preservation of the barycentric part with the correct scaling:}
    \[
    Vm_1
    =
    \frac{\xi^2}{\gamma^2+\xi^2}m_1,
    \qquad
    Vm_2
    =
    \frac{\xi^2}{\gamma^2+\xi^2}m_2.
    \]
\end{itemize}

   Since $\gamma^2>0$, the first condition now implies that both $U$ and $V$ are invertible.

From condition \((1)\), we have
\[
V
=
\frac{\gamma^2}{(\gamma^2+\xi^2)^2}U^{-1}.
\]
Using condition \((3)\), we get, for $i=1,2$,
\[
U m_i
=
\frac{\gamma^2}{\xi^2(\gamma^2+\xi^2)}m_i,
\qquad i=1,2.
\]
Thus $m_1$ and $m_2$ are eigenvectors of $U$ associated with the same eigenvalue. Hence, we also have $U\Delta m
=
\frac{\gamma^2}{\xi^2(\gamma^2+\xi^2)}\Delta m.
$
Therefore,
\[
\Delta m^\top U^\top \Delta m
=
(U\Delta m)^\top \Delta m
=
\frac{\gamma^2}{\xi^2(\gamma^2+\xi^2)}
\|\Delta m\|^2.
\]
On the other hand, condition \((2)\) gives $U^\top \Delta m
=
\frac{1}{\gamma^2+\xi^2}\Delta m.$ Taking the scalar product with $\Delta m$, we obtain
\[
\Delta m^\top U^\top \Delta m
=
\frac{1}{\gamma^2+\xi^2}\|\Delta m\|^2.
\]
Since $m_1\neq m_2$, comparing the two expressions yields
\[
\frac{\gamma^2}{\xi^2(\gamma^2+\xi^2)}
=
\frac{1}{\gamma^2+\xi^2},
\]
which is equivalent to $\gamma^2=\xi^2$.
For completeness, exchanging the endpoints instead gives
$Vm_1=am_2$, $Vm_2=am_1$ and
$U^\top\Delta m=-(\gamma^2+\xi^2)^{-1}\Delta m$.
The intercepts then require $\pi_1=\pi_2$.
Writing $q=\gamma^2/(\gamma^2+\xi^2)^2$, the relation $VU=qI_d$
implies $U\Delta m=-(q/a)\Delta m$. Taking the scalar product with
$\Delta m$ again gives $q/a=(\gamma^2+\xi^2)^{-1}$, hence the same
necessary equality $\gamma^2=\xi^2$. Consequently, if $\gamma^2\neq\xi^2$, no single attention head can satisfy the three constraints simultaneously.
\end{proof}

\subsection{Two-head construction (Theorem~\ref{thm:denoise-expressivity}, exact representation)}
\label{app:two-head-nondegenerate}

\begin{proposition}[Exact two-head representation in the nondegenerate case]
\label{prop:two-head-nondegenerate-denoiser}
Assume $d\geq2$, $\xi>0$, $\Gamma=\gamma^2I_d$ with $\gamma>0$,
$M=2$, positive mixture weights, and distinct means with
$\|m_1\|=\|m_2\|$. Then there exist two heads \((U_h,V_h)_{h=1}^2\)
such that
\[
T_{U_1,V_1}[\mu_Z](z)+T_{U_2,V_2}[\mu_Z](z)=f^\star(z)
\qquad\text{for every }z\in\mathbb R^d.
\]
\end{proposition}

\begin{proof}
Let $\Delta m=m_1-m_2$. We begin with the case $m_1\neq -m_2$. We look for parameters $(U_1,V_1)$ and $(U_2,V_2)$ satisfying the following three conditions.
First, both heads must reproduce the posterior weights:
\[
\tag{C1}
U_i^\top \Delta m
=
\frac{1}{\gamma^2+\xi^2}\Delta m,
\qquad i=1,2.
\]
Second, each value matrix must select one cluster contribution, with the correct scaling:
\[
\tag{C2}
V_im_i
=
\frac{\xi^2}{\gamma^2+\xi^2}m_i,
\qquad
V_im_j=0, \qquad (i,j) =(1,2),(2,1)
\]
Third, the sum of the linear parts must reproduce the linear term of the optimal denoiser:
\[
\tag{C3}
(\gamma^2+\xi^2)(V_1U_1+V_2U_2)
=
\frac{\gamma^2}{\gamma^2+\xi^2}I_d.
\]
These conditions are sufficient to exactly represent the optimal denoiser.

Since $m_1$ and $m_2$ are non-collinear, they span a two-dimensional subspace $E\coloneqq \operatorname{span}(m_1,m_2)$. We fix a complementary subspace $E^\perp$ and define $V_1,V_2$ as follows. On $E$, we impose
\[
V_1m_1
=
\frac{\xi^2}{\gamma^2+\xi^2}m_1,
\qquad
V_1m_2=0,
\]
and
\[
V_2m_1=0,
\qquad
V_2m_2
=
\frac{\xi^2}{\gamma^2+\xi^2}m_2.
\]
On the complementary subspace $E^\perp$, we set
\[
V_1x=V_2x
=
\frac{\xi^2}{2(\gamma^2+\xi^2)}x,
\qquad
x\in E^\perp.
\]
With this definition, condition \emph{(C2)} is satisfied. Moreover,
\[
V_1+V_2
=
\frac{\xi^2}{\gamma^2+\xi^2}I_d.
\]

We now set
\[
U_1
=
\frac{\gamma^2}{\xi^2(\gamma^2+\xi^2)}I_d
+
\frac{\xi^2-\gamma^2}{\xi^2(\gamma^2+\xi^2)}
\frac{m_2\Delta m^\top}{\langle m_2,\Delta m\rangle},
\]
and
\[
U_2
=
\frac{\gamma^2}{\xi^2(\gamma^2+\xi^2)}I_d
+
\frac{\xi^2-\gamma^2}{\xi^2(\gamma^2+\xi^2)}
\frac{m_1\Delta m^\top}{\langle m_1,\Delta m\rangle}.
\]

We first check condition \emph{(C1)}. For $U_1$, we have
\[
U_1^\top \Delta m
=
\frac{\gamma^2}{\xi^2(\gamma^2+\xi^2)}\Delta m
+
\frac{\xi^2-\gamma^2}{\xi^2(\gamma^2+\xi^2)}
\frac{\Delta m\,\langle m_2,\Delta m\rangle}{\langle m_2,\Delta m\rangle}.
\]
Hence
$U_1^\top \Delta m
=
\frac{1}{\gamma^2+\xi^2}\Delta m.$
The same computation gives
$U_2^\top \Delta m
=
\frac{1}{\gamma^2+\xi^2}\Delta m$. Thus condition \emph{(C1)} is verified.

It remains to verify condition \emph{(C3)}. By construction,
\[
V_1m_2=0,
\qquad
V_2m_1=0.
\]
Therefore, using the definitions of $U_1$ and $U_2$,
\[
V_1U_1
=
\frac{\gamma^2}{\xi^2(\gamma^2+\xi^2)}V_1,
\qquad
V_2U_2
=
\frac{\gamma^2}{\xi^2(\gamma^2+\xi^2)}V_2.
\]
Consequently,
\[
V_1U_1+V_2U_2
=
\frac{\gamma^2}{\xi^2(\gamma^2+\xi^2)}(V_1+V_2)
=
\frac{\gamma^2}{(\gamma^2+\xi^2)^2}I_d.
\]
We obtain
\[
(\gamma^2+\xi^2)(V_1U_1+V_2U_2)
=
\frac{\gamma^2}{\gamma^2+\xi^2}I_d,
\]
which is precisely condition \emph{(C3)}.

Thus conditions \emph{(C1)}, \emph{(C2)} and \emph{(C3)} all hold, so the associated two-head predictor satisfies
\[
\widehat f(z)=f^\star(z),
\qquad
\text{for all }z\in\mathbb{R}^d.
\]
This concludes the proof.

\medskip

For the case $m_1=-m_2\eqqcolon m$, the construction is similar. We replace condition~$(C2)$ by the following condition, which splits the contribution between the two heads:
\[
\tag{C2'}
V_im
=
\frac{\xi^2}{2(\gamma^2+\xi^2)}m,
\qquad i \in \{1,2\}
\]
Let $n\in\mathbb{R}^d$ be a unit vector such that $\langle m,n\rangle=0$. We define
\[
U_1
=
\frac{1}{\gamma^2+\xi^2}
\left(
I_d+\frac{nm^\top}{\|m\|}
\right),
\qquad
U_2
=
\frac{1}{\gamma^2+\xi^2}
\left(
I_d-\frac{nm^\top}{\|m\|}
\right),
\]
and
\[
V_1
=
\frac{\xi^2}{2(\gamma^2+\xi^2)}
\frac{mm^\top}{\|m\|^2}
+
\frac{\gamma^2}{2(\gamma^2+\xi^2)}
\left(
I_d-\frac{mm^\top}{\|m\|^2}
\right)
+
\frac{\gamma^2-\xi^2}{2(\gamma^2+\xi^2)}
\frac{mn^\top}{\|m\|},
\]
\[
V_2
=
\frac{\xi^2}{2(\gamma^2+\xi^2)}
\frac{mm^\top}{\|m\|^2}
+
\frac{\gamma^2}{2(\gamma^2+\xi^2)}
\left(
I_d-\frac{mm^\top}{\|m\|^2}
\right)
-
\frac{\gamma^2-\xi^2}{2(\gamma^2+\xi^2)}
\frac{mn^\top}{\|m\|}.
\]

Here is the verification. Write $s=\gamma^2+\xi^2$, $e=m/\|m\|$,
$P=ee^\top$, $Q=I_d-P$, and
\[
A=\frac{\xi^2}{2s}P+\frac{\gamma^2}{2s}Q,
\qquad B=\frac{\gamma^2-\xi^2}{2s}en^\top.
\]
Then $U_{1,2}=s^{-1}(I_d\pm ne^\top)$ and $V_{1,2}=A\pm B$.
Since $n\perp e$, $U_i^\top m=s^{-1}m$ and
$V_im=\xi^2m/(2s)$, proving (C1) and (C2'). Finally,
\begin{align*}
V_1U_1+V_2U_2
&=\frac{2}{s}(A+Bne^\top)\\
&=\frac{1}{s^2}\bigl(\xi^2P+\gamma^2Q
 +(\gamma^2-\xi^2)P\bigr)
=\frac{\gamma^2}{s^2}I_d.
\end{align*}
This is (C3), completing the opposite-means construction.

\end{proof}

\subsection{Strict risk gap (Theorem~\ref{thm:denoise-expressivity}, separation claim)}
\label{app:linear-attention-gap}

\begin{proposition}[Softmax attention strictly outperforms linear attention]
\label{prop:softmax_beat_linear_for_denoising}
Let \(\mathcal F_{\rm sm}\) and \(\mathcal F_{\rm lin}\) denote the classes
represented by finite multi-head softmax- and linear-attention networks,
respectively. In the nontrivial two-component setting considered above,
\[
\inf_{f\in\mathcal F_{\rm sm}}\bigl(R(f)-R(f^\star)\bigr)=0,
\qquad
\inf_{f\in\mathcal F_{\rm lin}}\bigl(R(f)-R(f^\star)\bigr)>0.
\]
\end{proposition}

\begin{lemma}
\label{lem:linear_attention_is_linear}
Let $f_{\mathrm{lin}}$ be a network composed of linear attention layers, with an arbitrary number of heads. Then $f_{\mathrm{lin}}$ is a linear function of $z$.
\end{lemma} 

\begin{proof}[Proof of Lemma~\ref{lem:linear_attention_is_linear}]
As with softmax attention, linear attention can be represented as an operator acting on probability measures. Let $\widehat{\mu}_L$ be the empirical measure associated with the i.i.d. samples $Z_1,\ldots,Z_L$ drawn from $\mu$. Then
$$
T_{\mathrm{lin}}[\widehat{\mu}_L](z)
\coloneqq 
\frac{1}{L}\sum_{k=1}^L
\left(Z_k^\top U z\right)VZ_k
\xrightarrow{L\to\infty} V\left(\int z'z'^\top\,\mathrm{d}\mu(z')\right)Uz
=T_{\mathrm{lin}}[\mu](z)
$$
Suppose now that $\mu=\mu_Z$ is the distribution of a Gaussian mixture,
$\mu_Z=
\pi_1\mathcal{N}(m_1,S)
+
\pi_2\mathcal{N}(m_2,S)$. Then,
\begin{align*}
T_{\mathrm{lin}}[\mu_Z](z)
&=
V\left(
\pi_1
\mathbb{E}_{Z_1\sim\mathcal{N}(m_1,S)}
\left[Z_1Z_1^\top\right]
+
\pi_2
\mathbb{E}_{Z_2\sim\mathcal{N}(m_2,S)}
\left[Z_2Z_2^\top\right]
\right)Uz\\
&=
V\left(
\pi_1m_1m_1^\top
+
\pi_2m_2m_2^\top
+
S
\right)Uz,
\end{align*}
Therefore, $T_{\mathrm{lin}}[\mu_Z](z)$ is linear in $z$. Since linear combinations and compositions of linear maps remain linear, any network composed of multi-head linear attention layers is also linear in $z$. This concludes the proof.
\end{proof}

\begin{proof}[Proof of Proposition~\ref{prop:softmax_beat_linear_for_denoising}]
We have already shown that a single-layer softmax attention network with two heads can attain the Bayes risk. It remains to prove that linear attention cannot achieve the same risk.

Let $f_{\mathrm{lin}}\in\mathcal{F}_{\mathrm{lin}}$. By Lemma~\ref{lem:linear_attention_is_linear}, there exists a matrix $A$ such that
$
f_{\mathrm{lin}}(z)=Az.
$
Moreover, there exists a matrix $A^\star$ minimizing the risk over all linear functions of $z$, namely
$$
R(A^\star)=
\inf_{A\in\mathbb{R}^{d\times d}}R(A).
$$
Indeed, we have
\begin{align*}
R(A)
&=
\mathbb{E}\left[\|X-AZ\|^2\right]\\
&=
\mathbb{E}\left[\|(I_d-A)X-A\varepsilon\|^2\right]\\
&=
\mathbb{E}\left[\|(I_d-A)X\|^2\right]
+
\mathbb{E}\left[\|A\varepsilon\|^2\right]\\
&=
\mathbb{E}\left[\|(I_d-A)X\|^2\right]
+
\xi^2|A|_{\mathrm{F}}^2.
\end{align*}
Thus, the risk is coercive and continuous as a function of $A$, and therefore admits a minimizer $A^\star$. Since $f^\star(z)=\mathbb{E}[X\mid Z=z]$ is the Bayes estimator for the squared loss, the excess-risk can be written for every measurable predictor $g$ as
$$
R(g)-R(f^\star)
=
\mathbb{E}\left[\|g(Z)-f^\star(Z)\|^2\right]
$$
Consequently,
\begin{align*}
R(f_{\mathrm{lin}})-R(f^\star)
&\geq
R(A^\star)-R(f^\star)\\
&=
\mathbb{E}\left[
\|
A^\star Z-f^\star(Z)
\|^2
\right].
\end{align*}
Using the expression of the optimal denoiser we obtain
$$
R(A^\star)-R(f^\star)
=
\mathbb{E}\left[
\|
\left(A^\star-\Gamma S^{-1}\right)Z-h(Z)
\|^2
\right].
$$
where
$
h(z)
=
\xi^2S^{-1}
\sum_{k}\beta_k(z)m_k
$. Assume, for contradiction, that this excess risk is equal to zero. It follows that
$$
\left(A^\star-\Gamma S^{-1}\right)z=h(z)
$$
first $\mu_Z$-almost everywhere, and then for every $z\in\mathbb R^d$:
both sides are continuous and $\mu_Z$ has an everywhere positive density
because $\xi>0$. The function $h$ is bounded so it forces
$$
A^\star-\Gamma S^{-1}=0 \qquad \text{and} \qquad h=0.
$$
This is impossible :
along the two directions $z=\pm t(m_1-m_2)$, the $h(z)$
converge to the two different endpoints, so $h$ is nonconstant. Therefore,
$$
R(A^\star)-R(f^\star)>0,
$$
and thus
$$
\inf_{f\in\mathcal{F}_{\mathrm{lin}}}
\left(R(f)-R(f^\star)\right)>0.
$$
This concludes the proof.
\end{proof}

\subsection{Residual in-context denoising (proof of Proposition~\ref{prop:icl})}
\label{app:icl-proof}

\begin{proof}[Proof of Proposition~\ref{prop:icl}]
Let $S=(\gamma^2+\xi^2)I_d$ and set
\[
a=\frac{\xi^2}{\gamma^2+\xi^2},
\qquad U_1=S^{-1},\quad V_1=aI_d,\quad U_2=V_2=0.
\]
For any admissible noisy mixture $\mu_Z$, Proposition~\ref{prop:gmm-attention}
and the equal-norm assumption give
\[
T_{U_1,V_1}[\mu_Z](z)
=az+a\sum_{k=1}^2\beta_k(z)m_k,
\qquad
T_{U_2,V_2}[\mu_Z](z)=0,
\]
where the $\beta_k(z)$ are the posterior component probabilities. Hence, with
$W_{\rm out}=(I_d,-2aI_d)$,
\begin{align*}
W_{\rm out}
\begin{pmatrix}
T^{\rm res}_{U_1,V_1}[\mu_Z](z)\\
T^{\rm res}_{U_2,V_2}[\mu_Z](z)
\end{pmatrix}
&=(1+a)z+a\sum_k\beta_k(z)m_k-2az\\
&=(1-a)z+a\sum_k\beta_k(z)m_k\\
&=\mathbb E[X\mid Z=z].
\end{align*}
The matrices depend only on the shared variances $\gamma^2$ and $\xi^2$;
the prompt-specific means and weights enter solely through the attention
operator. In addition, conditional expectation gives, for every task,
\[
R_{\mu_Z}(f)-R_{\mu_Z}(f^\star)
=\mathbb E\bigl[\|f[\mu_Z](Z)-f^\star[\mu_Z](Z)\|^2\mid\mu_Z\bigr].
\]
Integrating this nonnegative identity over any task distribution with
finite risk proves optimality for the in-context objective, since the
constructed predictor makes the right-hand side zero for every task.
This proves Proposition~\ref{prop:icl}.
\end{proof}

\end{document}